%% file: Thesis.tex
\documentclass[a4paper,fontsize=12pt]{scrbook}

\usepackage{mathptmx}
\usepackage[utf8]{inputenc}
\usepackage[UKenglish]{babel}
\usepackage{colorprofiles}
\usepackage[a-2b,mathxmp]{pdfx}[2018/12/22]
\usepackage{hyperref}

\usepackage{lmodern}
\usepackage{scrhack}

\usepackage[T1]{fontenc}
\usepackage{hyperxmp}
\hypersetup{
    pdftitle={TITLE},
    pdfauthor={first author, second author, …},
    pdfkeywords={KEYWORD, KEYWORD, …},
    pdflang={en},
    pdfcopyright={CC BY 4.0},
    pdflicenseurl={https://creativecommons.org/licenses/by/4.0/}
}

\usepackage{lmodern}

\usepackage{amsmath, amssymb}
\usepackage{algorithm}
\usepackage[noend]{algpseudocode}
\usepackage{booktabs}
\usepackage{harpoon}
\usepackage{scrhack}

\usepackage[utf8]{inputenc}

\usepackage{biblatex}
\preto\fullcite{\AtNextCite{\defcounter{maxnames}{99}}}

\usepackage{bibgerm}       		
\usepackage[T1]{fontenc} 
\usepackage{graphicx} 				
\usepackage{url}           		
\usepackage{listings, color}	
\usepackage{subfig}						
\usepackage{scrlayer-scrpage}					
\usepackage{todonotes}
\setuptodonotes{inline}
\usepackage{tabularx}
\usepackage{bibentry}

\usepackage{xcolor}
\usepackage{mathtools}
\usepackage{harpoon}
\usepackage{cancel}
\usepackage{nicefrac}
\usepackage{cleveref}
\usepackage{mathtools}
\usepackage{amsthm}
\usepackage{bbm}

\usepackage{csquotes}

\theoremstyle{plain}
\newtheorem{theorem}{Theorem}[section]
\newtheorem{proposition}[theorem]{Proposition}

\theoremstyle{definition}

\theoremstyle{remark}

\hypersetup{
    colorlinks=true,
    linkcolor=black,    
    urlcolor=black,      
    citecolor=black     
}
\ofoot[]{\thepage}

\usepackage{xcolor} 
\definecolor{purple}{rgb}{0.5,0.0,0.5}
\definecolor{mypink1}{RGB}{219, 48, 122}
\definecolor{mypink2}{cmyk}{0, 0.7808, 0.4429, 0.1412}
\definecolor{mygray}{gray}{0.6}

\newcommand\here[1]{\fcolorbox{red}{red}{\rule{0pt}{6pt}\rule{6pt}{0pt}}\quad}

\graphicspath{{./img/}}

\input{./misc/latex_shortcuts}

\begin{document}

\frontmatter
\include{./misc/titlepage}
\thispagestyle{empty}
\cleardoublepage

\include{./misc/self-assertion}
\thispagestyle{empty}
\cleardoublepage

\include{./misc/abstract}
\thispagestyle{empty}
\cleardoublepage

\include{./misc/abstract_de}
\thispagestyle{empty}

\pdfbookmark[section]{\contentsname}{toc}
\tableofcontents
\thispagestyle{empty}


\listoffigures
\thispagestyle{empty}

\listoftables
\thispagestyle{empty}


\mainmatter 
\input{./chapters/chapterIntroduction}
\input{./chapters/chapterFundamentals}
\input{./chapters/chapterMD}
\input{./chapters/chapterDiscDiff}
\input{./chapters/chapterSB}
\input{./chapters/chapterConclusion}


\printbibliography

\input{./chapters/chapterAppendix.tex}

\end{document}

%% file: misc/latex_shortcuts.tex
\newcommand{\Efunc}[1]{\mathbb{E}\left[ #1\right]}
\newcommand{\Efuncc}[2]{\mathbb{E}_{#1}\left[ #2 \right]}

\newcommand{\KLfunc}[2]{\mathbb{KL}\left[ \ #1 \ \middle|\middle| \ #2 \ \right]}

\newcommand{\Trfunc}[1]{\text{Tr}\left[ #1 \right]}
\newcommand{\tr}[1]{\text{Tr}\left[ #1 \right]}

\newcommand{\drift}{\mu(X_t, t)}
\newcommand{\discretizeddrift}{\mu(X_{t_n}, t_n)}
\newcommand{\diff}{\sigma(X_t, t)}
\newcommand{\discretizeddiff}{\sigma(X_{t_n}, t_n)}

\usepackage{mdframed}
\usepackage{tcolorbox}
\usepackage{xcolor}

\newif\ifshowprompt
\showpromptfalse

\makeatletter 
\DeclareFontFamily{U}{MnSymbolA}{}
\DeclareFontShape{U}{MnSymbolA}{m}{n}{
	<-6>  MnSymbolA5
	<6-7>  MnSymbolA6
	<7-8>  MnSymbolA7
	<8-9>  MnSymbolA8
	<9-10> MnSymbolA9
	<10-12> MnSymbolA10
	<12->   MnSymbolA12}{}
\DeclareFontShape{U}{MnSymbolA}{b}{n}{
	<-6>  MnSymbolA-Bold5
	<6-7>  MnSymbolA-Bold6
	<7-8>  MnSymbolA-Bold7
	<8-9>  MnSymbolA-Bold8
	<9-10> MnSymbolA-Bold9
	<10-12> MnSymbolA-Bold10
	<12->   MnSymbolA-Bold12}{}
\DeclareSymbolFont{MnSyA}{U}{MnSymbolA}{m}{n}
\SetSymbolFont{MnSyA}{bold}{U}{MnSymbolA}{b}{n}

\DeclareRobustCommand{\overleftharpoon}{\mathpalette{\overarrow@\leftharpoonfill@}}
\DeclareRobustCommand{\overrightharpoon}{\mathpalette{\overarrow@\rightharpoonfill@}}
\def\leftharpoonfill@{\arrowfill@\leftharpoondown\mn@relbar\mn@relbar}
\def\rightharpoonfill@{\arrowfill@\mn@relbar\mn@relbar\rightharpoonup}

\newcommand{\forward}[1]{\smash{\overrightharpoon{#1}}}
\newcommand{\backward}[1]{\smash{\overleftharpoon{#1}}}

\DeclareMathSymbol{\leftharpoondown}{\mathrel}{MnSyA}{'112}
\DeclareMathSymbol{\rightharpoonup}{\mathrel}{MnSyA}{'100}
\DeclareMathSymbol{\mn@relbar}{\mathrel}{MnSyA}{'320}
\makeatother

\renewcommand{\P}{\mathbbm{P}}
\newcommand{\E}{\mathbbm{E}}
\newcommand{\R}{\mathbb{R}}

\makeatletter
\DeclareRobustCommand{\cev}[1]{%
	{\mathpalette\do@cev{#1}}%
}
\newcommand{\do@cev}[2]{%
	\vbox{\offinterlineskip
		\sbox\z@{$\m@th#1 x$}%
		\ialign{##\cr
			\hidewidth\reflectbox{$\m@th#1\vec{}\mkern4mu$}\hidewidth\cr
			\noalign{\kern-\ht\z@}
			$\m@th#1#2$\cr
		}%
	}%
}
\makeatother

\makeatletter
\newcommand{\subsetcong}{\mathrel{\mathpalette\subset@cong\relax}}

\newcommand{\subset@cong}[2]{%
	\vbox{\offinterlineskip\m@th
		\ialign{\hfil$#1##$\hfil\cr
			\sim\cr\subset\cr
		}%
	}%
}

%% file: misc/titlepage.tex
\thispagestyle{empty}
\begin{center}





\vspace*{1.5cm}




{\LARGE \textbf{Time-Reversible Bridges of Data}}\\
\vspace*{0.3cm}
{\LARGE \textbf{with Machine Learning}}\\
\vspace*{1.0cm}
{vorgelegt von}\\
\vspace*{1.0cm}
{\LARGE \textbf{Ludwig Winkler, M. Sc.}}\\

\vspace*{1.0cm}
an der Fakultät IV – Elektrotechnik und Informatik \\
der Technischen Universität Berlin \\
zur Erlangung des akademischen Grades \\
\vspace*{0.5cm}
Doktor der Naturwissenschaften\\
– Dr. rer. nat. – \\
\vspace*{0.5cm}
genehmigte Dissertation
\vspace*{1.0cm}
\vspace*{\fill}

\begin{flushleft}
Promotionsausschuss:

\vspace{0.5cm}

{\hspace*{-0.35cm}
\begin{tabular}{ll}
Vorsitzender: & Prof. Benjamin Blankertz \\
Gutachter: & Prof. Klaus-Robert M\"uller \\
Gutachter: & Prof. Manfred Opper \\
Gutachter: & Prof. Sebastian Reich \\
Gutachter: & Prof. Wojciech Samek \\
\end{tabular}
}
\vspace{0.5cm}

Tag der wissenschaftlichen Aussprache: 18. November 2024
\end{flushleft}

\vspace*{\fill}

Berlin 2024

\end{center}

\cleardoublepage
\thispagestyle{empty}
\vspace*{\fill}
\begin{center}
  \emph{To my little entropy,}
  \newline
  \emph{and everyone who had the patience to explain it to me,}
  \newline
  \emph{then once more, and finally again a third time}
\end{center}
\vspace*{\fill}

\newpage


%% file: misc/self-assertion.tex
\newpage

\thispagestyle{empty}

\begin{large}

	\vspace*{6cm}

	\noindent
	Hereby I declare that I wrote this thesis myself with the help of no more than the mentioned literature and auxiliary means.
	\vspace{2cm}

	\noindent
	Berlin, 01.08.2024

	\vspace{3cm}

	\hspace*{7cm}%
	\dotfill\\
	\hspace*{8.5cm}%
	\textit{(Signature)}

\end{large}

%% file: misc/abstract.tex
\thispagestyle{empty}

\begin{center}
    \textbf{Abstract}
\end{center}


\noindent
The analysis of dynamical systems is a fundamental tool in the natural sciences and engineering.
It is used to understand the evolution of systems as large as entire galaxies and as small as individual molecules.
With predefined conditions on the evolution of dynamical systems, the underlying differential equations have to fulfill specific constraints in time and space.
This class of problems is known as boundary value problems.

This thesis presents novel approaches to learn time-reversible deterministic and stochastic dynamics constrained by initial and final conditions.
The dynamics are inferred by machine learning algorithms from observed data, which is in contrast to the traditional approach of solving differential equations by numerical integration.

The work in this thesis examines a set of problems of increasing difficulty each of which is concerned with learning a different aspect of the dynamics.
Initially, we consider learning deterministic dynamics from ground truth solutions which are constrained by deterministic boundary conditions.
Secondly, we study a boundary value problem in discrete state spaces, where the forward dynamics follow a stochastic jump process and the boundary conditions are discrete probability distributions.
In particular, the stochastic dynamics of a specific jump process, the Ehrenfest process, is considered and the reverse time dynamics are inferred with machine learning.
Finally, we investigate the problem of inferring the dynamics of a continuous-time stochastic process between two probability distributions without any reference information.
Here, we propose a novel criterion to learn time-reversible dynamics of two stochastic processes to solve the Schrödinger Bridge Problem.

In summary, we show that neural networks are a flexible function class to learn deterministic and stochastic dynamics of systems under the constraints of boundary conditions given by data.
Importantly, they are able to infer the dynamics of systems where the underlying differential equations are unknown or intractable.
The corresponding methodology can be applied to a wide range of problems in computer science, chemistry, and biology, providing a novel tool set for these fields.

%% file: misc/abstract_de.tex
\thispagestyle{empty}

\begin{center}
    \textbf{Zusammenfassung}
\end{center}

Die Analyse dynamischer Systeme ist ein grundlegendes Instrument in den Natur- und Ingenieurwissenschaften und wird verwendet, um die Entwicklung von Systemen zu verstehen, die so groß wie ganze Galaxien und so klein wie einzelne Moleküle sind.
Bei vordefinierten Bedingungen für die Entwicklung dynamischer Systeme müssen die zugrundeliegenden Differentialgleichungen bestimmte Bedingungen in Zeit und Raum erfüllen.
Diese Klasse von Problemen wird als Rand\-wertprobleme bezeichnet.

In dieser Arbeit werden neue Ansätze zum Erlernen zeitlich reversibler deterministischer und stochastischer Dynamiken vorgestellt, die durch Anfangs- und End\-wertbedingungen eingeschränkt sind.
Die Dynamik wird durch Algorithmen des maschinellen Lernens ermöglicht, welche im Gegensatz zum traditionellen Ansatz der Lösung von Differentialgleichungen durch numerische Integration steht.

Zunächst befassen wir uns mit dem Lernen der deterministischen Dynamik aus deterministischen Lösungen, die durch deterministische Randbedingungen einge\-schränkt sind.
Zweitens untersuchen wir ein Randwertproblem in diskreten Zustandsräumen, in denen die Vorwärtsdynamik einem stochastischen Sprungprozess folgt und die Randbedingungen diskrete Verteilungen sind.
Insbesondere wird die stochastische Dynamik eines speziellen Sprungprozesses, des Ehrenfest-Prozesses, betrachtet und die Rückwärtsdynamik mit maschinellem Lernen inferriert.
Als letztes untersuchen wir das Problem des Modellierens eines zeitkontinuierlichen stochastischen Prozesses zwischen zwei Wahrscheinlichkeits\-verteilungen ohne jeg\-liche Referenzinformationen.
Hier schlagen wir ein neuartiges Kriterium zum Erlernen der zeitlich umkehrbaren Dynamik zweier stochastischer Prozesse vor.

Zusammenfassend zeigen wir, dass neuronale Netze eine flexible Funktionsklasse zum Erlernen der deterministischen und stochastischen Dynamik von Systemen unter den Einschränkungen der durch die Daten gegebenen Randbedingungen sind.
Besonders wichtig ist, dass sie in der Lage sind, die Dynamik von Systemen abzuleiten, deren zugrunde liegende Differentialgleichungen unbekannt oder numerisch teuer zu lösen sind.
Die entsprechende Methodik kann auf ein breites Spektrum von Problemen in der Informatik, Chemie und Biologie angewandt werden und bietet neuartige Lösungsansätze für diese Anwendungsbereiche.

%% file: chapters/chapterIntroduction.tex
\chapter{Introduction}
\label{cha:chapter1}
\graphicspath{{./img/chapterintro}}

The intersection of machine learning and traditional computational methods presents a fertile ground for innovation.
In particular, the role of neural networks has become increasingly prominent in various fields of science and technology.
Initially conceived as computational models loosely inspired by the human brain, neural networks have undergone a series of transformative advances in capabilities in the last decade.
While their early iterations were limited by both computational resources and theoretical understanding, significant progress has been made since then on both fronts.
Recent advances in neural network architectures, coupled with substantial increases in computational power and the availability of large datasets, have extended the utility of these models far beyond their original scope \cite{lecun2015deep, alzubaidi2021review}.

As with other technological advances throughout history, neural networks have become a catalyst for innovation, rapidly diffusing across various fields as a powerful tool to solve computational problems and giving rise to a diverse array of applications.
Following their proven track record in areas such as machine perception, i.e. image and speech recognition, \cite{nassif2019speech, chai2021deep}, they swiftly found new and impactful applications among others in solving differential equations, further broadening their influence and utility \cite{beck2020overview, chen2018neural,richter2022robust,richter2023improved,berner2022optimal}.

Differential equations are the foundations upon which we understand a wide range of phenomena in the natural world, from the movement of individual atoms in molecules to the trajectories of enormous galaxies through faraway space \cite{hartmann1964ordinary, betounes2010differential, oksendal2003stochastic}.
Differential equations describe the infinitesimal changes of a system over time and their solution is a function that satisfies these equations.
However, the complexity inherent in these equations, particularly when dealing with systems characterized by high dimensionality, non-linearity, or stochasticity, poses significant challenges for traditional analytical and numerical methods \cite{braun1983differential, butcher2016numerical}.
Of particular concern to this thesis, two limiting factors to solving differential equations can be highlighted: the need for \emph{explicit knowledge of the underlying dynamics} and the \emph{computational cost of solving high-dimensional systems}.

Commonly, differential equations are solved from an initial condition, which is then propagated forward in time to predict the system's behavior at a future time.
But what happens if we require the system to also fulfill specific conditions at a later point in time?
This is referred to as a \emph{boundary value problem} which requires solving a differential equation subject to constraints on the possible solutions of the underlying differential equation \cite{gakhov1990boundary}.
In the context of this thesis, we aim to infer the differential equation of dynamical systems while requiring the corresponding solution to fulfill specific conditions at the initial and final time of the solution.


The stochastic counterpart to boundary value problems are stochastic bridges \cite{chen2015stochastic, vargas2021solving, berner2022optimal}, which provide probabilistic boundary conditions that constrain the solution of the underlying stochastic dynamics.
Here the dynamics are modeled by stochastic differential equations which provide a framework to describe the evolution of random processes over time \cite{chen2016modeling, oksendal2003stochastic, oksendal2013stochastic, gardiner1985handbook}.
Stochastic bridges assume a single overarching stochastic process in both directions.

Yet, fixing one direction of the bridge to a predetermined, tractable, and time-reversible stochastic diffusion process \cite{nelson1966derivation, nelson1967dynamical, nelson1988stochastic, nelson1979connection,anderson1982reverse}, allows us to reduce the complexity of the inference problem.
Now, the problem can be posed as learning the reverse dynamics of a known process.
The advantages this modeling ansatz has led to a series of breakthroughs in generative modeling \cite{cao2024survey, yang2023diffusion, song2023consistency,ho2020denoising, song2020denoising}.

Obtaining analytical solutions to differential equations is often infeasible for a wide range of problems of interest, especially in high-dimensional spaces, and numerical methods are required to approximate the solution \cite{butcher2016numerical,braun1983differential}.
The computational cost of numerical solutions poses one problem that neural networks can alleviate by learning the dynamics with a smaller computational footprint.
More profoundly, neural networks can infer previously unknown dynamics from data, thus accelerating the solution process by bypassing the need for explicit knowledge of the actual data generating process, making them an attractive tool for solving differential equations \cite{raissi2018deep, berner2020numerically, richter2022robust}.
By leveraging the flexibility and adaptability of neural networks, it is possible to learn the underlying dynamics of a system purely from data in an \emph{approximate} manner, thus preventing the need for explicit knowledge of the actual data generating process.


\section{Objective and Scope}
\label{sec:objective}

The goal of the research conducted within the context of this thesis is the exploration and development of methodologies for solving stochastic bridges and boundary value problems by integrating the principles of ordinary differential equations, stochastic differential equations, and the advanced computational capabilities of neural networks.
We aim to bridge the gap between traditional mathematical approaches and modern, data-driven computational techniques, creating a framework that leverages the strengths of both domains.
Instead of directly applying neural networks as a silver bullet to solve these problems by brute computational force, we aim to understand the underlying dynamics to deduce strong inductive biases and subsequently use neural networks only to infer the essential quantities required to model the identified family of dynamics.

The research conducted for this thesis is divided into three stages, each considering a more complex and challenging problem than the previous one.
We define more challenging as removing an increasing amount of information about the underlying dynamics and requiring the neural network to infer more and more of the missing information about the underlying dynamics.

\begin{figure}[H]
    \centering
    \includegraphics[width=0.6\linewidth]{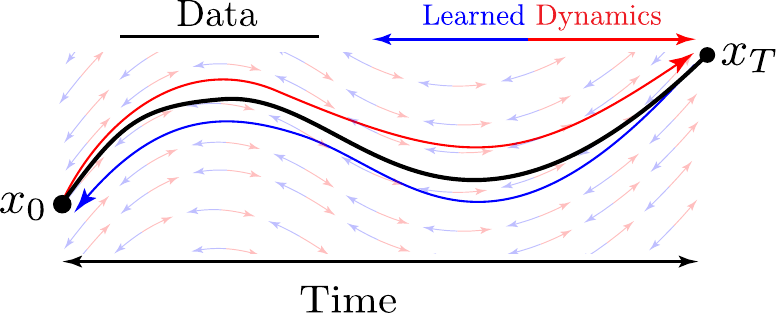} 
    \caption{In \Cref{cha:deterministicbvps} we consider boundary value problems in a deterministic setting in which we learn the time-reversible, deterministic dynamics with neural networks from observed solutions.
        The boundary conditions $x_0$ and $x_T$ are deterministic.}
    \label{fig:introtoy1}
\end{figure}

We can thus condense the problems to be tackled in this thesis into three (very compressed) sentences and their respective illustrative visualizations:
\begin{itemize}
    \item
          Firstly, the thesis focuses on \textbf{deterministic boundary value problems} for which the underlying, time-reversible dynamics are learned from \textbf{provided realizations of the dynamics} (see \Cref{fig:introtoy1}).
    \item
          Secondly, we consider stochastic half-bridges with \textbf{probabilistic boundary conditions} and stochastic forward dynamics in a discrete state space for which the \textbf{unknown, stochastic backward dynamics} have to be inferred(see \Cref{fig:introtoy2}).
    \item
          Finally, we propose a criterion to solve the \textbf{Schr\"odinger Bridge Problem}, a stochastic bridge, for which we have to learn the \textbf{unknown forward and backward stochastic dynamics between two provided probability distributions} serving as stochastic boundary conditions (see \Cref{fig:introtoy3}).
\end{itemize}

\begin{figure}[H]
    \centering
    \includegraphics[width=0.7\linewidth]{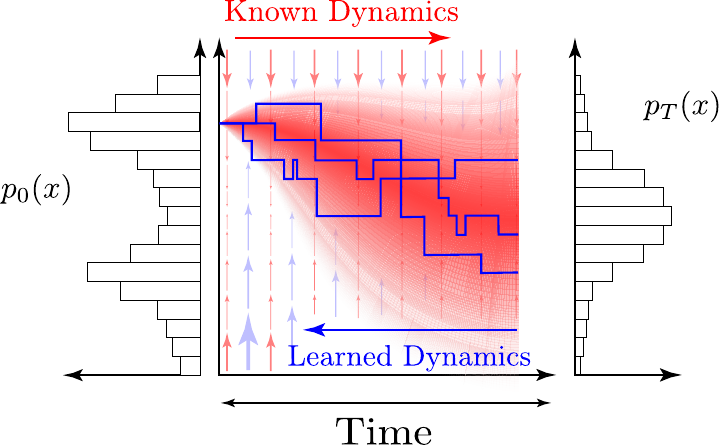} 
    \caption{In \Cref{cha:discretehalfbridge} we consider the learning of stochastic half bridges in discrete state spaces with known forward dynamics and unknown backward dynamics which are inferred with a neural network. The learned dynamics then invert the known dynamics in time.
        The boundary conditions are discrete probability distributions denoted by $p_0(x)$ and $p_T(x)$.}
    \label{fig:introtoy2}
\end{figure}

\begin{figure}[ht]
    \centering
    \includegraphics[width=0.7\linewidth]{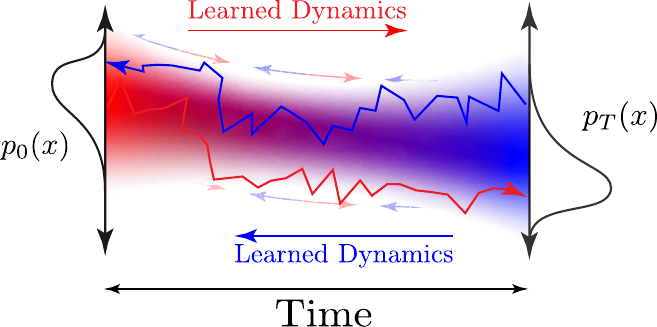} 
    \caption{In \Cref{cha:stochasticbridge} we consider the learning of Schrödinger bridges in continuous state spaces with unknown forward and backward dynamics which are inferred with a neural network.
        Each of the learned dynamics is trained to reverse the opposite dynamics in time between the two probabilistic boundary conditions given by the two continuous distributions $p_0(x)$ and $p_T(x)$.}
    \label{fig:introtoy3}
\end{figure}

Of particular interest to us is the concept of \emph{time-reversibility} and how it can be exploited to learn the reverse dynamics of a system from provided dynamics.
For deterministic systems, time-reversibility is a natural property of the underlying dynamics and the reverse dynamics can be derived in a sufficiently convenient way.
For stochastic systems, time-reversibility is a more complex concept and requires a more nuanced approach to revert the inherent diffusion in stochastic processes.

A significant portion of the thesis will be dedicated to the practical implementation of these models with the relevant dynamics approximated with neural network architectures.
As such the focus lies on the applicability of neural networks and how they can be integrated into the relevant mathematical frameworks.
This research aims to contribute to the broader field of computational mathematics by providing insights into how neural networks can be utilized to solve other complex differential equations.
Therefore, this thesis concerns itself with approaches to solving stochastic bridges and boundary value problems by approximating formerly intractable dynamics with neural networks in a substitutive manner.
By combining the rigor of differential equations with the adaptability of neural networks, this research seeks to advance the field of computational mathematics and provide practical solutions to complex stochastic problems.

\section{Outline}

The remainder of the thesis is organized into four chapters.

\textbf{Chapter \ref{cha:fundamentals}}
introduces the theoretical and practical topics deemed necessary to enjoy reading this dissertation.
The concepts can be broadly differentiated into two groups: ordinary differential equations for deterministic dynamical systems and stochastic differential equations for probabilistic systems.

\textbf{Chapter \ref{cha:deterministicbvps}}
examines the application of neural networks as learnable deterministic dynamics and showcases how recurrent neural network architectures can be used to speed up molecular dynamics simulations.
These architectures are successfully trained to reconstruct deterministic bridges between the initial and final point of the evolution of the underlying deterministic dynamical system.

\textbf{Chapter \ref{cha:discretehalfbridge}}
relaxes the strictly deterministic setting and introduces the concept of stochastic half bridges.
The chapter concerns itself with discrete state spaces in which the forward dynamics between two discrete probability distributions are known and the backward dynamics are learned.
For ordered discrete state spaces with vanishing differences between states, we show that the learned backward dynamics converge to a continuous state space reverse process.

\textbf{Chapter \ref{cha:stochasticbridge}} introduces the Schr\"odinger Bridge Problem.
Only the boundary conditions in the form of two probability distributions are provided and both the forward and backward dynamics have to be learned from the data.
This represents the most general formulation of which the deterministic bridges and the stochastic half bridges are special cases.

\textbf{Chapter \ref{cha:conclusion}} summarizes the thesis, discusses the problems that were posed, and gives an outlook about future work.

\section{List of Publications}

Significant parts of this thesis follow published or submitted work in peer-reviewed journals or conferences.
Therefore I would like to thank my Co-Authors for allowing me to use parts of our joint work for this thesis.
The primary contributions and ﬁndings of this thesis are based on the following peer-reviewed publications:

\begin{itemize}
    \item \fullcite{winkler_2022}
    \item \fullcite{winkler2024ehrenfest}
    \item \fullcite{winkler_2023}
\end{itemize}

Other publications that are not directly related to the content of this thesis but have been published during my PhD studies are:

\begin{itemize}
    \item \fullcite{studer2021towards}
    \item \fullcite{winkler2022stochastic}
    \item \fullcite{wang2023interpolating}
    \item \fullcite{vaitl2024fast}
\end{itemize}

%% file: chapters/chapterFundamentals.tex
\chapter{Theoretical Context and Fundamentals}
\label{cha:fundamentals}

\begin{tcolorbox}[colback=gray!10!white, colframe=black]
      Parts of this chapter are constituted from the publications:
      \begin{itemize}
            \item \fullcite{winkler_2022}
            \item \fullcite{winkler_2023}
            \item \fullcite{winkler2024ehrenfest}
      \end{itemize}
\end{tcolorbox}

\section{Neural Networks}

In the realm of computational intelligence, neural networks play a critical role due to their impressive capacity to learn from data and adapt to a diverse range of issues.
Over the last decade, neural networks have experienced notable advancements, broadening their impact beyond previously conventional applications of machine learning systems.

The architectures of these networks have undergone profound changes with the advent of deep learning, resulting in models that are deep and wide enough to capture complex patterns in data.
Innovations like convolutional neural networks (CNNs) and Recurrent Neural Networks (RNNs) have become increasingly nuanced and specialized \cite{lecun2015deep}.
Attention mechanisms and Transformer models have established novel benchmarks in sequence modeling tasks, impacting domains ranging from natural language processing to time-series analysis \cite{vaswani2017attention, GoogleNMT, bahdanau2014neural}.
With the development of these architectures serving as trailblazers, their application and machine learning in general has expanded to a wide array of fields, including computer perception, natural language processing, and protein folding \cite{he2015delving,radford2019language,jumper2021highly, studer2021towards, dotenco2016automatic, mullan2015unobtrusive}.

The computational power needed to train sizable neural networks has grown in tandem with architectural intricacies.
Advancements in graphics processing unit (GPU) technology, distributed computing, and specialized hardware have greatly decreased the time needed to train and infer complex models.
Experimental insights into model performance have revealed significant gains in performance by scaling up the size of neural networks, leading to the development of models with billions of parameters.
Naturally, the data requirements for training these models have also increased, necessitating the development of more efficient data collection and preprocessing techniques \cite{brown2020language, radford2019language, touvron2023llama}.
The biggest impact of these scaling laws has been observed in natural language processing, where a series of models like GPT-4, Gemini, and Llama have demonstrated remarkable capabilities in understanding and generating human-like text \cite{achiam2023gpt,radford2019language, team2023gemini, touvron2023llama}.

Algorithmic enhancements in training methods, including advanced optimization techniques, regularization methods, and novel activation functions, have addressed issues of convergence and overfitting \cite{talbi2020optimization, sun2019optimization}.
For example, the incorporation of batch normalization and residual connections has enabled the training of networks that are significantly deeper than previously achievable \cite{ioffe2015batch, srivastava2014dropout}.

Neural networks have progressed beyond the mere recognition of patterns in images and speech. 
They can now tackle more complex cognitive tasks, such as reasoning, planning, and generating human-like text.
This cognitive leap is best demonstrated by large language models, which now exhibit the ability to produce and comprehend natural language \cite{touvron2023llama, radford2019language, achiam2023gpt, team2023gemini}.

In scientific computing, neural networks are increasingly employed to model intricate systems described by differential equations.
The ability of neural networks to approximate functions and derivatives makes them useful in solving forward and inverse problems related to differential equations \cite{berner2022optimal,richter2023improved, beck2020overview}.
This has greatly expanded research opportunities and practical applications, especially in fields such as physics, chemistry, and systems biology \cite{unke2021machine,sgdml_bookAppl,chmiela2017machine}.



Although significant progress has been made, ongoing research addresses challenges in interpretability, robustness, and computational cost.
Efforts to enhance training algorithms' efficiency, minimize data requirements of models, and improve generalization dominate neural network research.

Future directions seek to seamlessly integrate domain expertise into neural network architectures, develop more sophisticated models of uncertainty, and establish frameworks that can optimize the strengths of neural networks in generalizing with the precision of exact methodologies.


\section{Ordinary Differential Equations}
\label{sec:ode}

In general, a differential equation is an equation that relates one or more unknown functions and their derivatives.
An ordinary differential equation is a differential equation that involves only one independent variable, typically denoted as time, and its derivatives of the dependent variable, usually represented by space \cite{hartmann1964ordinary, braun1983differential}.
Intuitively, an ordinary differential equation (ODE) describes how a function evolves over time, capturing the rate of change of the function at each point in time.

Mathematically, an ODE for the dependent variable $x_t$ and the independent time variable $t$ is expressed as
\begin{equation}
      \frac{dx_t}{dt} = f(t, x_t).
\end{equation}

Here, the derivative \( \frac{dx_t}{dt} \) signifies the rate of change of the function $x_t$, and $f(t, x_t)$ is a function that describes how the rate of change is influenced by \( t \) and $x$ itself \cite{hartmann1964ordinary}.
The `order' of an ODE is determined by the highest derivative present in the equation.

ODEs are ubiquitously employed across the sciences and engineering to model dynamic systems evolving over time.
From the oscillations of a simple pendulum to the complexities of celestial mechanics, ODEs provide a framework for understanding natural phenomena.

Solving an ODE means finding a function, or a set of functions, that satisfies the differential equation for all times $t$ and state values $x_t$.
There are two primary approaches: analytical and numerical.
Analytical solutions offer a closed-form expression, involving classical methods for finding solutions to ODEs. However, many ODEs, particularly non-linear ones, cannot be solved analytically \cite{hartmann1964ordinary}.

When dealing with an initial value problem (IVP), we include the additional condition $ x_0$ at the initial time $t=0$.
Solving an IVP involves finding a solution \( x_t \) that satisfies both this initial condition and its differential equation for all values of time $t$ \cite{hartmann1964ordinary,braun1983differential,dormand1980family}.
The value of $x_{t + \Delta t}$ with the initial condition $x_t$ can then be computed by solving the integral
\begin{align}
      x_{t + \Delta t} = x_t + \int_t^{t + \Delta t} f(x(\tau), \tau) d\tau.
\end{align}

\begin{figure}[htb]
      \centering
      \includegraphics[width=\textwidth]{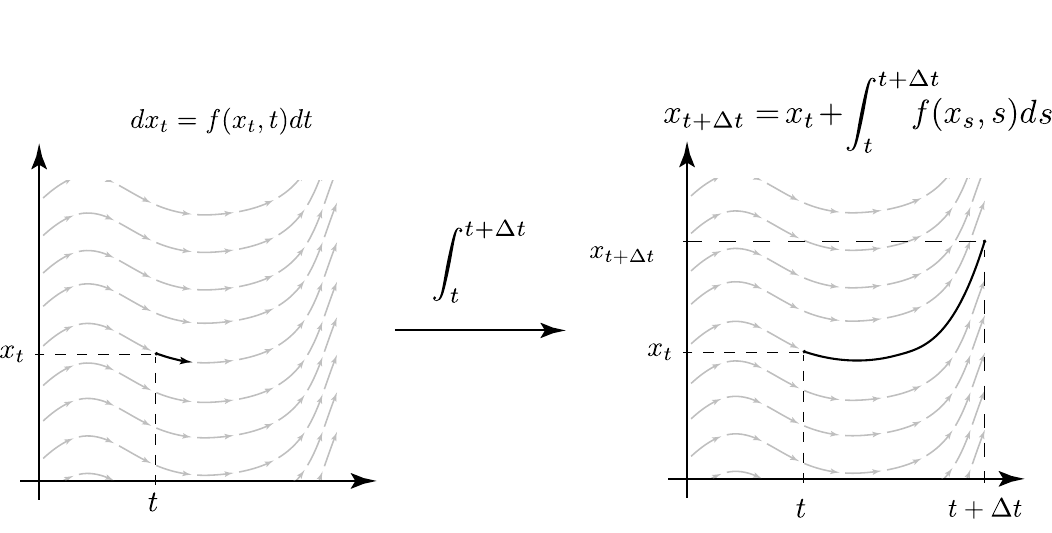}\\
      \caption{A figurative visualization of solving an ordinary differential equation.}\label{fig:intro}
\end{figure}

Obtaining the integral of the differential equation is a nontrivial problem as most differential equations take on complicated forms, thus rendering their integral even more complicated \cite{braun1983differential}.
Thus one resorts to numerical methods to solve ODEs by approximating the integral over time with a series of discrete steps in time.
One common numerical approach involves starting from $x_0$ and iteratively constructing the trajectory $x_t$ forward or backward in time, using a series of approximations.
Numerical methods provide approximate solutions and are essential when analytical solutions are unobtainable.
Techniques like the Euler method and Runge-Kutta methods are widely used to compute approximate solutions of ODEs \cite{betounes2010differential,butcher2016numerical,braun1983differential}.

These numerical methods work iteratively by tracing a trajectory through the state space of $x_t$, repeatedly evaluating and adding the time differential $f(x_t, t)$ to the respective state $x_t$.
Importantly, these evaluations occur at discrete points in time and thus are a numerical approximation.
The Euler method uses a predefined step size $\Delta t$ to evaluate the differential equation at a series of discrete points in time $t_n$ \cite{butcher2016numerical,dormand1980family},
\begin{align}
      x_{t_{n+1}} = x_{t_n} + f(x_{t_n}, t) \Delta t.
\end{align}
Choosing $\Delta t$ sufficiently small and repeatedly computing the recursive formula provides an approximate solution to the true solution of the differential equation.
The Euler method assumes that the differential $f(x_t, t)$ remains constant in the time frame $[t, t+\Delta t]$.
For nonlinear or rapidly changing differential equations this is not guaranteed and the solution of the Euler method can accumulate large errors.

To mitigate these numerical issues, higher order solvers such as the Runge-Kutta method have been proposed which evaluate the differential equation at multiple distinct points in the time frame $[t, t+\Delta t]$ \cite{butcher2016numerical,dormand1980family, butcher1996history}.
In cases where higher-order solvers fail, adaptive solvers can be used which adjust the step size $\Delta t$ based on various heuristics.
These methods belong to the family of explicit solvers which compute the future state of the system from the current state \cite{lambert1991numerical, butcher1996history}.
Implicit solvers compute the state of the system at a later time by solving an equation that involves both the current and the next state in order to provide numerical stability \cite{pareschi2000implicit}.
A large set of these solvers adapt their step size based on the local error of the solution, which is computed by comparing the solution of two different step sizes.
For ODE's which exhibit rapid changes for even small changes in time, these solvers resort to infeasible small step sizes, which makes computing their solution hard to attain in practical time \cite{braun1983differential,pareschi2000implicit,dormand1980family}.


\subsection{Reverse Time Ordinary Differential Equations}

Typically, ODEs are solved forward in time to project future states from known initial conditions.
However, the methodology of solving ODEs backward in time, known as backward time integration, is equally feasible.
This technique involves deducing a system's past states $x_t$ for $0 \leq t \leq T$ from its known final state $x_T$.

In backward time integration, the focus shifts to starting from a known final condition $x_T$ at time $T$ and integrating backward to determine $x_t$ for $t \leq T$ \cite{hartmann1964ordinary, chen2018neural,kolmogorov1938analytic}.
To obtain the reverse-time ODE we first define the reverse time index $\tau(t) = T - t$.
We can directly see that $d\tau = - dt$ by differentiating both sides.
Naturally, we can invert and equally express the original time index in terms of the reverse time index as $t(\tau) = \tau^{-1}(t) = T - \tau$.
Applying the chain rule to the time derivative of $x_t$ in terms of $\tau$ we obtain
\begin{align}
      \frac{d x_t}{d\tau} \Big|_{t=t(\tau)}
      = & \frac{d x_t}{dt}\Big|_{t=t(\tau)} \frac{dt(\tau)}{d\tau} \\
      = & f(x_t, t) \frac{d (T - \tau)}{d\tau}                     \\
      = & - f(x_{T - \tau}, T-\tau).
\end{align}

With this result we can equate the change $dx_t$ in terms of our reverse time $\tau$ and vice-versa by using $d\tau = - dt$,
\begin{align}
      dx_{T-\tau} = & - f(x_{T-\tau}, T-\tau) d\tau \\
      dx_t =        & f(x_t, t) dt
\end{align}

The reverse time ODE can be expressed in terms of the forward time ODE by changing the sign and defining the time progression in terms of the reverse time index $\tau$.
For this reason, the same precautions have to be taken and the same solvers can be applied to reverse time ODE which are applied to their forward time equivalents.

\section{Stochastic Processes}

Stochastic processes are crucial when mathematically modeling dynamic systems impacted by randomness.
Concretely, a stochastic process consists of random variables $\{X_t\}_{0 \leq t \leq T}$, where $X_t$ represents the system's state at a particular time $t$, and the index $t$ typically indicates time \cite{van1992stochastic,kolmogorov1938analytic,viniotis1998probability}.
This index set can be discrete, $t = \{0, 1, 2, \ldots\}$, or continuous, $t = [0, T)$.

A stochastic process is defined using a probability space $(\Omega, \mathcal{F}, \mathbb{P})$, where $\Omega$ represents the sample space, $\mathcal{F}$ the sigma-algebra of events, and $\mathbb{P}$ the probability measure \cite{van1992stochastic,kolmogorov1938analytic,viniotis1998probability}.
Each path, or instance, of the process, represents a possible realization of the random variables $\{ X_t\}_{0 \leq t \leq T}$ \cite{parzen1999stochastic}.

The classification of stochastic processes depends on the characteristics of their time index and state space.
In discrete-time processes (e.g. Markov chains), the advancement occurs at discrete and separate time intervals, which render them appropriate for models that require repeating observations at discrete time steps \cite{hastings1970monte}.
Conversely, continuous-time processes, such as Brownian motion, are defined for every moment within a time frame and are applied to model phenomena that present continuous fluctuations \cite{van1992stochastic,parzen1999stochastic}.
Processes in continuous time are further classified into continuous state spaces, modeled by stochastic differential equations, and discrete state spaces, modeled using continuous-time Markov chains (CTMC) \cite{privault2013understanding,privault2022introduction}.

Markov processes are a subset of stochastic processes distinguished by the Markov property: $\mathbb{P}(X_{t+1} | X_t) = \mathbb{P}(X_{t+1}| X_1, \ldots, X_t)$, for all $t=\{1, \ldots, t-1, t, t+1\}$ throughout the state space.
This feature implies that the future state depends exclusively on the current state, resulting in a \emph{memoryless} procedure.
The Brownian motion is a classic stochastic process commonly used to simulate events like fluctuations in stock prices and particle movements.
In its simplest form, the random walk is defined as a sequence of independently and identically distributed random variables.
The sum of these variables determines the process's position at each time step.

Another prominent discrete process is the Poisson process, characterized by independent events over non-overlapping intervals and Poisson distribution \cite{privault2022introduction,privault2013understanding}.
It is frequently used to model count data over time, for instance, call arrivals in a network or vehicles passing a checkpoint.


Stochastic differential equations (SDEs) and the Fokker-Planck Equation (FPE) frameworks play vital roles in modeling, analyzing and understanding stochastic processes \cite{risken1996fokker,gardiner1985handbook}.
While stochastic differential equations describe how stochastic processes evolve at the level of individual sample paths or trajectories, the Fokker-Planck Equation (FPE) offers a complementary perspective by focusing on the distributional evolution of these processes.
The FPE is derived from an SDE and elucidates how the probability density function (PDF) of the process's state variable changes over time.

The FPE facilitates the statistical analysis of a stochastic process, offering insights into the system's probabilistic behavior and distributional properties across time.
It enables us to examine the likelihood of different states of the system and understand its long-term, macroscopic behavior, including the presence of steady-state distributions or the process's ergodic features.
While stochastic differential equations provide a framework for modeling the path-wise behavior of a process, the Fokker-Planck equation helps us understand the aggregate evolution of the process.
This dual characterization of a stochastic process is particularly beneficial in intricate systems, where analyzing individual trajectories directly proves difficult or when more emphasis is placed on the statistical properties as opposed to specific paths.

\subsection{Stochastic Differential Equations}

Stochastic Differential Equations (SDEs) extend the concept of Ordinary Differential Equations (ODEs) to incorporate randomness or noise within dynamic systems \cite{oksendal2003stochastic,gardiner1985handbook}.
The inclusion of randomness or noise factors in SDEs allows for the modeling of more chaotic, complex, and unpredictable systems.
They are represented mathematically as Ito drift-diffusion processes \cite{ito1951stochastic, ito1984introduction}:
\begin{align}
      dX_t = \drift \, dt + \diff \, dW_t.
      \label{eq:SDE}
\end{align}
where $\{X_t\}_{0 \leq t \leq T}$ represents the state variable at time $t$, $f(X_t, t)$ is the deterministic drift component, $\sigma(X_t, t)$ denotes the diffusion term modeling the random fluctuations, and $dW_t$ refers to a differential of a Wiener process (Brownian motion), capturing the stochastic behavior of the system.

A Wiener process is a stochastic process with the initial condition of $W(0)=0$ and has independent Gaussian increments.
Thus a Wiener process is Gaussian distributed random variable with variance $u$, $W(t+u) - W_t \sim \mathcal{N}(0, u)$ for which the variance of the random variable is dependent on the elapsed time $u$.
Practically, a Wiener process can be approximately simulated for a fixed step size $\Delta t$ as $\Delta W_t = \epsilon \sqrt{\Delta t}, \epsilon \sim \mathcal{N}(0,1)$.

\begin{figure}[htb]
      \centering
      \includegraphics[width=\textwidth]{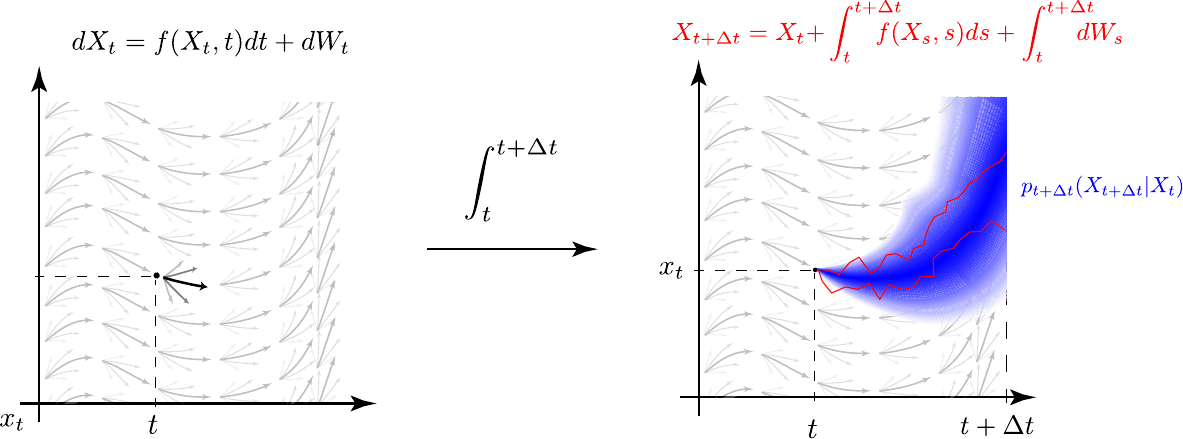}\\
      \caption{A figurative visualization of solving an Ito drift-diffusion process.
            Integrating the deterministic dynamics $f(x, t)$ and the stochastic Wiener process $dW_t$ generates a conditional probability distribution $p_{t+ \Delta t}(x_{t+\Delta t} | x_t)$ over subsequent states $X_t$.}
      \label{fig:SDE}
\end{figure}

The solution to an SDE, the random variable $X_t$, is itself a stochastic process that models the evolution of a system under both deterministic trends and random influences and describes the distribution of state variables \cite{van1992stochastic,gardiner1985handbook,kolmogorov1938analytic}.
Importantly, the inclusion of a random variable, modeled by $\sigma(X_t, t) dW_t$ forces us to consider not a single deterministic solution, but a probability distribution of solutions.

Numerical solutions often use the Euler-Maruyama and Milstein methods, which are stochastic counterparts to deterministic methods used for ODEs \cite{oksendal2013stochastic,parzen1999stochastic,gardiner1985handbook,sarkka2019applied}.
These involve discretizing the time domain and solving the SDE iteratively to construct solution paths to estimate the probability distribution of outcomes.
Similarly to earlier, we proceed by discretizing the time axis into distinct points in time $t_n$ and recursively compute the solution as
\begin{align}
      X_{t_{n+1}}
      = & X_{t_n} + \discretizeddrift \Delta t + \discretizeddiff \ \epsilon \sqrt{\Delta t}.
\end{align}

In the case of ordinary differential equations, more advanced solvers beyond the Euler discretization like the Runge-Kutta methods are used to solve the differential equation.
For stochastic differential equations, the inherent randomness present in the Wiener process makes such advanced solvers hard to apply as the error of the Wiener process is not deterministic \cite{sarkka2019applied,oksendal2003stochastic}.
In fact, the derivative of a Wiener process in the limit of $\Delta t \rightarrow 0$ has infinite variance, which makes the development of higher order and adaptive solvers more intricate.
Therefore, the Euler-Maruyama method is the most common method used to solve SDEs, as it is simple to implement and provides a good approximation of the true solution \cite{sarkka2019applied,oksendal2003stochastic}.

Integrating an Ito drift-diffusion process numerically realizes a sample path of the stochastic process, providing a trajectory of the system's state over time.
As the Wiener process introduces randomness, each realization of the process will differ, reflecting the inherent stochasticity of the system.
Therefore, while the numerical solution of an SDE provides a single path, the quantity of interest is in fact the distribution of these paths, each representing a possible realization of the system's evolution which is captured by the Fokker-Planck equation \cite{van1992stochastic,gardiner1985handbook}.



\subsection{The Fokker-Planck Equation}

The Fokker-Planck equation (FPE), also referred to as the Kolmogorov forward equation, is an important equation in the field of stochastic processes and statistical physics \cite{risken1996fokker,tabar2019analysis,van1992stochastic}.
It describes the temporal evolution of a particle's probability density function in the presence of both deterministic forces and stochastic perturbations, providing a statistical characterization of the system's behavior over time.

Mathematically, the Fokker-Planck equation is expressed as a partial differential equation (PDE) that governs the evolution of the probability density function \( p_t(x) \) of a particle's position \( x \) at time \( t \) the dynamics of which is governed by the corresponding SDE \eqref{eq:SDE},
\begin{align}
      \frac{\partial p_t(x)}{\partial t} = -\frac{\partial}{\partial x}\left[\mu(x, t) p_t(x)\right] + \frac{1}{2}\frac{\partial^2}{\partial x^2}\left[\sigma(x, t) p_t(x)\right].
\end{align}
where the drift component, $\mu(x, t)$, represents the deterministic aspect of motion and the diffusion term, $\sigma(x,t)$, accounts for random fluctuations usually related to thermal disturbances or noise.
Its derivation can be achieved succinctly from its governing stochastic differential equation via Ito calculus and integration by parts and is provided in its full length in appendix \ref{app:ch2fpederivation} for the interested reader.


It is an essential tool for understanding the probabilistic dynamics dictated by stochastic differential equations.
Although SDEs adeptly describe the path-dependent behaviors, the Fokker-Planck equation offers a complementary macroscopic perspective.
The methodology involves characterizing the evolution of probability density functions over time, which shifts the focus to a broader ensemble-level view and away from individual sample paths that are characterized by SDEs.

The solution of the Fokker-Planck equation characterizes the probability density function $p_t(x)$ and answers the question of how probable a certain position in space is at a specific time \cite{bogachev2022fokker,risken1996fokker,van1992stochastic,gardiner1985handbook}.
Analytical solutions are available only for a few cases with certain forms of $\mu(x,t)$ and $\sigma(x,t)$ as in linear or quadratic potential fields.

As PDE's in general are non-trivial to solve, numerical methods are frequently used to solve the Fokker-Planck equation \cite{risken1996fokker,van1992stochastic}.
Various techniques like finite difference, spectral, and finite element methods are employed to discretize the equation \cite{gardiner1985handbook}.
This reformulates the PDE into a system of algebraic equations that are easily computable.
Nevertheless, this represents a numerical discretization of the solution which automatically incurs a discretization error in the solution.

\subsection{Jump Processes}
\label{sec: jump process}

\begin{figure}[htbp]
      \centering
      \includegraphics[width=0.8\textwidth]{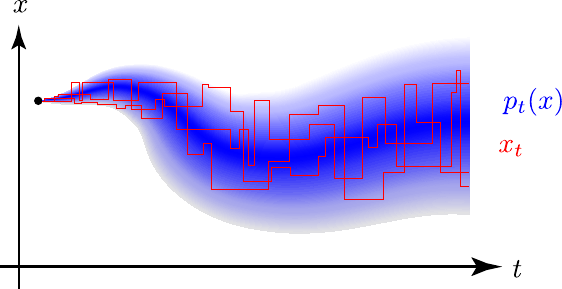}\\
      \caption{A figurative visualization of a continuous-time Markov chain and the corresponding marginal distribution.
            The trajectories move in stochastic, discrete jumps between states, and the holding time between each jump is stochastic. The marginal distribution $p_t(x)$ is a categorical distribution but was visualized in a continuous manner for illustrative purposes.}
      \label{fig:CTMC}
\end{figure}

Jump processes and continuous-time Markov chains (CTMCs) constitute a fundamental area of study in stochastic modeling, particularly in systems exhibiting sudden transitions \cite{berger2012introduction, viniotis1998probability,privault2022introduction,privault2013understanding}.
In essence, jump processes are a type of stochastic process characterized by the ability to model dynamics where systems experience abrupt changes, as opposed to the gradual evolution typical of diffusion processes.
These abrupt changes, or 'jumps', are important for representing real-world phenomena where instantaneous events significantly influence the system's behavior.
The discrete state space of CTMCs makes them particularly suitable for modeling systems where changes can be categorized into distinct, countable states.

Continuous-time Markov chains, a specific class of jump processes, are defined by their discrete state space and memoryless properties in a continuous time domain \cite{van1976expansion, privault2013understanding, privault2022introduction}.
Mathematically, a CTMC is defined over a discrete state space of $S$ states (which can also take on an infinite value), typically denoted as $s \in \{1, 2, \ldots, S \}$.

The evolution of the process is governed by transition probabilities $p_{t+\Delta t|t}(y | x)$ from state $x$ at time $t$ to state $y$ at time $t+\Delta t$ that depend only on the current state, not on the history of the process, embodying the Markov property.
The transition rates in a CTMC are derived by taking limits of the transition probabilities over infinitesimally small time intervals, $\lim_{\Delta t \rightarrow 0} p_{t+\Delta t|t}(y \neq x | x) = r_t(y, x) \Delta t$.
In order to conserve probability mass over all transitions, the rate of staying in the current state $x$ is given by the total probability of 1 from which we subtract the probability of leaving $x$, resulting in $r_t(x, x) = -\sum_{y \neq x} r_t(y, x)$ \cite{privault2013understanding,privault2022introduction}.

The transition dynamics in a CTMC are encapsulated in its generator matrix $R$, a $S \times S$ matrix where each element $R_{ij}$ for $i \neq j$ represents the rate of transitioning from state $i$ to state $j$ at time $t$.
The diagonal elements of $R$ are defined such that the sum of each row is zero, ensuring the conservation of probability.
For a small time increment \( \Delta t \), the probability of transitioning from state \( i \) to state \( j \) in time \( \Delta t \) is approximately \( R_{ij} \Delta t \) for \( i \neq j \), and the probability of remaining in state \( i \) is \( R_{ii}\Delta t = - \sum_{i\neq j} R_{ij} \Delta t \).
By taking the time derivative it follows that $R_{ii} = -\sum_{i \neq j} R_{ij}$.
As \( \Delta t \) approaches zero, these expressions converge to the instantaneous rates of transition, reflecting the jump nature of the process.




\subsection{Master Equation for Jump Processes}

\label{sec: master equation jump process}

In the study of jump processes, particularly in the context of continuous-time Markov chains (CTMCs), the master equation describes the time evolution of the probability distribution over the states of the system.
The master equation provides an analytical equation that formalizes how probabilities flow between different states over time, making it a tool in the analysis of stochastic processes where transitions between states occur randomly but with well-defined rates between discrete states.
Thus the master equation serves a similar purpose in discrete spaces as does the FPE in continuous spaces \cite{berger2012introduction,privault2013understanding,privault2022introduction}.

Consider a jump process with a finite state space of $S$ states. The probability of the system being in state $x$ at time \(t\) is denoted by $p_t(x)$.
The dynamics of the probability of each state in an infinitesimal change of time $\Delta t$ is governed by the master equation, which is expressed for the probability of a particular state $x$ as

\begin{align}
      p_{t+\Delta t}(x)
      = & p_t(x) + \overbrace{\sum_{x \neq y} p_{t+\Delta t | t}(x| y) p_t(y) \Delta t}^{\text{inflow from } y} - \overbrace{\sum_{x \neq y} p_{t+\Delta t | t}(y | x) p_t(x) \Delta t}^{\text{outflow to } y}
\end{align}

Rearranging the terms and taking the limit in $\Delta t$ results in
\begin{align}
      \lim_{\Delta t \rightarrow 0} p_{t+\Delta t}(x) - p_t(x)
      = & \sum_{x \neq y} \overbrace{\lim_{\Delta t \rightarrow 0} p_{t+\Delta t | t}(x | y)}^{r_t(x, y)} p_t(y) \Delta t    \\
        & - \sum_{x \neq y} \underbrace{\lim_{\Delta t \rightarrow 0} p_{t+\Delta t | t}(y | x)}_{r_t(y, x)} p_t(x) \Delta t
\end{align}
and we obtain the master equation for the probability distribution of the discrete system in continuous time,
\begin{align}
      \frac{d p_t(x)}{dt} = & \sum_{x \neq y} r_t(x, y) p_t(y) - \sum_{x \neq y} r_t(y, x) p_t(x).
\end{align}

Writing this out in matrix-vector notation with the transition rate matrix $R$ and the vectorized probabilities $ p_t(x_t)= [p_t(x^{(1)})_t, p_t(x^{(2)})_t, \ldots, p_t(x^{(S)})]^T$, we have
\begin{align}
      \label{eq: master equation 2}
      \frac{d p_t(x)}{dt} = R \ p_t(x)
\end{align}


A more flexible family of jump processes can be obtained by making the rates time-dependent, ergo $R_t$.
The time transformation can be seen by looking at the master equation defined in \eqref{eq: master equation 2}, namely
\begin{align}
      \label{eq: master equation}
      \frac{d p_t(x)}{ d t} = R_t \ p_t(x)
\end{align}
where now the rate matrix $R_t$ is time-dependent. For simplicity, let us assume $R_t = \lambda_t R$, where $\lambda : [0, T] \to \R$ and $R$ is time-independent. We can now introduce the new time $\tau = \tau(t)$ and apply the chain rule to compute
\begin{align}
      \frac{d p_{\tau(t)}(x)}{dt} =\frac{ d p_\tau(x)}{ d \tau} \frac{ d \tau}{ d t} = \lambda_t R \ p_{\tau(t)}(x).
      \label{eq: discrete rates time transformation one}
\end{align}
Now, choosing $\frac{ d \tau}{ d t} = \lambda_t$ and thus $\tau(t) = \int_0^t \lambda_s ds$ (where we have assumed $\tau(0) = 0$), eliminates the time-dependent factor $\lambda_t$ and yields the equation
\begin{align}
      \frac{ d }{ d \tau} p_\tau(x)  = R \ p_\tau(x),
      \label{eq: discrete rates time transformation two}
\end{align}
and thus the master equation is constant in time.



\section{Reverse Time Stochastic Processes}

\subsection{Reverse Time Stochastic Differential Equations}

The addition of the stochastic, diffusive term $\diff$ to an ordinary differential equation denoted by $\drift$ induces a probability distribution over all possible solutions.
The evolution of the distribution of paths $p_t(x_t)$ is described by the Fokker-Planck partial differential equation.
Fundamentally, the time reversion of the FPE rests upon the same simple time reversion $\tau = T -t$, yet its analytical form takes on a more intricate term.
In the following, we will consider a modified stochastic process with a state independent diffusion term $\sigma^2(t)$ which for a great majority of used processes in this thesis is suitable.

For an Ito drift-diffusion SDE in the form of
\begin{equation}
      dx_t = \drift \, dt + \sigma(t) \, dW_t.
\end{equation}
we have the FPE
\begin{align}
      \frac{\partial p_t(x)}{\partial t} = -\frac{\partial}{\partial x}\left[\mu(x, t) p_t(x)\right] + \frac{1}{2}\frac{\partial^2}{\partial x^2}\left[\sigma(t) p_t(x)\right].
\end{align}
which describes the evolution of the probability distribution $p(x_t, t)$ over time of all possible trajectories of $x_t$.
For the interested reader, the time reversion of a stochastic differential equation is derived in its full length in the appendix \ref{app:cha2reversetimederivation}.

We have seen that a stochastic differential equation is intricately linked to the corresponding FPE.
For a forward time SDE and FPE with state independent diffusion $\sigma^2(t)$, the corresponding reverse time SDE and FPE \cite{anderson1982reverse, nelson1988stochastic, nelson1979connection, karras2022elucidating,grenander1994representations} read in their most general form read
\begin{align}
      dX_\tau = & ( - \mu(X_\tau, \tau) + \sigma_\tau^2 \left( \nicefrac{1}{2} + \alpha^2 \right) \nabla_x \log p_{\tau}(X_\tau) ) d\tau
      + \alpha \ \sigma_\tau dW_\tau
\end{align}
and
\begin{align}
      \frac{ \partial p_{\tau}(x)}{\partial t}  = & - \partial_x \Bigg[ \Big( - \mu(x, \tau)
      + \sigma_\tau^2 \left( \nicefrac{1}{2} + \alpha^2 \right) \nabla_x \log p_{\tau}(x) \Big) \ p_{\tau}(x) \Bigg]         \nonumber \\
                                                  & \quad  + \ \alpha^2 \ \sigma_\tau^2 \partial_x^2 \left[ \ p_{\tau}(x) \right]
\end{align}
where the $\alpha$ term is a free-to-choose auxiliary parameter which in fact denotes a family of reverse time stochastic processes.
Varying $\alpha$ will yields the same marginal probability distributions $p_\tau(x)$ but the paths of the stochastic process will differ as both the diffusion and the additional control term of the score are scaled.

Interestingly, choosing $\alpha=0$ simplifies the dynamics of $x_\tau$ from an SDE into a regular ODE, obviating the need for SDE solvers \cite{maoutsa2020interacting,song2020score}.
In this special case, the dynamics of the particle are defined by the ODE
\begin{align}
      dX_\tau = & \left( - \mu(X_\tau, \tau) + \frac{1}{2} \sigma_\tau^2 \nabla_x \log p_{\tau}(X_\tau) \right) d\tau
\end{align}
which is defined by the reversed drift $-\mu(X_\tau, \tau)$ and the score of the solution of the FPE of the forward process.
Thus simulating an infinite ensemble of particles with the SDE defined above will in reverse time yield the same stationary distributions as the forward process, i.e. $p_\tau(x)= p_t(x)$.
Whereas solving ODE's backward in reverse time necessitates a change of sign of the drift, the addition of a stochastic diffusion term $\sigma_t dW_t$ requires us to reverse diffusive characteristics with the score of the respective stationary distribution at the given time.

\subsection{Reverse Time Jump Processes}
\label{sec: reverse time jump processes}

In reverse time jump processes, the conventional time parameter \( t \) in a stochastic process \( \{X_t\}_{t \geq 0} \) is replaced by a transformed time variable \( \tau = T - t \), where \( T \) is a fixed terminal time \cite{berger2012introduction, anderson2012continuous,privault2022introduction,privault2013understanding}.
This transformation inverts the direction of time.
Under this particular time transformation, the dynamics of the process are reversed, and the system evolves backward in time, from the terminal time \( T \) to the initial time \( 0 \).
The mathematical intricacies lie in redefining the jump rates and transition probabilities in this reversed time framework.
One of the key challenges is the derivation of the reversed jump rates and transition dynamics such that the marginal distributions of the backward process $p_\tau (x)$ necessarily have to match those of the forward process.

Fundamentally, the transitions out of a state $x$ in the forward jump process with time index $t$ need to match the transition into the state in the backward jump process with the reverse time index $\tau = T -t$, namely
\begin{align}
      p_{t+\Delta t| t}(y | x) \ p_t(x) = p_{t | t+ \Delta t}(x | y) \ p_{t+ \Delta t}(y).
\end{align}
Taking the limit of the respective increments $\lim_{\Delta t \rightarrow 0}$ we obtain the transition rates
\begin{align}
      \forward{r}_t(y | x) \ p_t(x) = \backward{r}_t(x | y) \ p_t(y).
\end{align}
which yields under the necessary condition of matching marginals the backward rate
\begin{align}
      \backward{r}_t(x | y) = \frac{p_t(x)}{p_t(y)} \forward{r}_t(y| x).
\end{align}

Under the provision of the marginal probabilities $p_t(x)$ the reverse rates $\backward{r}_t$ can be obtained in a straightforward manner.
Given a terminal state $X_T$, we can simulate a CTMC with the according backward rates $\backward{r}_t$, and the resulting backward stochastic process will generate the same marginal distributions $p_t(x)$ to the forward process with rates $\forward{r}_t$.

In practice, obtaining the marginal distributions for systems of relevant size and complexity is often infeasible and computationally intractable.
For this reason, we can employ an alternative formulation of the reverse rates based on the provision of a conditional expectation $p_{0|t}(x_0|x_t)$ \cite{campbell2022continuous}.
We proceed by expressing the reverse rates in terms of the initial condition $x_0$ as follows
\begin{align}
      \backward{r}_t(x | y)
       & = \frac{\sum_{x_0} p_{t|0}(x|x_0) p_0(x_0)}{p_t(y)} \forward{r}_t(y | x)                                         \\
       & = \sum_{x_0}p_{t|0}(x|x_0) \underbrace{\frac{p_0(x_0)}{p(y)}}_{=\frac{p_{0|t}(x_0|y)}{p_{t|0}(y|x_0)}} \forward{r}_t(y | x) \\
       & = \sum_{x_0} \frac{ p_{t|0}(x|x_0)}{p_{t|0}(y|x_0)} p_{0|t}(x_0 | y)
      \forward{r}_t(y | x)                                                                                                                              \\
       & = \mathbb{E}_{p_{0|t}(x_0 | y)} \left[ \frac{ p_{t|0}(x|x_0)}{p_{t|0}(y|x_0)} \right]
      \forward{r}_t(y | x)
\end{align}

The equation above expresses the reverse rates in terms of the forward rates $\forward{r}_t$, the solution of the forward Kolmogorov equation $p_{t|0}(x|x_0)$ and the conditional expectation $p_{0|t}(x_0 | x)$.
This formulation in terms of an unknown conditional expectation allows for the computation of the reverse rates without the need for the marginal distributions $p_t(x)$ or $p_t(y)$.

\section{Boundary Value Problems}

The preceding sections have introduced the fundamental concepts of different types of dynamics that can govern dynamical systems.
We want to highlight the fact that these dynamics can be reversed in time.
For different dynamical systems, this has different implications.
Whereas deterministic systems governed by ODEs can be reversed quite easily in time, the reverse dynamics of stochastic systems require solutions of the corresponding FPE and Kolmogorov equations to accurately invert the respective stochastic processes.

We will now consider the problem of learning time-reversible dynamics of a system between two boundary conditions \cite{gakhov1990boundary, ladyzhenskaya2013boundary, wendt2008computational}.
Boundary value problems (BVPs) constitute a broad class of mathematical challenges that are pivotal in understanding and analyzing the behavior of dynamical systems governed by differential equations \cite{gakhov1990boundary}.
These problems are distinguished by the specification of values, or conditions, that the solution must satisfy at the boundaries of the domain over which the differential equation is defined.
Unlike initial value problems, where the conditions are given at a single point (typically the start), boundary value problems involve conditions at two or more points.

Given a time interval $t \in [a, b]$ and a spatial domain $x$, the boundary conditions are typically given as $g_a(x_a, a)=0$ and $g_b(x_b, b)=0$.
These problems can be expressed mathematically as

\begin{align}
      \frac{d x_t}{dt} = f(t, x_t), \quad t \in [a, b], \quad g_a(x_a, a) = 0, \quad g_b(x_b, b) = 0.
\end{align}
where \(x_t\) denotes the state of the system at time \(t\), \(f\) represents the dynamics of the system, and \(g\) encapsulates the boundary conditions.
Whereas $g_a$ and $g_b$ can encode intricate boundary conditions, we only consider Dirichlet boundary conditions of the simple form $g_a(x_a, a) = 0$ and $g_b(x_b, b) =0$ which are given by the ground truth data.

The underlying dynamical system between the boundary conditions is of central interest in the study of BVPs.
The dynamics $f$ of the boundary value problem in question can be ordinary differential equations (ODEs) for systems with a single independent variable (e.g., time), partial differential equations (PDEs) for systems with multiple independent variables (e.g., time and space) or stochastic differential equations (SDEs).
The dynamics between the boundary conditions are often complex and difficult to model, particularly in high-dimensional systems or systems with intricate interactions \cite{gakhov1990boundary,ladyzhenskaya2013boundary}.

In the context of this thesis, a boundary value problem is formulated as follows: given a flexible class of learnable differential equations $f(t, x_t)$ that describes the evolution or change of a system's state, along with boundary conditions of the system specifying the state of the system $x_t$ at the initial time $t=0$ and the final time $t=T$ given by ground truth data, the task is to find a particular parameterization of $f$ that satisfies both the differential equation and the boundary conditions.

In practical terms, this means that we aim to learn the differential equations governing the dynamics with data posing the boundary conditions.
Consequently, the neural network is tasked with learning the underlying dynamics of the system while respecting the boundary conditions, capturing the complex interactions and stochastic influences that govern the system's behavior.

The neural network approach involves constructing a network that takes the independent variables (e.g., time, space) as inputs and outputs a function that approximates the solution of the differential equation while satisfying the boundary conditions.
Training the network involves minimizing a loss function that captures the discrepancy between the network's outputs and the desired outcomes (i.e., adherence to the differential equation and boundary conditions).
This framework transforms the problem of solving a BVP into one of learning: the neural network learns the dynamics of the system between the boundary conditions by adjusting its parameters to minimize the discrepancy between the predicted and actual behaviors.


\subsection{Stochastic Bridges}

\begin{figure}[htb]
      \centering
      \includegraphics[width=0.8\textwidth]{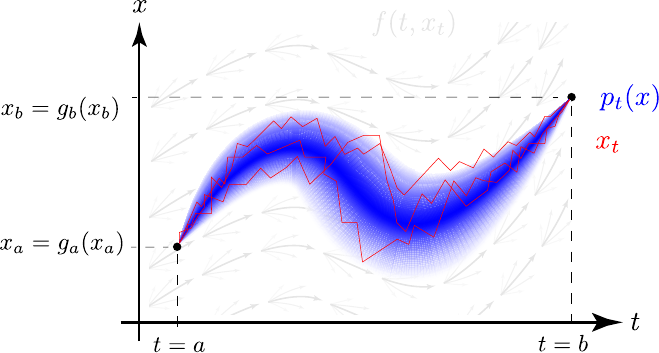}\\
      \caption{A figurative visualization of a boundary value problem with stochastic dynamics $f(t, x_t)$ and deterministic boundary conditions $g_a(x_a)$ and $g_b(x_b)$.}
      \label{fig:fundamentals bvp}
\end{figure}

In order to build up intuition for the Schr\"odinger Bridge Problem in the next section, we will introduce the concept of stochastic bridges.
A stochastic bridge is a type of boundary value problem where the dynamics are influenced by stochastic perturbations \cite{schrodinger1931umkehrung}.

In a stochastic setting, the dynamics between the boundary conditions are influenced by random fluctuations, which are captured by the stochastic differential equation \cite{nelson1979connection,nelson1966derivation}.
Whereas deterministic dynamics result in a particular solution in the form of a trajectory, the stochastic nature of SDEs results in a probability distribution of solutions, reflecting the system's behavior under uncertainty.
The distribution of solutions is governed by the Fokker-Planck equation, which describes the evolution of the probability density function over time through the formulation of the underlying stochastic differential equation.

The learning task is now to find a parameterization of the SDE that satisfies the differential equation and the boundary conditions, while also capturing the system's stochastic behavior.
Stochastic bridges formalize boundary value problems with non-deterministic dynamics and deterministic boundary conditions.

\section{The Schrödinger Bridge Problem}

Classical boundary value problems (BVPs) in physics and engineering typically involve determining a function that satisfies a differential equation subject to fixed conditions at the endpoints of a domain.
In many practical scenarios, especially those involving uncertainty and randomness, these deterministic models are inadequate.
By introducing stochasticity, we can more accurately model systems influenced by random fluctuations and noise, leading to the formulation of stochastic differential equations (SDEs) as a natural generalization of deterministic differential equations.

A stochastic boundary value problem extends this concept by not only considering the dynamics governed by an SDE but also requiring the solution to meet specific probability distributions at the start and end of the time interval.
This setup naturally leads into the framework of Schr\"odinger Bridges \cite{schrodinger1931umkehrung,vargas2021solving,chen2015stochastic,chen2016modeling,chen2016relation}, which aim to determine the dynamics of a stochastic process to achieve prescribed marginal distributions at the boundary times $t=0$ and $t=T$ while adhering closely to the original system dynamics.

Mathematically, the Schr\"odinger bridge problem is posited as constructing an SDE with fixed given diffusion $\sigma$ such that the probability densities of the corresponding state variables $X_{t=0}$ and $X_{t=1}$ at initial and final times, which for simplicity we take to be $t=0$ and $t=1$, coincide with given densities $\pi_0(x)$ and $\pi_1(x)$.
So far, we could choose any SDE which fulfills the given boundary conditions, but the Schr\"odinger bridge problem introduces an additional condition.
To make the problem unique, one imposes the additional constraint that the probability measure $\mathbb{P}$ over the corresponding paths of the stochastic process should be close to a given reference measure $\mathbb{Q}_0$.
The latter is itself defined by a drift function $\forward{\mu}^{Q_0}(x_t, t)$ and a given
\textit{initial} density $\pi_0(x)$.

If we define $\mathbb{D}(\pi_0, \pi_1)$ to be the set of probability measures over paths $\{X_t\}_{0\leq t \leq 1}$
with fixed marginal densities $\pi_0, \pi_1$, the measure $\mathbb{P}^*$ over paths of the SDE which solves the Schr\"odinger bridge is
defined by the solution of the minimization problem
\begin{align}
      \mathbb{P}^*  =  \arg\inf_{\mathbb{P} \in \mathbb{D}(\pi_0, \pi_1) }  \KLfunc{\mathbb{P}}{\mathbb{Q}_0} .
      \label{minim_Sbridge}
\end{align}
The explicit expression of the KL--divergence between two different path measures $\mathbb{P}$ and $\mathbb{Q}$ with the same diffusion parameter $\sigma$ induced by two SDE with drift functions $\forward{\mu}^P(x, t)$ and $\forward{\mu}^Q(x, t)$ and {\em initial} densities $\pi_0^P(x)$ and $\pi_0^Q(x)$ (for $X_{t=0}$) is
given by
\begin{align}
      \label{KL_FW}
      \KLfunc{\mathbb{P}}{\mathbb{Q}} = \KLfunc{\pi_0^P}{\pi_0^Q} + \frac{1}{2\sigma^2} \int_0^1
      \mathbb{E}_\mathbb{P}\left[\left(\forward{\mu}^P(x_t, t) - \forward{\mu}^Q(x_t, t)\right)^2\right] dt
\end{align}
where
\begin{align}
      \KLfunc{\pi_0^P}{\pi_0^Q} = \int \pi_0^P(x) \ln\left(\frac{\pi_0^P(x)}{\pi_0^Q(x)}\right) dx
\end{align}
denotes the usual KL--divergence between probability densities in $\mathbb{R}^D$.
Hence, the Schr\"odinger bridge problem can be understood as a problem of optimal stochastic control,
where one has to find a drift function $\forward{\mu}^{P^*}(x, t)$ as a state and time-dependent control variable which steers the stochastic dynamical system in such a way, that the marginal density of the state variable evolving form a given initial density reaches a predefined end density.
In addition, control variables are quadratically penalized by the KL divergence \eqref{KL_FW} to stay close on average to the drift of the reference system $\forward{\mu}^{Q} (x, t)$.

It should be noted that the reference measure $\mathbb{Q}_0$ is not unique and can be chosen in various ways.
As such, the reference measure need not fulfill the condition of matching the marginal densities $\pi_0$ and $\pi_1$.
Since the reference measure $\mathbb{Q}_0$ can be chosen, we choose it to be the measure of the forward process with the same diffusion parameter $\sigma$ and the drift function $\forward{\mu}^{Q_0}(x, t)$ with the initial distribution $\pi_0$.
This choice is motivated by the fact that the forward process is the most likely process to reach the final distribution $\pi_1$ from the initial distribution $\pi_0$.

Schr\"odinger bridges find extensive applications in fields like optimal control, optimal transport, and even fluid dynamics \cite{chen2016relation,chen2021optimal,berner2022optimal}.
In optimal control, Schr\"odinger bridges represent a strategy to control a stochastic system in such a way that it transitions between states in the most probable manner given the boundary constraints.
This concept closely relates to the principle of minimum energy or least action found in classical mechanics.

The derivation of $\forward{\mu}^{P^*}(x, t)$ can be derived by solving the optimal control problem under the constraint of the continuity equations of incompressible fluids \cite{chen2021optimal, chen2015stochastic}.
In this interpretation, the drift function $\forward{\mu}^{P^*}(x, t)$ represents the velocity field of the fluid, and the continuity equation ensures that the mass is conserved.
The analogy lies in the optimal rearrangement of fluid particles to achieve a desired flow pattern while minimizing energy dissipation, akin to the entropy-minimizing pathways in Schr\"odinger bridges.

In optimal transport \cite{villani2009optimal, peyre2019computational}, Schr\"odinger bridges provide a probabilistic method for transporting mass (or probability) from one distribution to another while minimizing a cost function related to the entropy of the transport plan \cite{villani2009optimal}.
This approach is particularly useful in economics and logistics where the cost of transportation is a critical factor.

Schr\"odinger bridges elegantly extend classical deterministic boundary value problems into the stochastic domain, offering a robust framework for managing uncertainty in dynamic systems.
By connecting the dots between stochastic processes, optimal control, and transport theory, Schr\"odinger bridges not only enrich our mathematical toolkit but also enhance our ability to engineer systems and processes that adeptly navigate the complexities of randomness and uncertainty.

\subsection{Schrödinger Half-Bridges as Generative Diffusion Models}

A special case of the Schr\"odinger bridge problem has garnered widespread interest in the machine learning community \cite{song2020score,ho2020denoising,song2020denoising}.
Known as diffusion models, this variant simplifies the original problem by carefully designing the reference dynamics $\forward{\mu}^Q$ to start from the initial distribution $\pi_0$ and subsequently choosing $\pi_1$ as the equilibrium distribution induced by the reference dynamics.
A careful choice of reference dynamics allows for the derivation of a closed-form solution to the half-bridge problem \cite{berner2022optimal,richter2023improved}.

In general, the optimal control solution can be derived from the gradient of the log densities of the marginal densities that characterize the reference path measure $\mathbb{Q}_0$ \cite{chen2016modeling,chen2015stochastic}.
Designing the reference dynamics to yield a path measure $\mathbb{Q}_0$ with analytical marginal densities, the solution, commonly referred to as the score of the marginal distribution, is known in its analytical form and can be used as a regression target.

Naturally, this necessitates the target distribution $\pi_1$ to be the equilibrium distribution of the reference dynamics, such that given sufficient time, the system will converge to the target distribution from any initial state from $\pi_0$.
Conveniently, the reference dynamics are commonly chosen deliberately such that the equilibrium distribution is an easy to sample from distribution such as the Normal distribution.

With these conditions fulfilled, we are then able to simulate the reverse-time stochastic process starting from the equilibrium distribution $\pi_1$ and ending in the initial distribution $\pi_0$.
Sampling this reverse-time process allows us to draw a new sample from the data distribution $\pi_0$, which is commonly chosen as a data distribution from which we want to draw more samples.

\section{Summary}
\Cref{cha:deterministicbvps} will focus on the learning of the dynamics of deterministic systems, with the ground truth dynamics provided by solutions of ODE's.
By formulating a neural network as a symplectic integrator, we learn time-reversible dynamics between the deterministic boundary values.

\Cref{cha:discretehalfbridge} will examine the applicability of neural networks modeling the dynamics of a reverse-time jump process such that the resulting process matches the empirical distribution of the boundary conditions.
Here, the dynamics are stochastic, and a reference stochastic process is provided in one direction between the boundary conditions.

Finally, \cref{cha:stochasticbridge} tackles the learning of continuous stochastic dynamics in the form of SDEs between two boundary values without the provision of any reference dynamics.
Two separate neural networks are trained recursively on each other to learn two stochastic processes that converge to the optimal control solution of the Schrödinger Bridge problem.

%% file: chapters/chapterMD.tex
\graphicspath{{./img/chapterMD}}

\chapter{Deterministic Boundary Value Problems in Molecular Dynamics}
\label{cha:deterministicbvps}

\begin{tcolorbox}[colback=gray!10!white, colframe=black]
      Parts of this chapter are mainly based on:
      \begin{itemize}
            \item \fullcite{winkler_2022}
      \end{itemize}
\end{tcolorbox}

In a dynamical system, the state of the system evolves according to the underlying ordinary differential equation integrated from an initial condition.
The provision of a final condition adds a further constraint to the solution of the ordinary differential equation and transforms the initial value problem into a boundary value problem.
Under the premise of time reversibility of dynamical systems, we can train the same neural network to provide a symplectic, first-order ODE solver.
This solver can be used to interpolate the trajectory of a molecular system over a finite time horizon by integrating the learned dynamics forward and backward in time.
Given a coarse simulation of the molecular system, the neural network can be used to reconstruct the missing trajectory segments at a higher resolution.

\section{Molecular dynamics}
Performing MD simulations of a molecular system in practice requires discretizing in time Newton's equations of motion $\ddot{a}=\nicefrac{F}{m}$.
The Velocity-Verlet algorithm is a popular choice for this discretization given a potential energy surface (PES) $U=U(r)$ which defines the force field $\mathbf{F}=-\nabla_{\mathbf{r}} U$ acting on each atom.

Hence, as a result, we obtain a trajectory with the molecular time evolution or, in other words, a time series of the atomic coordinates \textbf{r} and momentum \textbf{p}: $\mathcal{S}=\{ \mathbf{x}_t=(\mathbf{r}_t,\mathbf{p}_t); t=i\Delta\tau, i=0,\ldots,N_T \}$. Here $\mathbf{x}$ is known as a point in the phase space of the system.
The discretization parameter $\Delta \tau$ has to be selected according to the system and the simulation conditions.

In fact, it is key to choose it much smaller than the fastest oscillation period in the system.
Which, due to the nature of chemical bonds, renders $\Delta \tau$ in the order of 0.1 to 1 fs for organic systems.
Now, another fundamental aspect to generate meaningful simulations is the length of the time series or the total simulation time $T=N_T \cdot \Delta \tau$, which in contrast to the selection of $\Delta \tau$, has to be much larger than the slowest oscillation in the system.

\begin{figure}[tbp]
      \centering
      \includegraphics[width=0.9\textwidth]{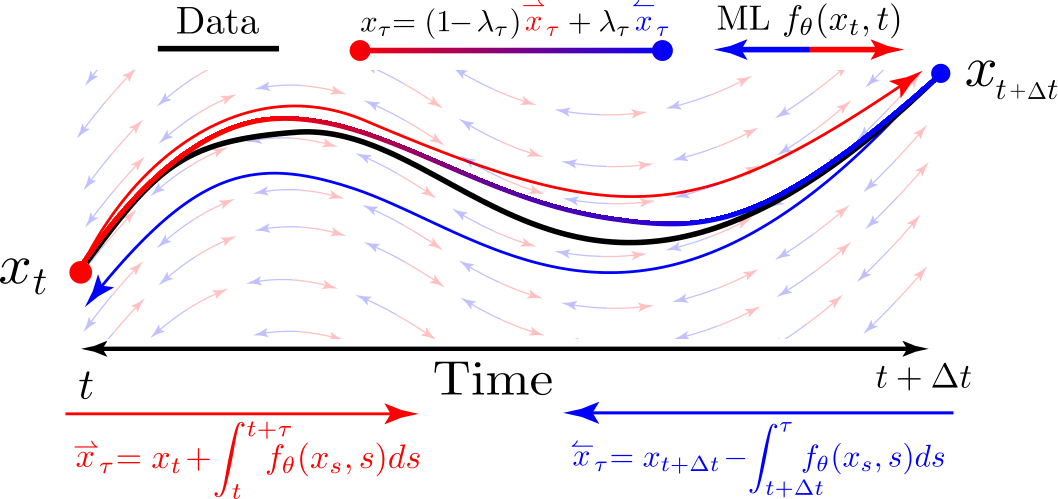}
      \caption{A figurative representation of the boundary values considered in this chapter.
            A time-reversible integrator $f_\theta(x_t, t)$ is learned with neural networks to reconstruct the solution $x_\tau$ with $\tau \in [t, t+\Delta t]$ of the boundary value problem by reconstructing the underlying data between $x_t$ and $x_{t+\Delta t}$.
            The interpolation $x_\tau$ is achieved by shifting the weight $\lambda_\tau$ from the forward trajectory $\smash{\stackrel{\rightharpoonup}{x}_\tau}$ to the backward trajectory $\smash{\stackrel{\leftharpoonup}{x}_\tau}$ over the course of the interpolation.}
      \label{fig:deterministic boundary value graphical abstract}
\end{figure}

According to the ergodic theorem, in order to fully recover the thermodynamical properties of the system $N_T\to\infty$, nevertheless in practice a careful selection of a finite value of these parameters gives converged thermodynamic averages.
Thereby, the selection of $N_T$ is a more abstract task because it is tightly correlated to the type of system and the physical phenomena to be studied.
For example, studying light diatomic molecules could require simulation times on the order of tens of fs ($N_T\approx10^3$), while more interesting molecules containing a couple of dozens of atoms and a fluxional molecular structure would require simulations on the order of ns ($N_T\approx10^7$)~\cite{sGDML_Appl_jcp,Sauceda_bigdml}.

Keeping in mind that every integration step during the MD simulation requires the explicit calculation of the potential energy and forces, which in the case of electronic structure calculations such as Density Functional Theory (DFT) can take on the order of seconds, we can see the benefit on creating a methodology that helps to reduce the value of $N_T$ without losing accuracy.

In a straightforward manner, this implies that we take larger integration steps in the simulation.


\section{MD Trajectory Interpolation with Neural Networks}

In this article, we propose to employ machine learning algorithms to integrate Newton's laws of motion directly from the phase space vector representation of molecules.
To that end, we train a machine learning integrator to interpolate the phase space trajectory of the molecules over a finite time horizon.
In general, $\mathbf{r}$, $\mathbf{p}\in\mathbb{R}^{3N}$ where $N$ is the number of atoms in the molecular system, but for the sake of simplicity here we will analyze the one-dimensional case, $r$, $p \in \mathbb{R}$ and later generalize to $3N$ dimensions.

A time-dependent variable can be described by the differential equation $\dot{x}_t = f(x_t, t)$ for a general time index $t$, which is derived from the phase space vector, which then has the discrete solution,
\begin{align}
      x_{t+\Delta t} = x_t + \int_{t}^{t+\Delta t} f(x_{\tau},\tau) d\tau.
      \label{eq:ODEsol}
\end{align}

Now, instead of performing a "big" jump by $\Delta t$, we would like to interpolate in between the two subsequent phase space vectors $x_{t}$ and $x_{t+\Delta t}$ to achieve a higher resolution.
Hence, we choose to integrate a dynamics described by, $\dot{x}_\tau = f(x_\tau, \tau)$ which has a higher temporal resolution $\tau$ between any two subsequent states in the simulation $x_t$ (i.e. $\tau\in [t,t+\Delta t]$).


Now, for better error control, we can make use of the fundamental property of time reversibility of Newton's equations.
We obtain the same trajectory if we start from the initial conditions $x_t$ and get to the final state $x_{t+\Delta t}$ and if we start from $x_{t+\Delta t}$ and propagate the system backwards in time to $x_t$.
With the provision of the initial condition $x_t$ and the final condition $x_{t+\Delta t}$, we can compute the forward $\overrightharp{x}_{\tau}$ and backward $\overleftharp{x}_\tau$ solution to the dynamics for $\tau \in \{ t, t+ \Delta t \}$,

\begin{align}
      \overrightharp{x}_{\tau} & = x_t + \int_{t}^{\tau} f_\theta (x_s, s) ds \label{eq:deterministic forward solution}                      \\
      \overleftharp{x}_{\tau}  & = x_{t+\Delta t} - \int_{t+\Delta t}^{\tau} f_\theta (x_s, s) ds \label{eq:deterministic backward solution}
\end{align}

The interpolation of the two trajectories is achieved via the interpolation parameter $\lambda_\tau \in [0,1], \tau \in [0, \Delta t]$,

\begin{align}\label{eq:lambda}
      \lambda_\tau = \frac{2 \int_{s=0}^{\tau} s \ ds}{\Delta t}
\end{align}

\noindent which is monotonically increasing for the duration $\tau \in [0, \Delta t]$ and is designed to shift the weight from the forward trajectory to the backward trajectory over the course of the interpolation.

\begin{align}
      x_\tau = (1 - \lambda_\tau) \overrightharp{x}_\tau + \lambda_\tau \overleftharp{x}_\tau.
\end{align}

This equation was inspired by the very insightful concept of thermodynamic integration from the free energy surface computation's toolbox~\cite{tuckerman2010statistical}.
The interpolation is visualized in Fig. \ref{fig:deterministic boundary value graphical abstract}, which highlights the shifting interpolation parameter $\lambda_\tau$.
We choose to model the dynamics in a time-reversible manner such that the forward and backward solutions are consistent with each other.

The notion of time-reversible integration of the dynamics can be summarized in the following proposition:
\begin{proposition}[Reconstruction of Time-Reversible Dynamics]
      \label{prop: deterministic time reversible dynamics}
      For a time-reversible dynamic $f(x_\tau, \tau)$ with the boundary conditions $x_t$ and $x_{t + \Delta t}$ on the time interval $[t, t + \Delta t]$, the solution $x_\tau$ at a time $\tau \in [t , t + \Delta t]$ is given by,
      \begin{align}
            x_\tau
            = (1 - \lambda_\tau) \left( x_t + \int_{t}^{\tau} f (x_s, s) ds \right)
            + \lambda_\tau \left( x_{t+\Delta t} - \int_{t+\Delta t}^{\tau} f (x_s, s) ds\right)
      \end{align}
      with the time-dependent weighting
      \begin{align}
            \lambda_\tau = \frac{2\int_{s=0}^\tau s ds}{\Delta t}
      \end{align}
\end{proposition}

\begin{proof}
      The proof is straightforward by inserting the definitions \cref{eq:deterministic forward solution} and \cref{eq:deterministic backward solution} for the solution at time $t + \tau$ into the equation and simplifying the terms:
      \begin{align}
            x_\tau
             & = (1 - \lambda_\tau) \underbrace{\left( x_t + \int_{t}^{\tau} f (x_s, s) ds \right)}_{=x_\tau}
            + \lambda_\tau \underbrace{\left( x_{t+\Delta t} - \int_{t+\Delta t}^{\tau} f (x_s, s) ds\right)}_{=x_\tau} \\
             & = (1 - \lambda_\tau) x_\tau + \lambda_\tau x_\tau                                                        \\
             & = x_\tau
      \end{align}
\end{proof}

\subsection{Learning the Dynamics}

Computing the forces acting on the atoms in a molecular system is a computationally expensive task, which is often the bottleneck in the simulation of molecular dynamics.
For this reason, we propose to learn the dynamics of the molecular system from the ground truth trajectories.
We choose a neural network $f_\theta(x_t, t)$ to learn the dynamics of the molecular system, which is trained to predict the time derivative of the phase space vector $\dot{x}_t \approx f_\theta(x_t, t)$.

The ground truth trajectories are sampled in regular steps of $\Delta \tau$ which can be subsampled to $\Delta t$.
This implies that $\Delta t= N \cdot \Delta \tau$ where $N$ is the subsampling factor.
The discretized solution of the ordinary differential equation is then given by the recursive application of the dynamics $f_\theta(x_t, t)$ to the phase space vector $x_t$,
\begin{align}
      \label{eq: discretized reconstruction}
       x_{t+\Delta t} = x_t + \sum_{\tau_i=t}^{\tau_N=t+\Delta t} f_\theta(x_{\tau_i}, \tau_i) \Delta \tau.
\end{align}

To infer the correct set of parameters $\theta$ of the integrator network we have to optimize the scalar loss function which reduces to the learning of the dynamics of the molecular system,
\begin{align}
      \label{eq: deterministic loss}
       & \ \quad \mathcal{L}(f_\theta(x_\tau, \tau), f(x_\tau, \tau), t, t+\Delta t) \\
       & =\frac{1}{\Delta t} \left\lVert x_0 + \int_{t}^{t+\Delta t} f_\theta(x_\tau, \tau) d\tau - \big(x_0 + \int_{t}^{t+\Delta t} f(x_\tau, \tau) d\tau \big) \right\rVert^2 \\
       & =\frac{1}{\Delta t} \left\lVert \int_{t}^{t+\Delta t} \bigg[ f_\theta(x_\tau, \tau) - f(x_\tau, \tau) \bigg] d\tau \right\rVert^2.
\end{align}

The loss function is evaluated on randomly sampled sections of the trajectories in mini batches.
This ensures that the integrator is able to predict the time derivative of the phase space vector at any point in time.
The minimum to the optimization problem \cref{eq: deterministic loss} is the model $f_\theta(x_\tau, \tau)$ which minimizes the distance to the true dynamics under the L2 norm.

The loss function is minimized with respect to the parameters $\theta$ of the integrator network using the backpropagation algorithm.
The solution to the boundary value problem is then obtained by integrating the learned dynamics forward and backward in time and interpolating the two solutions.

\subsubsection{Neural Network Integrators}

In order to assess which NN architecture is the most suitable to interpolate MD trajectories, we considered four NN integrators: neural ODEs, Hamiltonian neural networks, recurrent neural networks and long short-term memory (LSTM) \cite{hochreiter1997long}.
%

\noindent\textbf{Differential architectures}:
The most accessible approach to modeling the approximate dynamics $f_\theta$ is by directly computing the differentials with a fully-connected neural network.
The Euler discretization is often sufficient for convergence but more sophisticated, adaptive solvers are applicable such as Runge-Kutta and Dormand-Prince solvers for differential equations \cite{dormand1980family}.
In order to be applicable to deep neural networks trained with the backpropagation algorithms, these solvers require the use of the adjoint sensitivity method, which backpropagates an adjoint quantity as a surrogate gradient through time \cite{chen2018neural}.

The adjoint sensitivity method backpropagates the error through adaptive solvers with a constant memory cost, which is highly suitable for adaptive solvers with potentially large number of evaluations.
Once the adjoint is backpropagated to all evaluations, the gradients of the parameters $\theta$ can be obtained and gradient descent training is eligible.\\

Newton's equation for dynamical systems can be generalized into the Hamiltonian mechanics framework.
The canonical coordinate positions $r_t$ and momentum $p_t$ can be obtained through the partial derivatives of  the Hamiltonian $\mathcal{H}(r_t, p_t)$.
Hamiltonian neural networks~\cite{greydanus2019hamiltonian} predict an approximate Hamiltonian $\mathcal{H}_\theta \approx \mathcal{H}$ parameterized as a deep neural network and compute the time derivatives $\dot{r}_t$ and $\dot{p}_t$ during the forward pass.
In terms of functional analysis, one has to be careful in using an architecture which is differentiable at least twice, since the network is trained through backpropagation which requires a second differentiation of the model.
In practice, this amounts to using continuously differentiable activation functions such as tangent hyperbolic or sigmoid and refrain from using piece-wise linear activation functions such as rectified linear units.
By virtue of their construction, Hamiltonian networks exhibit energy conserving properties, such that the total energy of a dynamical system remains constant.
This is of interest for energy-based systems such as molecular dynamics, in which energy is shifted between potential and kinetic energy but never lost.\\

\noindent\textbf{Recurrent architectures}:
RNNs extend feed-forward neural networks through recurrent connections through time.
They offer the ability to explicitly model time dependent relationships by incorporating the neuron activations of the previous time step.
LSTM networks are recurrent architectures that resolve some important issues on general RNNs by using memory cells that can selectively read and write to them.
LSTM's are widely used in modelling time series and provide a remedy for the vanishing and exploding gradient problem of classical RNN's due to excessive or miniscule eigenvalues in the recurrent weight matrices \cite{pascanu2013difficulty}.

Whereas Ordinary Differential Networks and Hamiltonian Networks are considered Markovian in the sense that they only use the current state $x_t$ of the dynamical system to predict the time derivative $\dot{x}_t$, RNNs and LSTMs are capable of modelling long-distance dependencies through their hidden states in memory cells.
In the next section, we analyze in detail the performance of these two different approaches, as well as their bi-directional variants on the task of trajectory reconstruction.

The training setup is summarized in \cref{app:cha3training}.

\begin{figure*}[hbtp]
      \includegraphics[width=1.0\textwidth]{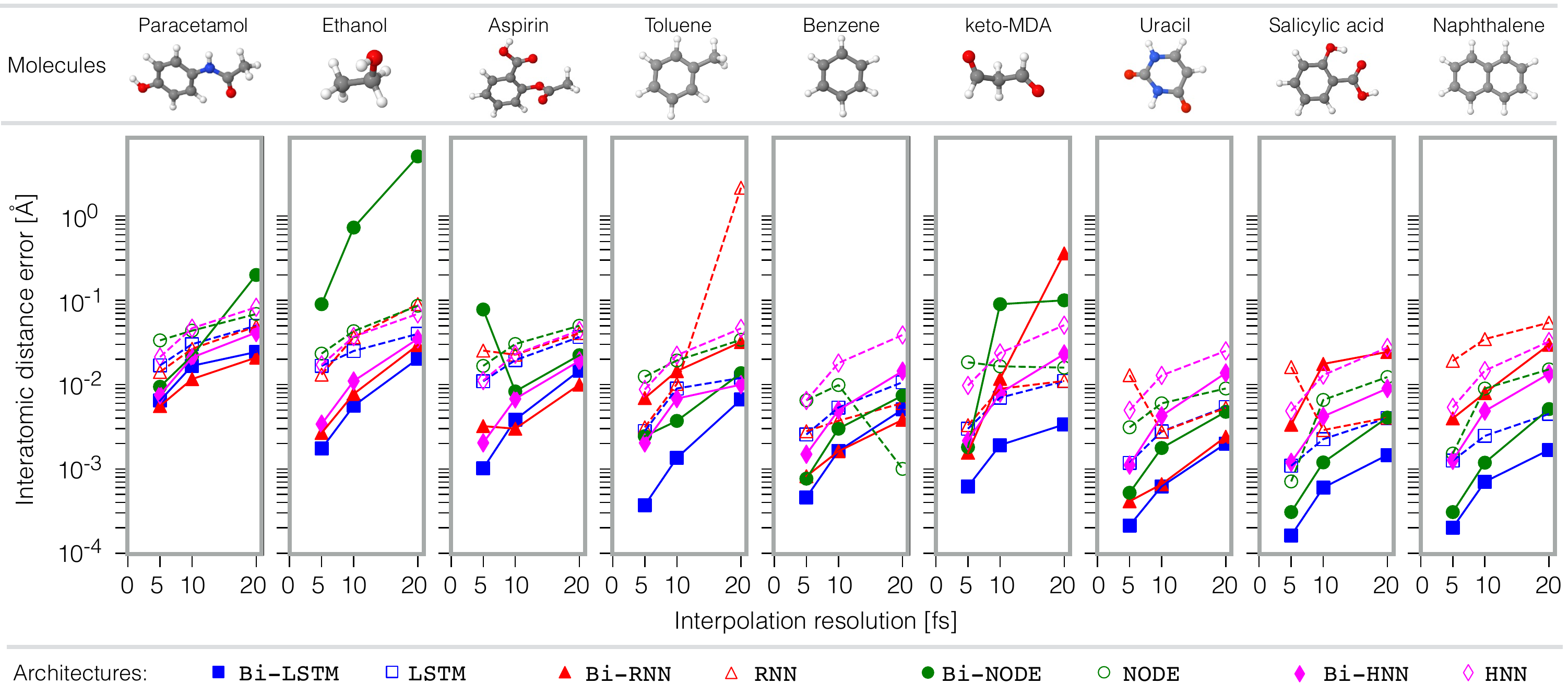}
      \caption{Performance comparison of uni- and bi-directional models over multiple integration lengths $\Delta t$ on the extended-MD17 dataset~\cite{chmiela2017machine, sGDML_Appl_jcp}. Uni-directional and bi-directional methods are represented by empty and full symbols, respectively. The considered interpolation resolutions are 5, 10, and 20 in femtoseconds.}
      \label{fig:performance}
\end{figure*}

\begin{figure*}[hbtp!]
      \includegraphics[width=\textwidth]{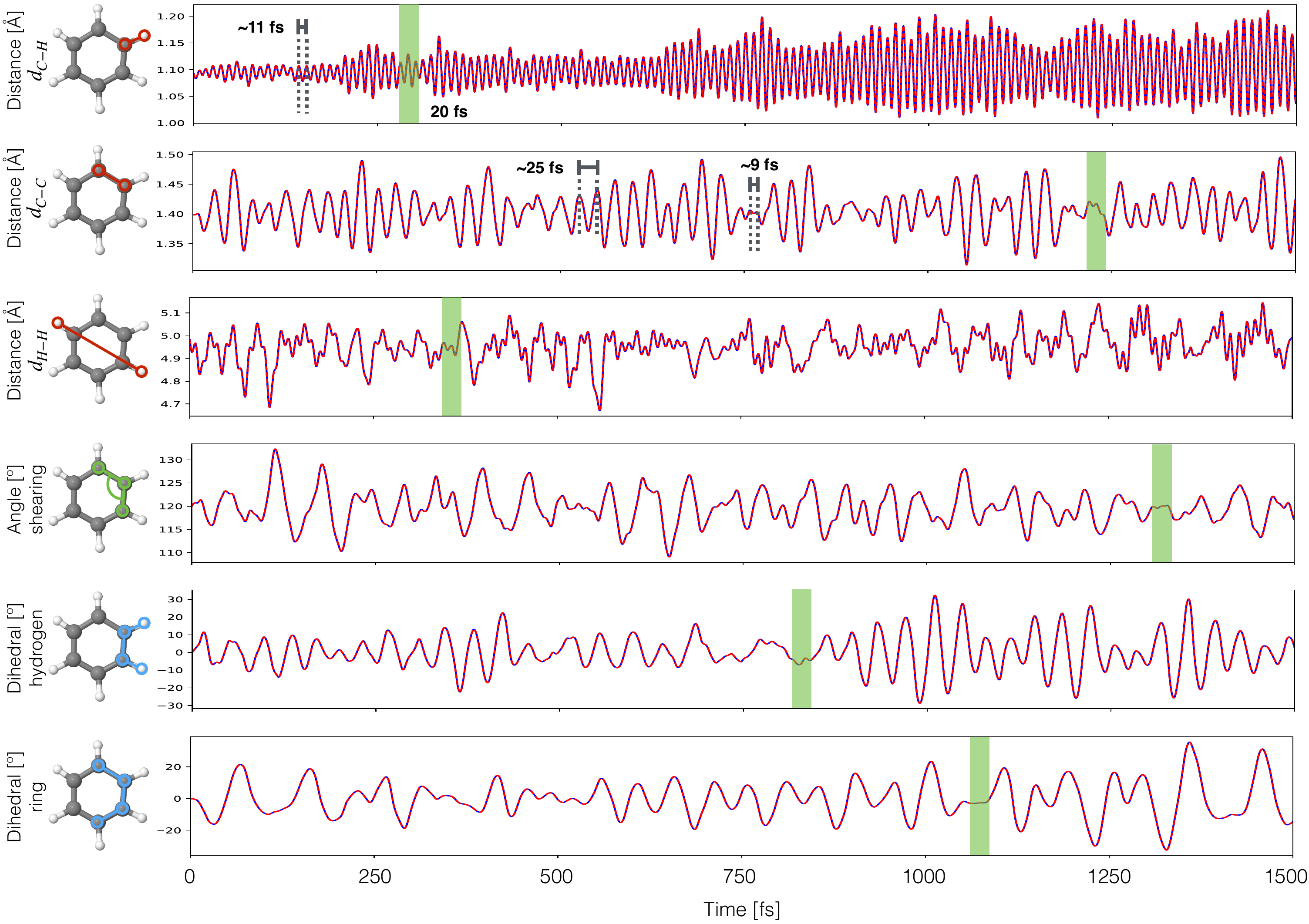}
      \caption{
            Analysis of the benzene dynamics and its interpolation for the 20 fs interpolation case.
            The six internal variables displayed are representative feature of the molecule, including interatomic distances (in \AA), internal angles (in degrees), and dihedral angles (in degrees) (represented in red, green and blue in the left-most molecular structure, respectively).
            The true dynamics is displayed by the blue curve and the predicted results are given by the dashed red lines. In grey dotted lines we show some relevant periods for some variables, and in a green window we represent the 20 fs interpolation interval.}
      \label{fig:trajBenzene}
\end{figure*}

\section{Results}

\subsection{Training and validation of the model}
The general goal of this study is to establish the applicability of the presented methodologies on realistic dynamics from physico-chemical simulations.
To this end, we have selected the well established extended-MD17 dataset~\cite{chmiela2017machine, sGDML_Appl_jcp} which contains middle-sized molecules with various dynamical complexities.
In this section, we analyze the performance of the bi-directional approaches (i.e. LSTM, RNN, HNN, and NODE) compared to their uni-directional counterparts while considering different resolution for the interpolation tasks.
These results are shown in Fig.~\ref{fig:performance}.

\subsubsection{Uni- vs Bi-directional NNs}
Common methods for sequence and trajectory learning are usually deployed using the natural arrow of time, nevertheless it has been demonstrated in other areas of machine learning, such as language modelling~\cite{vaswani2017attention, bahdanau2014neural, GoogleNMT}, that bi-directional learning substantially improve the results.
Here, we exploit the regularity of the physical trajectories and local time reversal symmetry by using bi-directional versions of the architectures.
In order to directly show the benefit of this approach, in Fig.~\ref{fig:performance} we compare the performance of the two approaches for all the molecules from the extended-MD17 dataset~\cite{chmiela2017machine}.
From this figure, we can quickly see that using bi-directional NNs can boost the accuracy of the model by up to two orders of magnitude.
This reveals that including both directions of the arrow of time goes beyond simply duplicating the amount of data.
Instead this strongly encodes the trajectory's regularity in the model which results in a more accurate description without increasing the number of parameters in the architecture.
Additionally, we can see that this phenomenon is not architecture specific, since all the tested NNs obtained a similar increase in accuracy and similar learning curves when considering bi-directionality.
Additionally, such behavior is also independent of the interpolation resolution.

\subsubsection{Bi-LSTM performance}
Now, comparing the different bi-directional methods, we can see that recurrent neural networks (i.e. RNN and LSTM) always perform better than neural ODE based approaches.
This is due to the fact that recurrent based NNs are more robust against noise, as will be the case when dealing with thermostated molecular dynamics simulations.
Furthermore, Fig.~\ref{fig:performance} shows that the Bi-LSTM architecture gives more consistent results across molecular structures with different complexities compared to the rest of the architectures, reaching interpolation errors of up to one order of magnitude lower.
This advantage of Bi-LSTMs over other architectures could be related to the fact that LSTM cells propagate information more efficiently and keep track of the many subtleties in the molecular trajectories for longer periods of time.
It is worth to remark that the accuracies reached by Bi-LSTMs in real space (i.e. coordinates) range on the order of 10$^{-3}$ \AA, which in practical terms is indistinguishable from the reference trajectory.
Given these results, from now on, we will focus our analysis only on Bi-LSTM networks.

\subsubsection{Molecular complexity and degrees of freedom}
As an additional observation, we have noticed that molecules with higher complexity (e.g. larger number of rotors and/or higher fluxionality) generate more intricate trajectories that render the reconstruction process more challenging.
This is only evident for the cases of paracetamol, ethanol and aspirin, nevertheless, there are cases where the correlation between interpolation accuracy and molecular morphology is non-intuitive.

For example, even though the keto-MDA molecule (see Fig.~\ref{fig:performance}) presents a complex PES with two rotors as main degrees of freedom~\cite{sgdml_bookAppl}, the reconstruction precision of its trajectory is similar to benzene and uracil, molecules that have no rotors and are formed only by an aromatic ring.
Furthermore, the best accuracy reached by the Bi-LSTM model was for salicylic acid, a molecule that has two rotors coupled by a complex hydrogen bond where the proton is dynamically being transferred between the two functional groups.

The physical interpretation of these results is that the reconstruction process occurs on the free energy surface (FES) of the molecular system, which means that the molecule is moving on its FES instead of the PES.
Hence, moving on the molecular FES at a given temperature can result on non-trivial behavior originated from the entropic contributions.
Thereby, a molecular system containing complex interaction such as salicylic acid which has a H--bond can generate a smoother dynamics compared to uracil or naphthalene (see Fig.~\ref{fig:performance} for molecular structure reference).

The insight behind these results is that, for the considered molecules,  thermal fluctuations considerably reduce the dynamical complexity of molecular systems with high fluxionality. Thereby, simplifying the learning process.

\subsection{Achieving super-resolution in MD trajectories}

In the previous section, we have demonstrated that Bi-LSTMs are a suitable architecture for trajectory learning of realistic molecular simulations.
In this section, we continue analyzing what it implies to reconstruct molecular trajectories, but now from the vibrational point of view, first from the normal frequencies framework and then from the real dynamics perspective.

One of the key aspects in trajectory interpolation is the time resolution that can be achieved by the model.
Given that the MD17 dataset has integration steps of 1 fs, here the task was to skip a number of frames $n$ for each molecule and then reconstruct the missing fragments on the trajectory.
Hence, for this dataset, $\Delta \tau=1$ fs and $\Delta t=n$ fs.
In this regard, Fig.~\ref{fig:performance} shows the performance of all the models for $\Delta t=$ 5, 10 and 20 fs.
Let's analyze the results for each value of $\Delta t$ in the context of its physical implications for the vibrational normal modes (i.e. harmonic analysis).
As a reference, the fastest atomic oscillation periods in paracetamol and in benzene are 
$\approx$9.3 fs and $\approx$11 fs, respectively~\cite{sGDMLsoftware2019}.
Such vibrations are mainly due to fast oscillations of the hydrogen atoms in the molecule.
Hence, sampling a trajectory every 5 fs in the context of molecular vibrations means that we are skipping half the period of the fastest oscillation in the molecule, which implies that, at most, the model has to interpolate half of the oscillation cycle.
This fact is reflected on the validations errors shown in Fig.~\ref{fig:performance}.
More challenging cases are sampling molecular trajectories every 10 and 20 fs, given that in these two cases the model has to reconstruct at least one full period (in the case of $\Delta t=$10 fs).

Furthermore, the most interesting resolution to analyze in more detail is $\Delta t=$20 fs, given that in such case the model has to reconstruct at least one full cycle of all the molecular oscillations with frequencies larger than $\approx$1600 cm$^{-1}$.
Which in the case of the benzene molecule, there are eight out of 30 normal modes with frequency values larger than such value (i.e. $\sim$27\% of the normal modes).
More interestingly, six of them have oscillation periods of $\approx$10 fs, meaning that the $\Delta t=$20 fs model has to reconstruct two full oscillations periods for the six vibrational modes using as inputs only the initial and final states (i.e. ($\mathbf{r}_t,\mathbf{p}_t$) and ($\mathbf{r}_{t+21},\mathbf{p}_{t+21}$)).
In the Supporting Information, we have animated such example for the toluene molecule case, where the true dynamics is represented by blue atoms and the interpolated dynamics appears in red.
As expected from Fig.~\ref{fig:performance}, the dynamics is practically indistinguishable.

In order to visualize this analysis, in Fig.~\ref{fig:trajBenzene} we present the reconstructed trajectory for the benzene molecule using the Bi-LSTM architecture for $\Delta t=$20 fs, and we show its dynamics in terms of its main internal degrees of freedom: interatomic distances, angles and dihedral angles.
In this figure, the blue curve is the ground truth and the red dashed line represents the predicted dynamics.
A quick glance over the plot shows that, even though the $\Delta t=$20 fs case in Fig.~\ref{fig:performance} is the one that gives the largest error, \textit{such accuracy still corresponds to an excellent agreement between the reference trajectory and the ML prediction.}

As mentioned before, the fastest oscillation in a molecule are due to hydrogen atoms oscillations, which in the benzene case it can be tracked by plotting the interatomic distance $d_{\text{C-H}}$ (Fig.~\ref{fig:trajBenzene} top panel).
Here, the measured oscillation period from the signal is $\approx$11 fs, slightly larger than the normal mode value.
The origin of this red shift is the fact that at finite temperatures, the system visits the anharmonic region of the PES, which then generate slower oscillations.
As a reference, a green rectangular window is used in Fig.~\ref{fig:trajBenzene} to show examples of interpolated intervals of trajectories.
Hence, in the case of $d_{\text{C-H}}$ we can clearly see that the method is successfully reconstructing two full cycles of the variable, rendering indistinguishable the comparison to the reference trajectory.
Another important internal variable is the first neighbor carbon-carbon distance $d_{\text{C-C}}$ (second panel in Fig.~\ref{fig:trajBenzene}).
$d_{\text{C-C}}$ oscillation period is in general considerably slower, but because such interatomic distance is part of many anharmonically-coupled vibrational modes,
it can create apparent high-frequency diatomic oscillations (as short as $\approx$9 fs).
Here, despite the complex dynamics of this internal variable, the model manages to faithfully recover its behavior.
Now, in order to incorporate a highly non-linear and weakly correlated interatomic distance, we considered two opposite hydrogen atoms, $d_{\text{H-H}}$ (third panel in Fig.~\ref{fig:trajBenzene}).
In principle, this variable should amplify small errors in the reconstructed trajectory, nevertheless the results are still in excellent agreement.

Another important aspect in molecular fluctuations is the analysis of internal shearing deformations as well as out of plane deformations, which represent a global mechanical property of the system.
These mechanical deformations can be analyzed by measuring shearing angles and dihedral angles as shown in the lower half of Fig.~\ref{fig:trajBenzene}.
The oscillations in these variables are considerably slower compared to interatomic distance fluctuations, but they can contain highly anharmonic contributions, making them a good measure for reconstruction accuracy.
Again, here the interpolation accuracy of the model is very high, which makes the predicted trajectories indistinguishable from the reference curve.
In fact, the MAE reconstruction accuracy of the Bi-LSTM model for the benzene molecule's trajectory is $\sim$10$^{-4}$~\AA, a value that is close to the accuracy of the actual numerical integrators such as the Verlet algorithm.

The results obtained in this section show that the methods presented here, and in particular Bi-LSTMs, give excellent trajectories' reconstruction accuracies, rendering their predictions indistinguishable from full \textit{ab-initio} MD results.
Hence, this allows to confidently use these techniques in a diverse range of applications such as data augmentation, super-resolution generation, or even storage capacity reduction.
In the next section, we use Bi-LSTMs to extract some physical insights from learning trajectories at different temperatures, where we demonstrate that we can faithfully recover its molecular free energy surface.

\begin{figure}[htbp]
      \centering
      \includegraphics[width=0.66\textwidth]{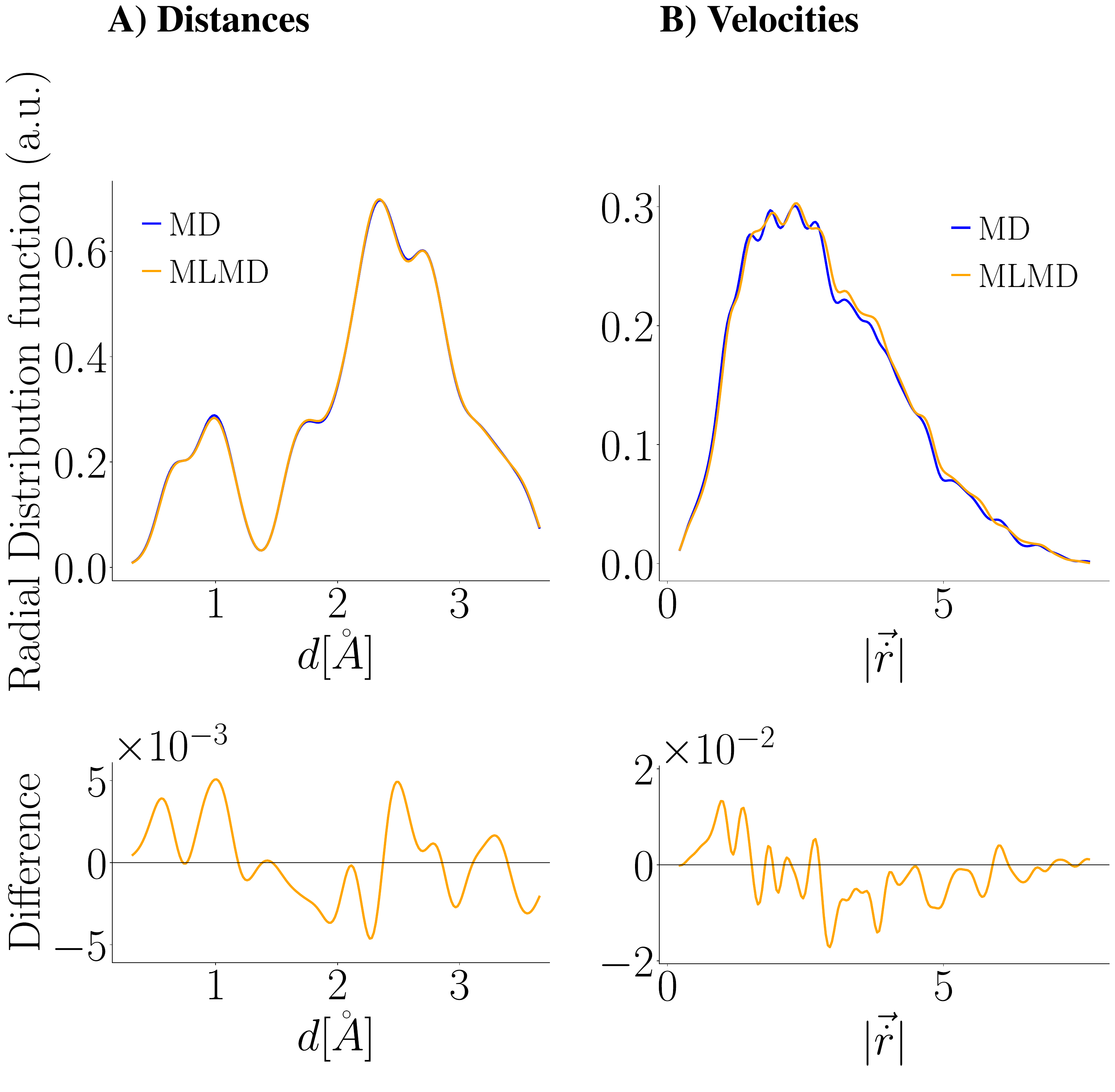}
      \caption{Comparison of the interpolated (yellow) and ground truth (blue) trajectories for keto-MDA molecule. A) Interatomic distance histogram and B) speed histogram.
            The interpolation was done with a Bi-LSTM integrator with $\Delta t = 10$. The bottom panels provide the difference between the two curves.}
      \label{fig:HistMDADistVel}
\end{figure}


\subsection{Physical insights from the reconstruction process: Application to keto-MDA as case of study}

The keto-MDA molecule is a challenging molecule given that it presets a wide variety of interactions that generate a PES with strong electrostatic interactions but also regions with particle-in-a-box behavior~\cite{sauceda2020JCP,sgdml_bookAppl}.
Hence, such intricate energy landscape generates interesting dynamics from which, for example, force field reconstruction is not a straightforward task and requires the use of sophisticated ML methodologies~\cite{chmiela2017machine, SchNet2018,painn,unke2021spookynet}.
Interestingly, there is evidence that the force field learning process gets easier when the temperature of the generated training data gets lower~\cite{sgdml_bookAppl}, given that, in general, they explore smaller regions of the configurational space.

From the statistical point of view, the MD trajectories at a given temperature sample the configuration space in such a way that the molecular free energy surface can be estimated by $\sim \text{ln}~P(\theta_1,\theta_2)$ where $P(\theta_1,\theta_2)$ is the integrated configurational probability density, and $\theta_1$ and $\theta_2$ are the main degrees of freedom of the system (see Fig.~\ref{fig:HistkMDA_distTemps}).
In general, the FES is known to get smoother as we increase the temperature due to entropic effects.
This is represented in Fig.~\ref{fig:FESvsPES}, where we can see that as the temperature increases, the curved regions of the FES get smoother until they progressively flatten.
This also means that the molecular trajectory in phase space becomes more stable, and hence it should be easier to reconstruct.
In order to get some insights about this effect as well as to further validate the performance of our models, we have run three simulations of the keto-MDA molecule at different temperatures (100K, 300K and 500K) using a pretrained sGDML model~\cite{sGDMLsoftware2019} and analyzed the generated trajectories.
In Table~\ref{table:TempTrajLearn_MAEdist_normalized}, we summarize the results for the learning procedure for the three temperatures.

\input{img/chapterMD/tables/table_kMDA_diffTemp_BiLSTM_normalized_distVel}

\subsubsection{Phase-space histograms}
The frameworks presented in this study work on phase space, $(r,p)$, meaning that during the reconstruction task the trajectory and the velocity field are recovered. The velocity field refers to the atomic velocities.
Hence, here, we analyze Bi-LSTM's predictive power in terms of physically meaningful distributions from molecular dynamics simulations by examining the speed and interatomic distances distributions.
In previous sections, we have focused on spacial accuracy for trajectory reconstruction (up to $\sim10^{-4}$ \AA), but another important aspect is to recover an accurate velocity field.
In Table~\ref{table:TempTrajLearn_MAEdist_normalized}, we report the interpolation accuracy for both of these variables, displaying an excellent agreement with the reference data.
In particular, in Fig.~\ref{fig:HistMDADistVel} we show the explicit comparison of the reconstructed probability distributions for the interatomic pair distance distribution function, $h(r)$, and the speed distribution function
for the case of keto-MDA's trajectory at 300K.
From figures~\ref{fig:HistMDADistVel} and~\ref{fig:HistkMDA_distTemps} and from Table~\ref{table:TempTrajLearn_MAEdist_normalized}, we can see that the acquired velocity field precision is maintained for the three different temperatures.

\subsubsection{Temperature dependent learning}
It is worth noticing that the temperature has indeed an effect during the learning process.
From the dynamical point of view, we know that the free energy surface depends on the temperature, meaning that the trajectory generated will follow different patterns and statistically sample differently the phase-space.
A pictorial representation of this effect is shown in Fig.~\ref{fig:FESvsPES}.
Fig.~\ref{fig:HistkMDA_distTemps} shows this effect for the keto-MDA molecule, where the left panels show the sampling for the different temperatures and the right panels show the pair distance distribution.

From here, we can see how the pair distance distribution function evolves from a multimodal histogram at 100K to a less complex function as the temperature increases.
In terms of trajectory reconstruction, from Table~\ref{table:TempTrajLearn_MAEdist_normalized}, we can quickly see that when increasing the temperature of the reference data, the accuracy in the interatomic distances marginally decreases but the accuracy of the velocity field reconstruction considerably increases.
Furthermore, in the case of the Bi-LSTM model with $\Delta t=20$, the prediction accuracy increases roughly by a factor of 4 when training on data generated at 500K relative to the 100K case.

In a broader picture of the physical problem, these results tell us that the generated dynamics at higher temperatures are smoother even though the atomic speeds are higher.
This is actually because the molecular system spends more time in anharmonic regions, which then generates the well-known frequency red shift of most of the vibrational frequencies~\cite{sGDML_Appl_jcp}.
In other words, increasing the system's temperature reduces the oscillation periods and generates larger oscillation amplitudes, thereby reducing the complexity of the learning problem.

\begin{figure}[htbp]
      \centering
      \includegraphics[height=0.8\textheight]{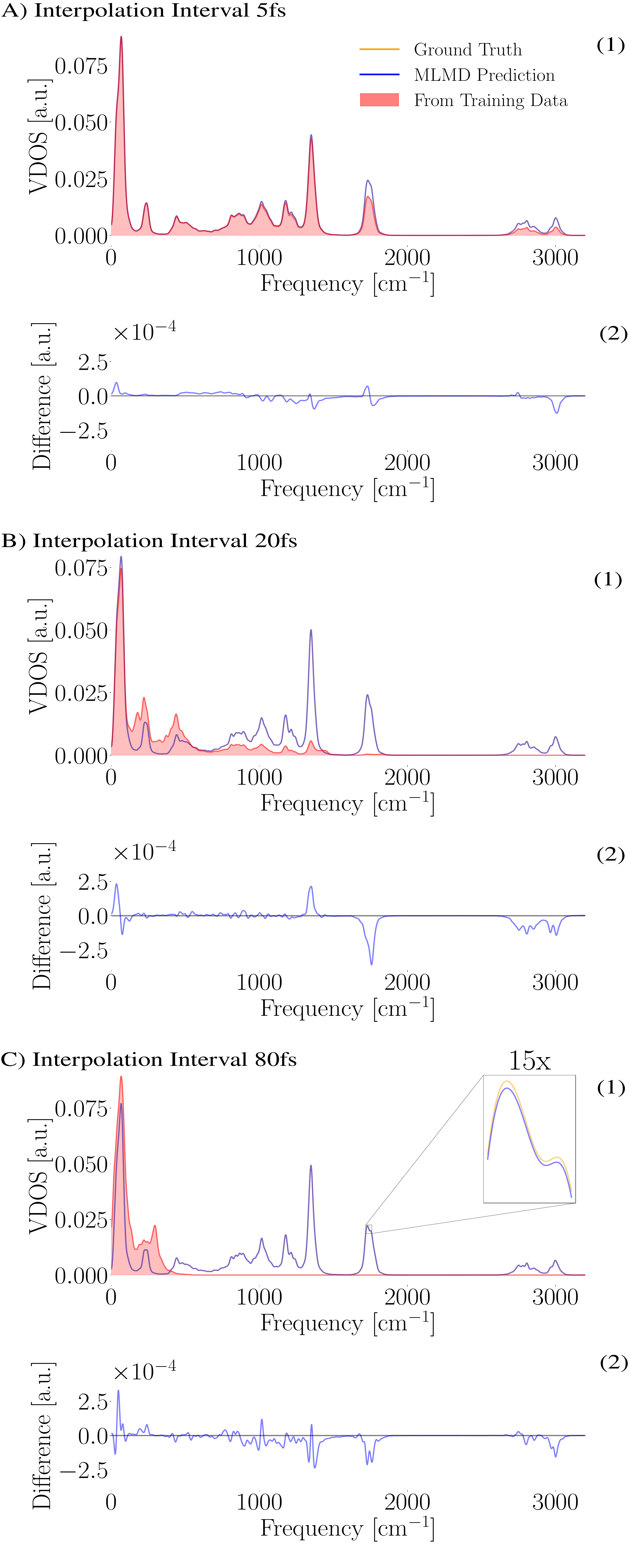}
      \caption{Reconstruction of the vibrational density of states (VDOS) using the interpolation intervals: A) 5fs, B) 20 fs, and C) 80 fs. In panel (1) are the ground truth VDOS (yellow), the predicted VDOS (blue), and the VDOS obtained from the training data (shaded red). In Panel (2) is the differences between the ground truth and predicted VDOSs.}
      \label{fig:VDOS_diffInter}
\end{figure}

\subsubsection{Free energy surface reconstruction}

To complement the results from the previous sections, here we correlate those results to the free energy surface (FES).
The FES can be estimated by integrating the probability distribution generated by the molecular trajectories, hence smoother trajectories in principle should lead to less complex FES as shown in Fig.~\ref{fig:HistkMDA_distTemps}.
Now, this actually has a beneficial connection to the learning problem, because given the previous analysis, we can accurately reconstruct the molecular FES from a reduced amount of data or sparse trajectories.
In order to corroborate such accuracy, in Fig.~\ref{fig:Malondialdehyde_DiffHist} we show a comparison between the reference and Bi-LSTM-predicted configurations' sampling for the keto-MDA molecule for its main degrees of freedom i.e. the two aldehyde dihedral angles.
As mentioned before, the accuracy in trajectory reconstruction is $\sim10^{-4}$ \AA, hence renders the differences between the two FES  barely noticeable.

\subsubsection{Testing the limits of the methodology: Larger interpolation intervals}

In order to assess the accuracy of the Bi-LSTM models as a function of the sampling frequency, in Fig.~\ref{fig:VDOS_diffInter} we show their different models with an interpolation capacity of 5, 20 and 80 fs, where the resolution of the training data is 1 fs.
This means that during training, the 80 fs model for example, was trained using time series sampled every 80 fs, given the initial and final conditions, the model will predict the intermediate 78 points in the series.
If we analyze the information encoded in a time series sampled every 80 fs, we can clearly see that the training data set only contain information regarding molecular oscillations with frequencies of $\sim$500 cm$^{-1}$ at most, and the rest of the spectrum is not explicitly embedded in the data.
The frequency spectrum shown to each model during the training process is highlighted in red in Fig.~\ref{fig:VDOS_diffInter}.
From this figure we obtain that in the three cases, the reconstructed vibrational spectrum is indistinguishable from the ground truth calculations (see bottom plot in each of the panels), despite the wide variability of the interpolation parameter.
Nevertheless, in the 80 fs case (inset in Fig.~\ref{fig:VDOS_diffInter}-C), we can appreciate that some slight deviations start to appear.

These results immediately bring up the question of how the model can still reconstruct very high frequencies when these are not explicitly given.
The answer lies on the memory cells of the LSTM architectures.
It doesn't matter that high frequencies are not shown to the model, since in each pass of the training process, the model is getting different configurations that sample the high frequency normal modes.
Hence, by keeping in memory all that information, the model can infer the existence of all the vibrations in the system.

From the results in this section and the previous one, we obtain that bi-LSTM models not only manage to accurately reconstruct the spacial components of the trajectories, but also the velocity field.
Thereby, ensuring their application to reconstruct the phase-space dynamics as well as the physical properties such as free energies and vibrational spectra.

\section{Discussion and Limitations}

In this work we have introduced the concept of learning the dynamics of a boundary value problem with deterministic boundary conditions with a series of bi-directional NN integrators leveraging different levels of prior physical knowledge of the underlying dynamics (i.e. NODE, HNN, RNN and LSTM) to increase the resolution of MD trajectories for a variety of molecular systems.
An extensive validation process of these models on the MD17 dataset and different interpolation resolutions were presented (see figure 3).
Interestingly, we found that the physically most agnostic model, the Bi-LSTM model, is the better performing method generating more stable results across all the extended-MD17 dataset and better accuracies, reaching errors as low as
$\sim 10^{-4}$ \AA.
These errors render the interpolated and ground truth trajectories indistinguishable.
This is a remarkable result, given that the Bi-LSTM model is able to learn the full spectrum of molecular vibrational frequencies (harmonic and anharmonic) even though this information is not explicitly shown during the training process.
The overall higher performance of Bi-LSTM networks is due to the robustness against noisy (thermostated) reference data, as well as their capacity to retain long time correlations’ during the learning process.
The higher flexibility of Bi-LSTMs leading to improved performance corroborates a wider trend in machine learning that more data and higher flexibility in the model architecture leads to better performance, even in contrast to prior knowledge.

Beyond the machine learning task of high dimensional trajectory interpolation, by varying the temperature of the reference data used for training the models, we obtained important physical insights regarding the dynamical behavior of molecular systems.
There is evidence in the literature that learning molecular force fields (i.e. the underlying PES) gets more complicated as the temperature of the training data increases.
Contrasting this behavior, here we have observed that trajectory interpolation becomes easier as the temperature increases.
The origin of such different temperature-dependent learning behavior is the fact that the trajectories generated by MD simulations effectively move on the Helmholtz FES due to entropic contributions. Additionally, going from low to high temperatures progressively pushes the FES’ landscape to deviate from the underlying PES, reducing its complexity as the temperature increases. From these results we can summarize that the temperature of MD databases have opposite effects on learning the potential and FESs.
Additionally, we found that Bi-LSTMs have the remarkable property of being able to learn the full spectrum of molecular vibrational frequencies (harmonic and anharmonic) even though this information is not explicitly shown during the training process.
This capability is due to their memory cells.

From the results obtained in this work, we have evinced the great learning capacity of Bi-LSTMs on the reconstruction of realistic high dimensional molecular behavior. This opens up a new set of applications for the family of RNNs on post-processing MD results. Furthermore, the results here presented can be used to set the stage for robust extrapolation techniques, either by using it as a data enhancement method for forecasting or as the basis for numerical propagator.

The aim of molecular dynamics is to obtain the macroscopic properties of a system from the microscopic interactions of its constituents.
With this goal in mind, this work is only a stepping stone towards accelerated, data-driven MD simulations.
The results presented here are a proof of concept that the dynamics of molecular systems can be learned and predicted with high accuracy.
Yet, by construction of the boundary value problem, these models are not able to predict the future states of the system but rely on the second boundary condition.
Empirically, we have seen that framing molecular dynamics as a boundary value problem instead of a initial value problem improved the performance of the models by two orders of magnitude.
Thus, the next step in this line of research is to develop a model that can predict the future states of the system \emph{with high accuracy}, which would be a significant step towards the development of a data-driven molecular dynamics simulator.

%% file: img/chapterMD/tables/table_kMDA_diffTemp_BiLSTM_normalized_distVel.tex
\begin{table}[!ht]
    \centering
\begin{tabular}{c | *3c}
    \toprule
    \multicolumn{4}{c}{\textbf{A)} Interatomic distances [$10^{-3}$ \AA]}\\
    \cline{1-4}
    Temperature [K] & $T_{\text{tr}}=$5 & $T_{\text{tr}}=$10 & $T_{\text{tr}}=$20        \\
    \cline{1-4}
    100         & 0.259     & 	0.930  &    3.287  \\
    \cline{1-4}
    300         & 0.417    & 	1.500   &   5.321 \\
    \cline{1-4}
    500         & 0.501    & 	1.801   & 	6.399 \\
    \midrule
    \multicolumn{4}{c}{\vspace{0.01cm}} \\
    \toprule
    \multicolumn{4}{c}{\textbf{B)} Velocity field [$10^{-3}$ \AA/fs]}\\
    \cline{1-4}
     Temperature [K] & $T_{\text{tr}}=$5 & $T_{\text{tr}}=$10& $T_{\text{tr}}=$20        \\
    \cline{1-4}
    100         & 11.7100 	 & 38.390 	 & 132.700  	 \\
    \cline{1-4}
    300         & 11.320 	 & 36.820  	 & 127.210 	 \\
    \cline{1-4}
    500        & 10.380 	 & 33.070 	 & 113.60 	 \\
    \midrule
\end{tabular}
	\caption{Experimental results for evaluating the influence of varying training integration time $T_{tr}$ of the Bi-LSTM architecture on A) inter-atomic distances and B) velocity field. This was done for the keto-MDA molecule as a case of study for simulations run at three different temperatures. The interatomic distances and velocity field were normalized to follow the first two moments of a standard Normal distribution.}
	\label{table:TempTrajLearn_MAEdist_normalized}
\end{table}

%% file: chapters/chapterDiscDiff.tex
\chapter{Stochastic Half-Bridges with the Ehrenfest Process}
\label{cha:discretehalfbridge}
\graphicspath{{./img/chapterDiscDiff}}

\begin{tcolorbox}[colback=gray!10!white, colframe=black]
  Parts of this chapter are mainly based on:
  \begin{itemize}
    \item \fullcite{winkler2024ehrenfest}
  \end{itemize}
\end{tcolorbox}

In the previous chapter we considered the inference of the deterministic dynamics of a boundary value problem constrained by deterministic conditions with time-reversible ground truth solutions to the differential equations.
In this chapter, we will take (half) a step into the stochastic realm.
We will infer the dynamics of a given stochastic process in reverse time.
Equally important, we will consider stochastic boundary conditions, represented by an initial and a final distribution.
Compared to the previous chapter, the boundary conditions are stochastic and only one direction of the stochastic dynamics is known.

The time reversion of stochastic processes has garnered wide-spread interest in the machine learning community, where it has been popularized as \textit{generative diffusion models} \cite{ho2020denoising,sohl2015deep}.
In terms of boundary value problems, the generative diffusion models can be interpreted as learning the reverse-time dynamics of a forward stochastic process, where the initial distribution is the data distribution and the final distribution is a tractable and easy to evaluate distribution.
In generative diffusion models, the forward dynamics are usually provided by a tractable stochastic differential equation (SDE).
The forward dynamics transform the data over a sequence of time steps by gradually adding noise, transforming the data into the final equilibrium distribution that is typically Gaussian.
Given sufficient time, any sample from the initial data distribution is diffused into the same equilibrium distribution and can not be distinguished from a different, equally sufficient diffused sample.
The goal in terms of machine learning is to infer the reverse-time dynamics to iteratively transform a a sample of pure noise into a sample seemingly coming from the data distribution.

\begin{figure}[btp]
  \centering
  \includegraphics[width=0.75\textwidth]{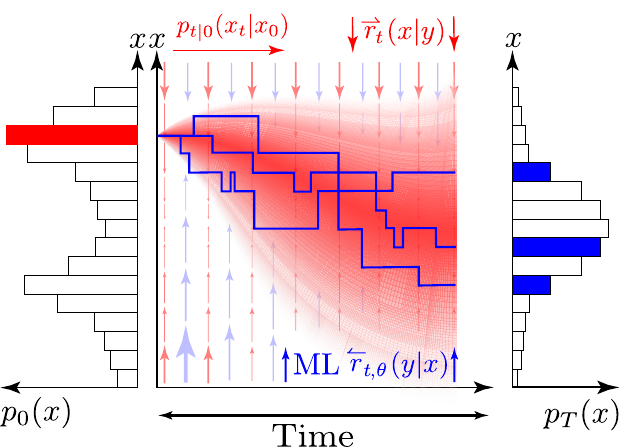}
  \caption{A figurative representation of the boundary value problem considered in this chapter.
    The conditional stochastic process $p_{t|0}(x|x_0)$ transforms samples from the initial distribution $p_{0}(x_0)$ into samples of the final distribution $p_{t}(x)$ with the stochastic jump dynamics $\smash{\stackrel{\rightharpoonup}{r}_t(x|y)}$.
    The aim is to learn the time-reversed dynamics $\smash{\stackrel{\leftharpoonup}{r}_{t, \theta}(y|x)}$ to revert the original dynamics and transforms samples from the final distribution $p_T(x)$ into samples of the initial distribution $p_0(x)$.}
  \label{fig:half bridge graphical abstract}
\end{figure}

Unlike traditional BVPs that are solved with respect to specific differential equations with predetermined boundary conditions, in generative diffusion models, the reverse-time dynamics are not known a priori and must be learned.
The mathematical formulation underlying diffusion models, particularly the continuous-time formulation, draws from the theory of stochastic calculus and involves solving reverse-time stochastic differential equations. This formulation aligns with the idea of solving a BVP in reverse time, where the aim is to reconstruct the original data from its diffused state by learning the appropriate reverse dynamics.

In the stochastic processes community, this settings is known as \emph{stochastic half-bridges} where the ends of the stochastic bridges are probability distributions and one direction of the dynamics is known.
Here, we will study stochastic half bridges in discrete state spaces.
We will consider the stochastic dynamics of a particular forward process, the Ehrenfest process, and study the learning of the time-reversed dynamics.
Importantly, the Ehrenfest process is a birth-death process, which is particularly suitable for generative modeling of discrete data with an ordinal structure.

Equally consequential, it can be shown that the Ehrenfest process converges to the Ornstein-Uhlenbeck process in the infinite state limit.
With the tools of stochastic analysis, the score function of the Ornstein-Uhlenbeck process can be linked to the time-reversed transition probability of the Ehrenfest process.

The main practical insight from this chapter is that we can perform \emph{state-space transfer learning}, stemming from our theoretical analysis of the Ehrenfest process.
In the infinite state space limit, we can learn the reverse-time dynamics of the Ehrenfest process with an Ornstein-Uhlenbeck process.
The OU process has been used ubiquitously in generative diffusion models and is the go-to forward process for a wide variety of generative diffusion models.

The advantages are twofold.
For one, we can employ experimentally beneficial optimization procedures from OU models to the Ehrenfest models.
Secondly, we can even use generative diffusion models trained in continuous state spaces to simulate a reverse time Ehrenfest process in discrete state spaces.
The focus of this chapter lies with rigorously elucidating the theoretical connections between the Ehrenfest process and the Ornstein-Uhlenbeck process and demonstrating the practical implications of these connections for generative modeling.

\section{Time-reversed Markov jump processes}
\label{sec: time-reversed jump processes}

We consider Markov jump processes $M(t)$ that runs on the time interval $[0, T] \subset \R$ and are allowed to take values in a discrete set $\Omega \subsetcong \mathbb{Z}^d$. Usually, we consider $\Omega \cong \left\{0, \dots, S \right\}^d$ such that the cardinality of our space is $|\Omega| = (S+1)^d$. Jumps between the discrete states appear randomly, where the rate of jumping from state $y$ to $x$ at time $t$ is specified by the function $r_t(x | y)$. The jump rates determine the jump probability in a time increment $\Delta t$ via the relation
\begin{equation}
  \label{eq: transition Markov jump}
  p_{t+\Delta t| t}(x | y) = \delta_{x, y} + r_t(x| y) \Delta t + o(\Delta t),
\end{equation}
i.e. the higher the rate and the longer the time increment, the more likely is a transition between two corresponding states.
In order to simulate the process backwards in time, we are interested in the rates of the time-reversed process $\cev{M}(t)$, which determine the backward transition probability via
\begin{equation}
  p_{t-\Delta t| t}(x | y) = \delta_{x, y} + \cev{r}_t(x| y) \Delta t + o(\Delta t).
\end{equation}


While the forward transition probability $p_{t|0}$ can usually be approximated (e.g. by solving the corresponding master equation, see \cref{sec: master equation jump process}, the time-reversed transition function $p_{0|t}$ is typically not tractable, and we therefore must resort to the learning task of learning an approximate $p_{0|t}$.

\section{The Ehrenfest process}
\label{sec: Ehrenfest process}

In principle, we are free to choose any forward process $M(t)$ for which we can compute the forward transition probabilities $p_{t|0}$ and which is close to its stationary distribution after a not too long run time $T$. In the sequel, we argue that the \textit{Ehrenfest process} is particularly suitable -- both from a theoretical and practical perspective.

The Ehrenfest process was introduced by the Russian-Dutch and German physicists Tatiana and Paul Ehrenfest to explain the second law of thermodynamics, see \cite{ehrenfest1907zwei}.
It is defined as the number of particles in the first of two containers, where the particles can move between the containers. The process is a birth-death process with values in $\{0, \dots, S\}$, where $S$ is the total number of particles. The transition rates are given by
\begin{equation}
  \label{eq: definition Ehrenfest}
  E_S(t) := \sum_{i=1}^{S} Z_i(t),
\end{equation}
where each of the independent $Z_i$ is a process on the state space $\Omega = \{0, 1\}$ with transition rates $r(0 | 1) = r(1 | 0) = \frac{1}{2}$ (sometimes called \textit{telegraph} or \textit{Kac process}). We note that the \textit{Ehrenfest process} is a birth-death process with values in $\{0, \dots, S\}$ and transition rates
\begin{equation}
  \label{eq: Ehrenfest rates}
  r(x+1|x) = \frac{1}{2}(S-x), \qquad r(x-1|x) = \frac{x}{2}.
\end{equation}

One compelling property of the Ehrenfest process is that we can sample without needing to simulate trajectories.

Assuming $E_S(0) = x_0$, the Ehrenfest process can be written as
\begin{equation}
  E_S(t) = E_{0,S}(t) + E_{1,S}(t),
\end{equation}
where $E_{0, S}(t) \sim B(S - x_0, 1 - f(t))$ and $E_{1, S}(t) \sim B(x_0, f(t))$ are independent binomial random variables and $f(t) := 
  \frac{1}{2}\left(1 + e^{-t}\right)$. Consequently, the solution of the master equation for the Ehrenfest process is given by the discrete convolution
\begin{equation}
  \label{eq: convolution of binomials}
  p_{t|0}(x | x_0) 
  = \sum_{z\in \Omega}  \P\left(E_{0,S}(t) = z \right) \P\left(E_{1,S}(t) = x-z \right).
\end{equation}

\subsection{The Ehrenfest Process in the infinite state space limit}
\label{sec: forward scaled Ehrenfest process}

It is known that certain (appropriately scaled) Markov jump processes converge to state-continuous diffusion processes when the state space size $S + 1$ tends to infinity (see, e.g., \cite{kurtz1972relationship,gardiner1985handbook}).
For the Ehrenfest process, this convergence can be studied quite rigorously.
To this end, let us introduce the scaled Ehrenfest process

\begin{align}
  \label{eq: scaled Ehrenfest}
  \widetilde{E}_S(t) := \frac{2}{\sqrt{S}}\left(E_{S}(t) -  \frac{S}{2}\right)
\end{align}
with transition rates
\begin{align}
  r\left(x\pm \frac{2}{\sqrt{S}}\bigg|x\right) & = \frac{\sqrt{S}}{4}(\sqrt{S} \mp x), 
  \label{eq: scaled Ehrenfest rates}
\end{align}
now having values in $\Omega = \left\{-\sqrt{S},-\sqrt{S}+\frac{2}{\sqrt{S}}, \dots, \sqrt{S} \right\}$. We are interested in the large state space limit $S \to \infty$, noting that this implies $\frac{2}{\sqrt{S}} \to 0$ for the transition steps, thus leading to a refinement of the state space.
We can augment the rates \cref{eq: scaled Ehrenfest rates} to be time-dependent by introducing the time transformation $\lambda_t$ as noted in \cref{eq: discrete rates time transformation one} and \cref{eq: discrete rates time transformation two} such that the time-dependent rates become
\begin{align}
  \label{eq: scaled Ehrenfest rates time dependent}
  r_t\left(x\pm \frac{2}{\sqrt{S}}\bigg|x\right) & = \lambda_t \ \frac{\sqrt{S}}{4}(\sqrt{S} \mp x).
\end{align}

The convergence of Markov jump processes to SDEs in the limit of large state spaces (with appropriately scaled jump sizes) has formally been studied via the Kramers-Moyal expansion \cite{gardiner1985handbook,van1992stochastic}. For more rigorous results, we refer, e.g., to \cite{kurtz1972relationship,kurtz1981approximation}.

The following convergence result is shown in \cite[Theorem 4.1]{sumita2004numerical}.
\begin{proposition}
  \label{prop: convergence of forward Ehrenfest}
  In the limit $S \to \infty$, the scaled Ehrenfest process $\widetilde{E}_S(t)$ converges in law to the Ornstein-Uhlenbeck process $X_t$ for any $t \in [0, T]$, where $X_t$ is defined via the SDE
  \begin{equation}
    \label{eq: OU process}
    \mathrm d X_t = -X_t \, \mathrm d t + \sqrt{2} \, \mathrm d W_t,
  \end{equation}
  with $W_t$ being standard Brownian motion.
\end{proposition}
\begin{proof}
  The proof is given in \cite[Theorem 4.1]{sumita2004numerical} and a simplified derivation can be found in \cref{app:cha4proofreverseehrenfest}.
\end{proof}

\subsection{The reverse Ehrenfest Process in the infinite state space limit}
\label{sec: scaled Ehrenfest process}

We first recall the scaled Ehrenfest process from \eqref{eq: scaled Ehrenfest},
\begin{equation}
  \widetilde{E}_S(t) := \frac{2}{\sqrt{S}}\left(E_{S}(t) -  \frac{S}{2}\right),
\end{equation}
and note that $\widetilde{E}_S \in \left\{-\sqrt{S},-\sqrt{S}+\frac{2}{\sqrt{S}},\dots,\sqrt{S} \right\}$, where the birth-death transitions transform from $\pm 1$ in $E_S$ to $\pm \frac{2}{\sqrt{S}}$ in its scaled version $\widetilde{E}_S$. Accordingly, the reverse rates from \cref{sec: reverse time jump processes} translate to
\begin{equation}
  \label{eq: scaled Ehrenfest backward rates}
  \cev{r}_t\left(x \pm \nicefrac{2}{\sqrt{S}} \middle| x\right)
  = \E_{x_0 \sim p_{0|t}(x_0 | x)}\left[ \frac{p_{t|0}\left(x \pm \nicefrac{2}{\sqrt{S}} \middle| x_0\right)}{p_{t|0}(x | x_0)} \right] r\left(x \middle| x \pm \nicefrac{2}{\sqrt{S}}\right).
\end{equation}

Inspecting \Cref{eq: scaled Ehrenfest backward rates}, which specifies the rate function of a backward Markov jump process, we realize that the time-reversal essentially depends on two things, namely the forward rate function with switched arguments as well as the conditional expectation of the ratio between two forward transition probabilities.
To gain some intuition, let us first assume that the state space size $S + 1$ is large enough and that the transition density $p_{t | 0}$ can be extended to $\R$ (which we call $\overline{p}_{t | 0}$) such that it can be approximated via a Taylor expansion.

For very large state spaces, the differences in the forward rates between two neighboring states can be assumed to vanish such that we have
\begin{equation}
  r\left(x \pm \nicefrac{2}{\sqrt{S}}\middle| x\right) \approx  r\left(x \middle| x \mp \nicefrac{2}{\sqrt{S}} \right)
\end{equation}

Furthermore, we can expand the ratio with the Taylor expansion of the first order to obtain
\begin{subequations}
  \begin{align}
    \label{eq: Taylor fraction continous extension}
    \frac{p_{t|0}\left(x \pm \nicefrac{2}{\sqrt{S}}\middle| x_0\right)}{p_{t|0}(x|x_0)}
     & \approx \frac{\overline{p}_{t|0}(x|x_0) \pm \nicefrac{2}{\sqrt{S}} \nabla \overline{p}_{t|0}(x|x_0)}{\overline{p}_{t|0}(x|x_0)} \\
     & = 1 \pm \nicefrac{2}{\sqrt{S}} \nabla \log \overline{p}_{t| 0}(x | x_0),
  \end{align}
\end{subequations}
where the log-derivative trick was employed.
The conditional expectation of $\nabla \log \overline{p}_{t| 0}(x | x_0)$ is reminiscent of the score function in SDE-based diffusion models, albeit be it the score function of a conditional distribution.
This already hints at a close connection between the time-reversal of Markov jump processes and score-based generative modeling.

In fact, by taking  the appropriate conditional expectation of the conditional score function, we can recover the unconditional score function, as we will see in the following.
Noting the identity $p^\mathrm{SDE}_t(x) = \int_{\R^d} p^\mathrm{SDE}_{t | 0}(x|x_0) p_\mathrm{data}(x_0)\mathrm d x_0$, we can compute
\begin{subequations}
  \begin{align}
    \nabla \log p^\mathrm{SDE}_t(x) & = \frac{\nabla_x p^\mathrm{SDE}_t(x)}{p^\mathrm{SDE}_t(x)}                                                                                                                                                    \\
                                    & = \frac{\int_{\R^d} \nabla_x \log p^\mathrm{SDE}_{t | 0}(x|x_0) p^\mathrm{SDE}_{t | 0}(x|x_0) p_\mathrm{data}(x_0)\mathrm d x_0}{\int_{\R^d} p^\mathrm{SDE}_{t | 0}(x|x_0) p_\mathrm{data}(x_0)\mathrm d x_0} \\
                                    & = \E_{x_0 \sim p^\mathrm{SDE}_{0|t}(x_0 | x)}\left[\nabla_x \log p^\mathrm{SDE}_{t | 0}(x |x_0) \right],
  \end{align}
\end{subequations}
where it holds $p^\mathrm{SDE}_{0|t}(x_0 | x) = \frac{p^\mathrm{SDE}_{t | 0}(x|x_0) p_\mathrm{data}(x_0)}{\int_{\R^d} p^\mathrm{SDE}_{t | 0}(x|x_0) p_\mathrm{data}(x_0)\mathrm d x_0}$.
Thus we can see that the ratio at the heart of the reverse rates in fact corresponds to the score function for Ehrenfest processes in the limit of infinite states,
\begin{align}
  \label{eq: reverse Ehrenfest with score function}
  \cev{r}_t\left(x \pm \nicefrac{2}{\sqrt{S}} | x\right)
   & = \E_{x_0 \sim p_{0|t}(x_0 | x)}\left[ \frac{p_{t|0}\left(x \pm \nicefrac{2}{\sqrt{S}} | x_0\right)}{p_{t|0}(x | x_0)} \right] r\left(x | x \pm \nicefrac{2}{\sqrt{S}}\right) \\
   & = \E_{x_0 \sim p_{0|t}(x_0 | x)}\left[1 \pm \nicefrac{2}{\sqrt{S}} \nabla \log p_{t|0}(x | x_0) \right] r\left(x \middle| x \pm \nicefrac{2}{\sqrt{S}}\right).
\end{align}

Earlier, it was shown that the scaled Ehrenfest process converges to the OU process for sufficiently large states spaces.
Given the rates of the reverse Ehrenfest process, we can recover the reverse OU process as well, as we will see in the following proposition

\begin{proposition}
  \label{prop: jump moment reversed Ehrenfest}
  Let $b$ and $D$ be the first and second jump moments of the scaled Ehrenfest process $\widetilde{E}_S$. The first and second jump moments of the time-reversed scaled Ehrenfest $\cev{\widetilde{E}}_S$ are then given by
  \begin{align}
    \label{eq: reversed first jump moment}
    \cev{b}(x, t)                                     & = -b(x) + D(x) \, \E_{x_0 \sim p_{0|t}(x_0 | x)}\left[ \frac{\Delta_S p_{t | 0}(x|x_0)}{p_{t | 0}(x|x_0)} \right] + o(S^{-1/2}), \\
    \label{eq: reversed second jump moment}\cev{D}(x) & = D(x) + o(S^{-1/2}),
  \end{align}
  where we can express the finite difference in terms of the actual number of states $S$ instead of the increment $\delta$, such that
  \begin{equation}
    \Delta_S p_{t | 0}(x|x_0) := \frac{p_{t | 0}(x + \frac{2}{\sqrt{{S}}}|x_0) - p_{t | 0}(x|x_0)}{\frac{2}{\sqrt{S}}}
  \end{equation}
  is a one step difference and $p_{t|0}$ and $p_{0|t}$ are the forward and reverse transition probabilities of the scaled Ehrenfest process.
\end{proposition}
\begin{proof}
  The proof is outlined in the appendix \ref{app:cha4proofreverseehrenfest}.
\end{proof}

The reverse jump moments imply that the time-reversed Ehrenfest process is expected to converge in law to the time-reversed Ornstein-Uhlenbeck process.
For $S \to \infty$, we know that the forward Ehrenfest process converges to the Ornstein-Uhlenbeck process, i.e. $p_{t|0}$ converges to $p_{t|0}^\mathrm{OU}$, where $p_{t|0}^\mathrm{OU}(x | x_0)$ is the transition density of the Ornstein-Uhlenbeck process starting at $X_0 = x_0$.

Together with the fact that the finite difference approximation operator $\Delta_S$ converges to the first derivative, this implies that $\E_{x_0 \sim p_{0|t}(x_0 | x)}\left[ \frac{\Delta_S p_{t | 0}(x|x_0)}{p_{t | 0}(x|x_0)} \right]$ is expected to converge to $\E_{x_0 \sim p_{0|t}^\mathrm{OU}(x_0 | x)}\left[ \nabla \log p_{t|0}^\mathrm{OU}(x | x_0)  \right]$.
An intuition can be build by applying the log-derivative trick of a conditional distribution $p(x|y)$ by observing that $\nabla \log p(x|y) = \frac{\nabla p(x|y)}{p(x|y)}$.

This shows that this conditional expectation is the score function of the Ornstein-Uhlenbeck process, i.e. $\nabla \log p_t^\mathrm{OU}(x) = E_{x_0 \sim p_{0|t}^\mathrm{OU}(x_0 | x)}\left[ \nabla \log p_{t|0}^\mathrm{OU}(x | x_0)  \right]$. Finally, we note that the first and second jump moments converge to the drift and the square of the diffusion coefficient of the limiting SDE, respectively \cite{gardiner1985handbook}.

Therefore, the scaled time-reversed Ehrenfest process $\cev{\widetilde{E}}_S(t)$ is expected to converge in law to the process $Y_t$ given by
\begin{align}
  \label{eq: time-reversed OU}
  \mathrm d Y_t = \left(Y_t + 2 \nabla \log p_{T-t}^\mathrm{OU}(Y_t) \right) \mathrm d t + \sqrt{2} \, \mathrm d W_t,
\end{align}
which is the time-reversal of the Ornstein-Uhlenbeck process stated in \eqref{eq: OU process}. Note that we write \eqref{eq: time-reversed OU} as a forward process from $t = 0$ to $t=T$, where $W_t$ is a forward Brownian motion, which induces the time-transformation $t \mapsto T-t$ in the score function.

\subsection{Reverse Rate Estimators}

It should be noted that the convergence of the scaled Ehrenfest process to the Ornstein-Uhlenbeck process implies
\begin{equation}
  \label{eq: approx of forward transition by Gaussian}
  p_{t|0}(x|x_0) \approx p_{t|0}^{\mathrm{OU}}(x|x_0) := \mathcal{N}(x; \mu_t(x_0), \sigma_t^2)
\end{equation}
with $\mu_t(x_0) = x_0 e^{-t}$ and $\sigma_t^2 = (1 - e^{-2t})$.

Substituting $p_{t|0}$ in \eqref{eq: reverse Ehrenfest with score function} with the approximation $p_{t|0}^{\mathrm{OU}}$ allows us to define the ratio in terms of two Gaussian distributions which simplifies the ratio.
With this approximation \eqref{eq: approx of forward transition by Gaussian}, we can define the loss
\begin{align}
  \label{eq: Gauss loss}
  \begin{split}
    \mathcal{L}_\mathrm{Gauss}(\varphi):=
    \quad\E\left[\left(\varphi(x,t) - \exp\left(\frac{\mp 2(x - \mu_t(x_0))\delta - \delta^2}{2 \sigma^2_t}\right)\right)^2\right].
  \end{split}
\end{align}

Unfortunately, \cref{eq: Gauss loss} includes the non-linear exponential function in the target function, which makes it suboptimal to train against.
The approximation \eqref{eq: approx of forward transition by Gaussian} therefore allows us to approximately model the marginal distributions $p_{t|0}(x|x_0)$ as Gaussian distributions.
For the quantity in the conditional expectation \eqref{eq: scaled Ehrenfest backward rates} and the shorthand notation $\delta := \frac{2}{\sqrt{S}}$ for the state space differences we can thus compute
\begin{subequations}
  \begin{align}
    \label{eq: Gaussian approximation of ratio}
    \frac{p_{t|0}\left(x \pm \delta \middle| x_0\right)}{p_{t|0}(x|x_0)}
    \approx & \exp\left(\frac{\mp 2(x - \mu_t(x_0))\delta - \delta^2}{2 \sigma^2_t}\right)                                                                                                                                                     \\
    \approx & \exp\hspace{-0.1cm}\left( \hspace{-0.1cm}-\frac{\delta^2}{2 \sigma_t^2}\right)\hspace{-0.15cm}\left(\hspace{-0.05cm}1 \mp \frac{(x-\mu_t(x_0))\delta}{\sigma^2} + \frac{\left((x-\mu_t(x_0))\delta\right)^2}{2\sigma^4} \right).
    \label{eq: Taylor of fraction}
  \end{align}
\end{subequations}
\Cref{eq: Taylor of fraction} is a Taylor expansion of the exponential function in \eqref{eq: Gaussian approximation of ratio} and is valid for small $\delta$.

The full reverse jump rates are computed by taking the expectation of the fraction, respectively its approximations, over the distribution $p_{0|t}(x_0 | x)$.
The approximation \eqref{eq: Taylor of fraction} has allowed us to linearize the previously non-linear fraction such that the application of the conditional expectation is now feasible.
Furthermore the first moment of an OU process $\mu_t(x_0) = w(t) x_0$ is a linear function in $x_0$ which further simplifies the expectation of \cref{eq: scaled Ehrenfest backward rates} to
\begin{subequations}
  \begin{align}
            & \E_{x_0 \sim p_{0|t}(x_0 | x)}\left[ \frac{p_{t|0}\left(x \pm \delta | x_0\right)}{p_{t|0}(x | x_0)} \right]                                                                                     \\
    \approx & \exp\hspace{-0.1cm}\left( \hspace{-0.1cm}-\frac{\delta^2}{2 \sigma_t^2}\right)\hspace{-0.15cm}\Bigg( \hspace{-0.05cm}1 \mp \frac{(x-\E_{x_0 \sim p_{0|t}(x_0 | x)}[\mu_t(x_0)])\delta}{\sigma^2} \\
            & \hspace{2.6cm} + \frac{\E_{x_0 \sim p_{0|t}(x_0 | x)}\left[\left((x-\mu_t(x_0))\delta\right)^2 \right]}{2\sigma^4} \Bigg)
  \end{align}
\end{subequations}

Under the assumption that the forward transition probabilities $p_{t|0}$ can be approximated by Gaussian distributions and its corresponding Taylor expansion, we can now approximate the reverse jump rates of the scaled Ehrenfest process by the conditional expectation of $\E_{x_0 \sim p_{0|t}(x_0 | x)}[x_0]$ for the first order approximation.
The expansion of the fraction in \eqref{eq: Taylor of fraction} to the second order adds an additional term to both the birth as well as the death rate equally.

We distinguish our proposed reverse rate estimators into two broad classes: conditional expectations and continuous processes.
The conditional expectation based losses are derived with the Taylor expansion of the approximation of the ratio of transition probabilities and predict the conditional expectations of the first and second moments of the Taylor expansion.
In consequence this allows us to consider the following loss functions,
\begin{align}
  \label{eq: first Taylor ratio loss}
  \mathcal{L}_\mathrm{Taylor}(\varphi_1) :=   & \E\left[\left(\varphi(x, t) - \mu_t(x_0) \right)^2 \right]                          \\
  \mathcal{L}_\mathrm{Taylor,2}(\varphi_2) := & \E\left[\left(\varphi(x, t) - \left((x-\mu_t(x_0))\delta\right)^2 \right)^2 \right]
\end{align}
where the expectations are over $x_0 \sim p_\mathrm{data}, t\sim\mathcal{U}(0,T),x\sim p_{t|0}(x|x_0)$ and $\varphi$ are function approximators such as specific deep neural network architectures.

We saw earlier in \cref{eq: reverse Ehrenfest with score function} that we can leverage the continuous state limit to approximate the reverse Ehrenfest process by the Stein score of the Ornstein-Uhlenbeck process.
The first order Taylor approximation fortuitously connects to score based generative modeling, where the score function is the gradient of the log-likelihood of the data distribution.
In the case of the Gaussian marginal distribution of the OU process, the score equates to
\begin{align}
  \label{eq: OU score}
  \nabla \log p^\mathrm{OU}_{t|0}(x|x_0) = -\frac{x - \mu_t(x_0)}{\sigma^2_t}.
\end{align}
and we can see that the scaled Ehrenfest process employs a scaled score function in the reverse jump rates.
This motivates us to train a state-discrete scaled Ehrenfest model with the loss defined by
\begin{subequations}
  \label{eq: forward OU loss}
  \begin{align}
    \mathcal{L}_\mathrm{OU}(\varphi)
     & := \E\left[\left(\varphi(x, t) - \nabla \log p_{t|0}^\mathrm{OU}(x|x_0)\right)^2\right]         \\
     & =  \E\left[\left(\varphi(x, t) +\frac{\left(x-\mu_t(x_0) \right)}{\sigma_t^2} \right)^2\right],
  \end{align}
\end{subequations}
where the expectation is over $x_0 \sim p_\mathrm{data}, t\sim\mathcal{U}(0,T),x\sim p_{t|0}(x|x_0)$ and where $\mu_t(x_0) = x_0 e^{-t}$ and $\sigma_t^2 = (1 - e^{-2t})$, as before.
In fact, this loss is completely analog to the denoising score matching loss in the state-continuous setting.
We later set $\varphi = 1 \pm \frac{2}{\sqrt{S}}\varphi^*$, where $\varphi^*$ is the minimizer of \eqref{eq: forward OU loss}, to get the approximated conditional expectation.

\subsection{Connection to DDPM}

As we have seen in \cref{eq: Taylor fraction continous extension}, we can directly link the scaled Ehrenfest process to score-based generative modeling in continuous time and space.
In particular, we can employ any model which reverses a diffusive forward Ornstein-Uhlenbeck process by learning the corresponding score.
For instance, we can rely on DDPM models \cite{ho2020denoising}, which typically consider the forward SDE
\begin{align}
  \label{eq: DDPM forward process}
  \mathrm d X_t = -\frac{1}{2} \beta(t) X_t + \sqrt{\beta(t)} \mathrm dW_t
\end{align}
on the time interval $[0, 1]$, where $\beta:[0, 1] \to \R$ is a function that scales time.
For the process \cref{eq: DDPM forward process} conditioned at the initial value $X_0 = x_0$ it holds
\begin{equation}
  X_t | x_0 \sim \mathcal{N}\left(\exp\left( -\frac{1}{2} \int_0^t \beta(s) \mathrm ds \right) x_0, \ 1 - \exp\left(-\int_0^t \beta(s) \mathrm ds \right) \right).
  \label{eq: ddpm forward}
\end{equation}
Whereas DDPM was initially proposed with discretized time steps, we consider the continuous time limit here and choose $\beta(t) := \beta_\mathrm{min} + t (\beta_\mathrm{max} - \beta_\mathrm{min})$ with $\beta_\mathrm{min} = 0.1, \beta_\mathrm{max} = 20$, as suggested in \cite{song2020score}. Note that this typically guarantees that $X_1$ is approximately distributed according to $\mathcal{N}(0, 1)$, independent of $x_0$.

These models usually transport a standard Gaussian to the target density that is supported on $[-1, 1]^d$.
In order to cope with the fact that the scaled Ehrenfest process terminates (approximately) at a standard Gaussian irrespective of the size $S + 1$, we typically choose $S = 255^2$ such that the interval $[-1, 1]$ contains $256$ states that correspond to the RGB color values of images, recalling that the increments between the states are $\frac{2}{\sqrt{S}}$.

We can draw samples from the solution of the forward process by sampling from the corresponding Gaussian distribution conditioned on the initial condition $x_0$,
\begin{align}
  X_t & = \underbrace{\exp\left( -\frac{1}{2} \int_0^t \beta(s) \mathrm ds \right) x_0}_{\mu_t(x_0)} + \epsilon \underbrace{\sqrt{1 - \exp\left(-\int_0^t \beta(s) \mathrm ds \right)}}_{\sigma_t}, \quad \epsilon \sim \mathcal{N}(0, 1).
\end{align}

The DDPM framework implicitly trains a model $varphi(x, t)$ on the score of the forward Ornstein-Uhlenbeck process \cref{eq: OU score}.
As outlined by the authors of DDPM, the loss function can be simplified by noting that $\mu_t(x_0) = X_t - \epsilon_t \sigma_t$ which simplifies the score, and therefore the loss, to $\nabla \log p_{t|0}^\mathrm{OU}(x|x_0) = \nicefrac{\varepsilon_t}{\sigma_t}$,
\begin{align}
  \mathcal{L}_\mathrm{OU}(\varphi)
   & = \E\left[\left(\varphi(x, t) - \nabla \log p_{t|0}^\mathrm{OU}(x|x_0)\right)^2\right]          \\
   & = \E\left[\left(\varphi(x, t) +\frac{\left(X_t-\mu_t(x_0) \right)}{\sigma_t^2} \right)^2\right] \\
   & = \E\left[\left(\varphi(x, t) +\frac{\varepsilon_t}{\sigma_t} \right)^2\right]
\end{align}
With appropriate scaling we can train the model against the unscaled noise $\varepsilon_t$,
\begin{align}
  \mathcal{L}_\mathrm{DDPM}(\varphi)
   & = \E\left[\left(\varphi(x, t) - \varepsilon_t \right)^2\right].
\end{align}

Subsequently, we can recover the reverse-time jump rates with
\begin{align}
  \label{eq: DDPM to reverse Ehrenfest}
  \cev{r}_t\left(x \pm \nicefrac{2}{\sqrt{S}} | x\right) = \left( 1 \pm \frac{2}{\sqrt{S}} \frac{\varphi(x, t)}{\sigma_t} \right) r_t(x | x \pm \nicefrac{2}{\sqrt{S}}),
\end{align}

Further, noting the actual Ornstein-Uhlenbeck process that DDPM is trained on has a time-dependent behavior due to $\beta(t)$, we employ the time scaling $\lambda_t = \frac{1}{2}\beta(t)$, resulting in the (time-dependent) rates following \cref{eq: scaled Ehrenfest rates time dependent},
\begin{equation}
  r_t\left(x\pm \frac{2}{\sqrt{S}}\bigg|x\right) = \frac{\beta(t)}{2} \frac{\sqrt{S}}{4}(\sqrt{S} \mp x).
\end{equation}


\section{Computational aspects}
\label{sec: computational aspects}

In this section, we comment on computational aspects that are necessary for the training and simulation of the time-reversal of our (scaled) Ehrenfest process.

\subsection{Modeling of dimensions}
\label{sec: modeling of dimensions}

In order to make computations feasible in high-dimensional spaces $\Omega^d$, we typically factorize the forward process, such that each dimension propagates independently, cf. \cite{campbell2022continuous}. Note that this is analog to the Ornstein-Uhlenbeck process in score-based generative modeling, in which the dimensions also do not interact, see, e.g., \eqref{eq: OU process}.

We thus consider
\begin{equation}
  p_{t|0}(x|y) = \prod_{i=1}^d p_{t|0}^{(i)}(x^{(i)}|y^{(i)}),
\end{equation}
where $p^{(i)}_{t|0}$ is the transition probability for dimension $i \in \{1, \dots, d \}$ and $x^{(i)}$ is the $i$-th component of $x \in \Omega^d$.

In \cite{campbell2022continuous} it is shown that the forward and backward rates then translate to
\begin{equation}
  \label{eq: high dim forward rate}
  r_t(x | y) = \sum_{i=1}^d r_t^{(i)}(x^{(i)}|y^{(i)}) \Gamma_{x^{\neg i},y^{\neg i}},
\end{equation}
where $\Gamma_{x^{\neg i},y^{\neg i}}$ is one if all dimensions except the $i$-th dimension agree, and
\begin{equation}
  \label{eq: high dim backward rate}
  \cev{r}_t(x | y) = \sum_{i=1}^d \E\left[ \frac{p_{t|0}(y^{(i)}|x_0^{(i)})}{p_{t|0}(x^{(i)}|x_0^{(i)})}\right] r_t^{(i)}(x^{(i)}|y^{(i)}) \Gamma_{x^{\neg i},y^{\neg i}},
\end{equation}
where the expectation is over $x_0^{(i)} \sim p_{0|t}(x^{(i)}_0 | x)$. Equation \eqref{eq: high dim backward rate} illustrates that the time-reversed process does not factorize in the dimensions even though the forward process does.

Note with \eqref{eq: high dim forward rate} that for a birth-death process a jump appears only in one dimension at a time, which implies that
\begin{equation}
  r_t(x \pm \delta_i | x) = r_t^{(i)}(x^{(i)} \pm \delta^{(i)}_i|x^{(i)}),
\end{equation}
where now $\delta_i = (0, \dots, 0, \delta^{(i)}_i, 0, \dots, 0)^\top$ with $\delta^{(i)}_i$ being the jump step size in the $i$-th dimension. Likewise, \eqref{eq: high dim backward rate} becomes
\begin{equation}
  \label{eq: high dim backward rate birth-death}
  \cev{r}_t(x\pm \delta_i | x) =  \E\left[ \frac{p_{t|0}(y^{(i)}|x_0^{(i)})}{p_{t|0}(x^{(i)}|x_0^{(i)})}\right] r_t^{(i)}(x^{(i)}|x^{(i)} + \delta_i^{(i)}),
\end{equation}
where the expectation is over $x_0^{(i)} \sim p_{0|t}(x^{(i)}_0 | x)$, which still depends on all dimensions.

\subsection{Accelerated Sampling with Tau-Leaping}
\label{sec: tau leaping}

The fact that jumps only happen in one dimension at a time implies that the naive implementation of changing component by component (e.g. by using the Gillespie’s algorithm, see \cite{gillespie1976general}) would require a very long sampling time. As suggested in \cite{campbell2022continuous}, we can therefore rely on $\tau$-leaping for an approximate simulation methods \cite{gillespie2001approximate}. The general idea is to not simulate jump by jump, but wait for a time interval of length $\tau$ and apply all jumps at once. One can show that the number of jumps is Poisson distributed with a mean of $\tau\, \cev{r}_t \,(x | y)$.
As a side benefit, the highlighted factorization above is implicitly applied in the $\tau$-leaping algorithm, as the jumps are computed dimension by dimension.

\section{Numerical experiments}
\label{sec: experiments}

In this section, we demonstrate our theoretical insights in numerical experiments. If not stated otherwise, we always consider the scaled Ehrenfest process defined in \eqref{eq: scaled Ehrenfest}. We will compare the different variants of the loss, namely $\mathcal{L}_\mathrm{Taylor}$ defined in \eqref{eq: first Taylor ratio loss} and $\mathcal{L}_\mathrm{OU}$ defined in \eqref{eq: forward OU loss}.

\subsection{Illustrative example}

Let us first consider an illustrative example, for which the data distribution is tractable. We consider a process in $d=2$ with $S=32$, where the $(S+1)^d = 33^2$ different state combinations in $p_\mathrm{data}$ are defined to be proportional to the pixels of an image of the letter ``E''. Since the dimensionality is $d=2$, we can visually inspect the entire distribution at any time $t\in [0, T]$ by plotting 2-dimensional histograms of the simulated processes. With this experiment we can in particular check that modeling the dimensions of the forward process independently from one another (as explained in \Cref{sec: modeling of dimensions}) is no restriction for the backward process.
Indeed \Cref{fig: toy example E} shows that the time-reversed process can transport the prior distribution (which is approximately binomial, or, loosely speaking, a binned Gaussian) to the specified target. Again, note that this plot does not display single realizations, but entire distributions, which, in this case, are approximated with $500.000$ samples. We realize that in this simple problem $\mathcal{L}_\mathrm{Gauss}$ performs slightly better than $\mathcal{L}_\mathrm{OU}$ and $\mathcal{L}_\mathrm{Taylor}$. As expected, the approximations work sufficiently well even for a moderate state space size $S + 1$. As argued in \Cref{sec: scaled Ehrenfest process}, this should get even better with growing $S$.
For further details, we refer to \Cref{app: illustrative example}.


\subsection{MNIST}

For a basic image modeling task, we consider the MNIST dataset, which consists of gray scale pixels and was resized to $32 \times 32$ to match the required input size of a U-Net neural network architecture\footnote{Taken from the repository \url{https://github.com/w86763777/pytorch-ddpm}.}, such that $d=32\times32=1024$ and $S=255$.
In \Cref{fig: mnist} we display generated samples from a model trained with $\mathcal{L}_\mathrm{OU}$. The models with the other losses look equally good, so we omit them.
For further details, we refer to \Cref{app: MNIST}.

\subsection{Image modeling with CIFAR-10}

As a more challenging task, we consider the CIFAR-10 data set, with dimension $d = 3 \times 32 \times 32 = 3072$, each taking $256$ different values \cite{krizhevsky2009learning}.
In the experiments we again compare our three different losses, however, realize that $\mathcal{L}_\mathrm{Gauss}$ did not produce satisfying results and had convergence issues, which might follow from numerical issues due to the exponential term appearing in \eqref{eq: Gauss loss}.
Further, we consider three different scenarios: we train a model from scratch, we take the U-Net model that was pretrained in the state-continuous setting, and we take the same model and further train it with our state-discrete training algorithm.

We display the Frechet-Inception-Distance (FID) and the Inception-Score (IS) metrics \cite{salimans2016improved, heusel2017gans} in \Cref{tab: CIFAR metrics}.
The IS score measures how well the generated images can be classified by a pretrained Inception network.
Since this is based on a discriminative evaluation approach, the FID was proposed as an alternative.
It measures the distance between the generated images and the real data distribution in the feature space of the penultimate layer of the Inception network.
Therefore, more attention should be given to the FID score, as it is more informative about the diversity and quality of the generated images.

When using only transfer learning, the different losses indicate different ways of incorporating the pretrained model.
We realize that both losses produce comparable results, with small advantages for $\mathcal{L}_\mathrm{OU}$.
Even without having invested much time in finetuning hyperparameters and sampling strategies, we reach competitive performance with respect to the alternative methods LDR \cite{campbell2022continuous} and D3PM \cite{austin2021structured}.
Remarkably, even the attempt with transfer learning returns good results, without having applied any further training.
For further details, we refer to \Cref{app: CIFAR}, where we also display more samples in Figures \ref{fig:cifar 10 taylor big 1}-\ref{fig:cifar 10 score big 2}.

\begin{figure}[ht]
  \centering
  \begin{minipage}{0.49\textwidth}
    \centering
    \includegraphics[width=\linewidth]{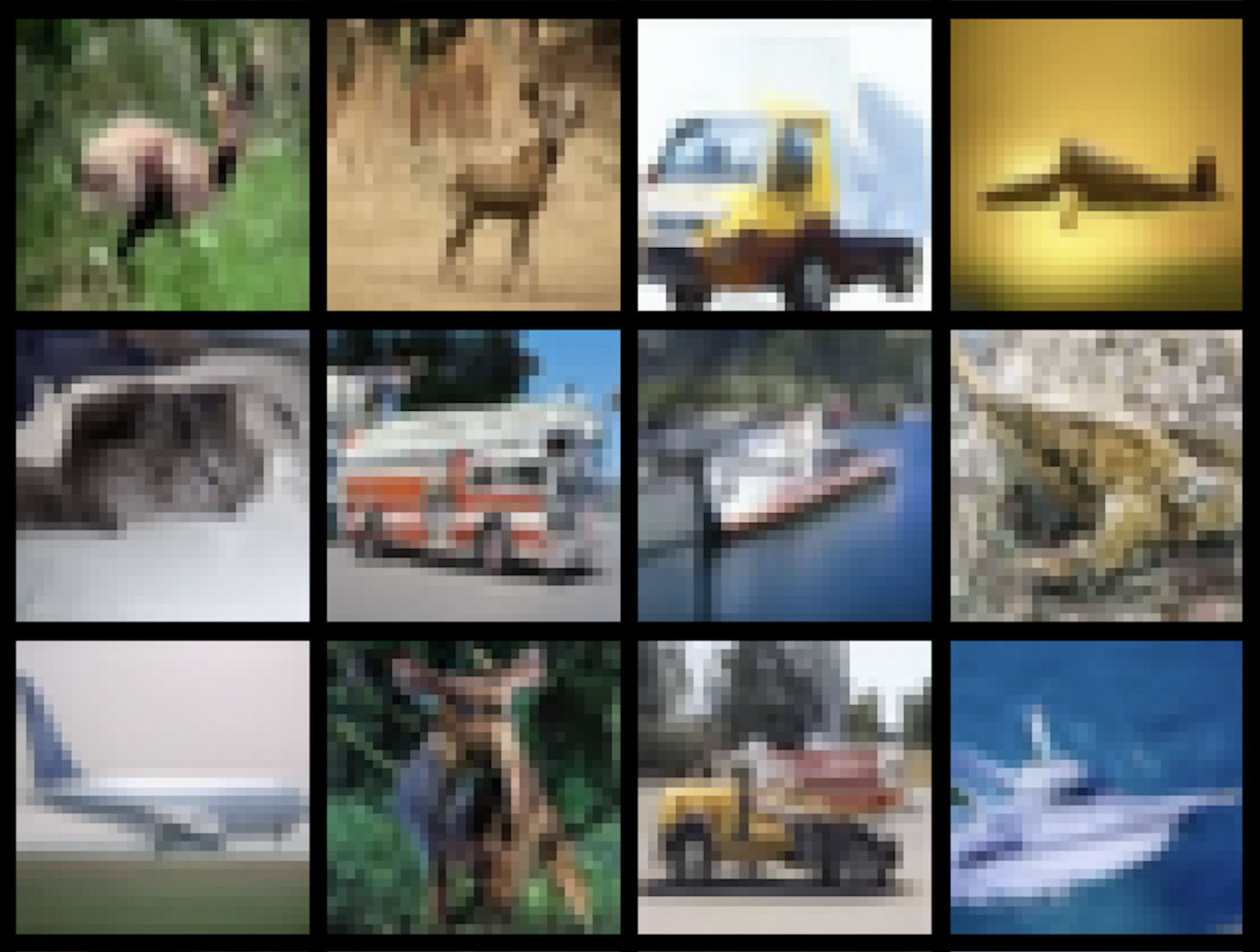} 
    \caption{CIFAR-10 samples from the Ehrenfest process with a pretrained model, further finetuned with $\mathcal{L}_\mathrm{OU}$.}
    \label{fig:figure1}
  \end{minipage}\hfill
  \begin{minipage}{0.49\textwidth}
    \centering
    \includegraphics[width=\linewidth]{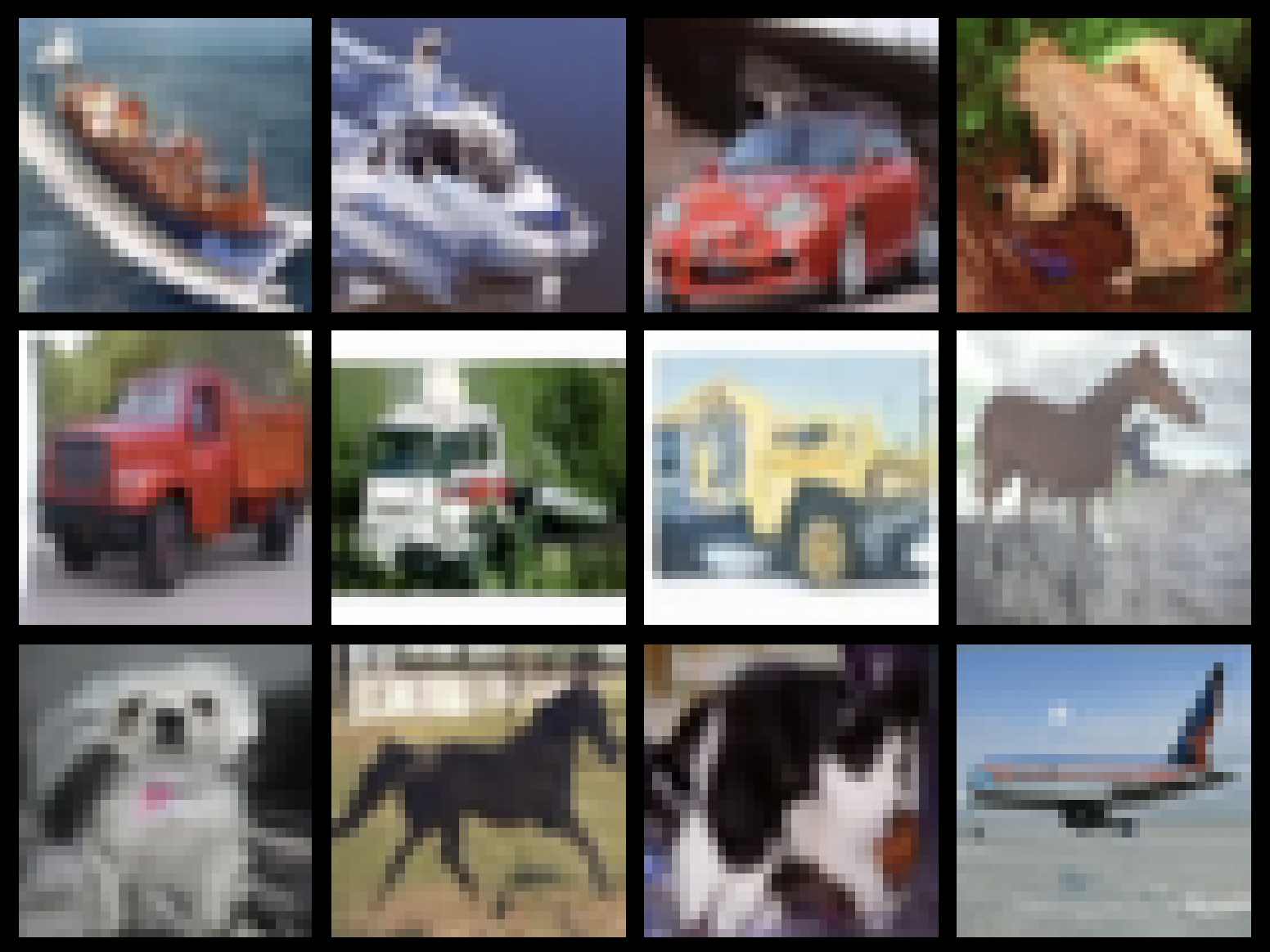} 
    \caption{CIFAR-10 samples from the Ehrenfest process with a pretrained model, further finetuned with $\mathcal{L}_\mathrm{Taylor}$.}
    \label{fig:figure2}
  \end{minipage}
\end{figure}

\begin{table}[H]
  \centering
  \begin{tabular}{clcc}
                        &                                & IS ($\uparrow$) & FID ($\downarrow$) \\
    \hline
    Ehrenfest           & $\mathcal{L}_\mathrm{OU}$      & $8.75$          & $11.57$            \\
    (transfer learning) & $\mathcal{L}_\mathrm{Taylor}$  & $8.68$          & $11.72$            \\
    \hline
    Ehrenfest           & $\mathcal{L}_\mathrm{OU}$      & $9.35$          & $5.59$             \\
    (from scratch)      & $\mathcal{L}_\mathrm{Taylor}$  & $9.33$          & $5.53$             \\
                        & $\mathcal{L}_\mathrm{Taylor2}$ & $9.40$          & $5.44$             \\
    \hline
                        & $\tau$-LDR (0)                 & $8.74$          & $8.10$             \\
    Alternative         & $\tau$-LDR (10)                & $9.49$          & $3.74$             \\
    methods             & D3PM Gauss                     & $8.56$          & $7.34$             \\
                        & D3PM Absorbing                 & $6.78$          & $30.97$            \\
    \hline
  \end{tabular}
  \caption{Performance in terms of Inception Score (IS) \cite{salimans2016improved} and Frechet Inception Distance (FID) \cite{heusel2017gans} on CIFAR-10 over $50.000$ samples. We compare two losses and consider three different scenarios: we train a model from scratch and we take the U-Net model that was pretrained in the state-continuous setting (called ``transfer learning'').}
  \label{tab: CIFAR metrics}
\end{table}

\nopagebreak
\section{Discussion and Limitations}

In this work, we have shown that the time-reversal of the scaled Ehrenfest process can be approximated by a score-based generative model.
This is a remarkable result, as it allows us to directly link the state-discrete Ehrenfest process to state-continuous score-based generative modeling.
We have shown that the reverse jump rates of the scaled Ehrenfest process can be approximated by the conditional expectation of the ratio of two forward transition probabilities.
This conditional expectation can be directly linked to the score function in the SDE setting, thus for the first time connecting a reverse-time discrete state-space jump process to a reverse-time continuous state-space diffusion process.

We have also shown that the time-reversed Ehrenfest process is expected to converge in law to the time-reversed Ornstein-Uhlenbeck process, which is the process that is typically used in score-based generative modeling.
This allows for transfer learning between the two cases.
Prominently, we can employ the approximated conditional expectation as the score function in a continuous model, or vice versa, we can train a continuous model and approximate the conditional expectation by the score.
In particular we can both use the score-based diffusion models which directly approximate the scaled score of the forward stochastic process or use the DDPM framework to translate the noise based formulation into the respective score of the forward process.

Naturally, the choice of a birth-death process is suboptimal compared to categorical diffusion models the movement of which is not restricted to immediate neighboring states.
The exact simulation of both the forward and backward processes changes a single dimension at a time by a single increment or decrement, which makes exact simulation of high-dimensional processes computationally infeasible.
For that reason, the methodology of $\tau$-leaping is employed to approximate the simulation of the time-reversed process which is a limiting factor in the precision of the simulation.
Secondly, the approximation of the conditional expectation by a function $\varphi_y$ is a challenging task, as we may need to approximate different functions $\varphi_y$ for different $y \in \Omega$ which adds a second approximating factor to the mix.
While being the most challenging task as it is a high dimensional optimization problem it is the single-most important factor in the precision of the simulation.
Thirdly, the choice of a discrete modelling of data might be the correct approach in terms of the definition of the data space, yet empirical results of continuous models for ordinal data with subsequent discretization have yielded consistently better results.
The ability of the approach to model data even between the discrete states seems to be a crucial factor in the success of continuous state diffusion models.

%% file: chapters/chapterSB.tex
\chapter{Stochastic Bridges with Score Estimation}
\label{cha:stochasticbridge}

\begin{tcolorbox}[colback=gray!10!white, colframe=black]
      Parts of this chapter are mainly based on:
      \begin{itemize}
            \item \fullcite{winkler_2023}
      \end{itemize}
\end{tcolorbox}

Initially, we studied deterministic boundary values, in which the solution of the underlying dynamics was known and the boundary conditions were deterministic.
Subsequently, we extended our study to the stochastic setting, in which the boundary conditions were stochastic and the underlying dynamics were known only in the forward direction.
In this concluding chapter, we will study the most general case, in which the boundary conditions are stochastic and the underlying dynamics are unknown in both the forward and reverse directions.

In the stochastic process community, this problem setting is known as the Schr\"odinger bridge problem, which is a special case of a stochastic control problem that seeks to find the most likely evolution of a stochastic process that starts from a given probability distribution at an initial time and arrives at another given distribution at a final time.
This is done under the constraint that the evolution of the process follows the dynamics of a stochastic differential equation (SDE).
Essentially, it's about finding a process that connects two probability distributions over time in the most likely or "optimal" way, according to some criterion, usually involving minimizing an entropy-related functional.

In the context of boundary value problems, Schr\"odinger Bridges are boundary value problems with unknown dynamics in any direction and probabilistic boundary conditions.
The task is to learn \emph{two separate stochastic processes in both directions}, such that the probability densities of the corresponding state variables at initial and final times coincide with given densities serving as boundary conditions.
Upon solution, the most likely path or distribution evolution of the system from the initial to the final state is found, adhering to the dynamics governed by the differential equation of the problem.

\begin{figure}[bt]
      \centering
      \includegraphics[width=0.8\textwidth]{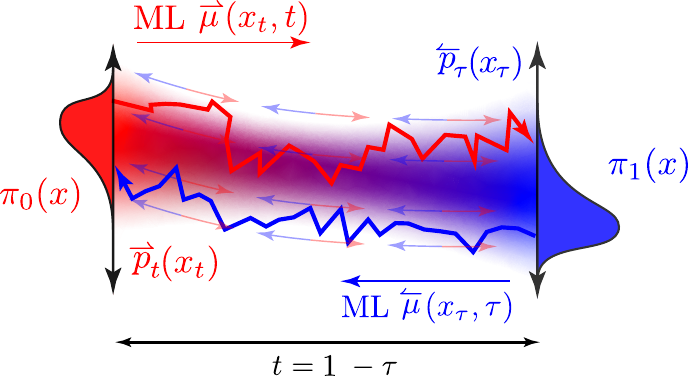}
      \caption{A figurative representation of the boundary value problem considered in this chapter.
            The task is to learn a time-reversible stochastic process in both directions, such that the marginal probability densities $\smash{\stackrel{\rightharpoonup}{p}_t(x_t)}$ and $\smash{\stackrel{\leftharpoonup}{p}_\tau(x_\tau)}$ of the corresponding state variables induced bei their respective drifts $\smash{\stackrel{\rightharpoonup}{\mu}(x_t, t)}$ and $\smash{\stackrel{\leftharpoonup}{\mu}(x_\tau, \tau)}$ of the respective SDEs at initial and final times coincide with given densities $\pi_0$ and $\pi_1$ serving as boundary conditions.}
      \label{fig:full bridge graphical abstract}
\end{figure}

Stochastic half-bridges can be seen as a special case of Schr\"odinger bridges as they provide the dynamics of the forward stochastic process while requiring the inference of the dynamics of the reverse process.
In contrast to earlier chapters, we were required to learn either time-reversible dynamics or the reverse dynamics of a given forward process.
In these cases either the solution of the dynamics or the dynamics in one direction were known.
This does not apply to Schr\"odinger bridges, where the dynamics in both directions are unknown and have to be learned from data.

Schr\"odinger bridges were introduced by E. Schr\"odinger in 1931 \cite{schrodinger1931umkehrung} \cite{schrodinger1932theorie} as the most likely temporal evolution of the probability density for diffusing particles between given initial and end distributions.
The boundary values thus become stochastic and the task is to learn a stochastic process in both directions.
Upon convergence, the forward process, parameterized path-wise by a stochastic differential equation, should transform a sample from the initial probability distribution into a sample from the final probability distribution.
Similarly, the backward process should achieve the same but in reverse.
Importantly, the two processes should converge onto the same unifying stochastic processes representing the stochastic bridge between the initial and final probability distribution.

\section{Iterative Proportional Fitting in Schr\"odinger Bridges}
\label{subsec:ipf}

The tractability of generative diffusion models, as introduced in the last chapter, rests upon the access to the solution of the diffusive forward process.
For any initial condition, we can evaluate analytically the probability and its corresponding score for any point in time $t$ and value $x_t$ from the readily available analytical probability distribution $p_{t|0}(x_t | x_0)$ of the forward process.
The solution of Schr\"odinger bridge problems is computationally more demanding.
Feasible methods are usually based on the so--called Iterative Proportional Fitting (IPF) algorithm \cite{chen2021optimal} \cite{sinkhorn1964relationship} \cite{peyre2019computational} (also known as Sinkhorn algorithm)  which can be understood from the formulation of the bridges as entropically regularized optimal transport
problems.

IPF solves the bridge problem by creating a convergent sequence of simpler forward and backward processes known as
half--bridges. For those sub--problems only \textit{one} of the two distributions at the boundaries of the time interval is kept fixed (alternating between initial and end points).
Recent algorithms differ in the way these half--bridge problems are solved.
\cite{bernton2019schr} presents a method that is based on sequential Monte Carlo techniques for efficiently sampling processes in both directions.
\cite{de2021diffusion} \cite{vargas2021solving} are,  to our knowledge,  the first papers to discuss Schr\"odinger bridges as generative data models from a machine learning perspective.
The construction of the half--bridges of the IPF algorithm is formulated in terms of an estimation for the drift functions of the corresponding SDE.
A drift function is learned from samples created by the half--bridge of the previous iteration using either a Gaussian process regression approach or by training a deep neural network.

In each iteration step $i$ two so--called {\em half--bridge} problems have to be solved.
Each half bridge consists of a reference process onto which we fit a second stochastic process with reverse reference time.
The provision of a reference process, thus fixing one half of the bridge, corresponds to the setup of diffusion models.
These half--bridges are defined by the recurrent optimization problems
\begin{align}
      \label{half_one}
      \mathbb{P}_i^* & = \arg\inf_{\mathbb{P} \in \mathbb{D}(\cdot, \pi_1) }  \KLfunc{\mathbb{P}}{\mathbb{Q}_{i-1}^{*}} \\
      \label{half_two}
      \mathbb{Q}_i^* & = \arg\inf_{\mathbb{Q} \in \mathbb{D}(\pi_0, \cdot) }  \KLfunc{\mathbb{Q}}{\mathbb{P}_{i}^{*}}
\end{align}

for $i=1,2,\ldots$ with an initial measure defined by the reference process, i.e. $\mathbb{Q}_{0}^{*} = \mathbb{Q}_{0}$ for $i=1$.
In the first half--bridge, one minimizes the KL--divergence with only the end condition $\pi_1$ fixed, whereas for the second half--bridge, only the initial condition $\pi_0$ is fixed.
As $i\to\infty$, the sequences $\mathbb{P}_{i}^{*}$ and $\mathbb{Q}_{i}^{*}$ converge to the solution of the
Schr\"odinger Bridge Problem. For a proof, see \cite{ruschendorf1995convergence}.
In Figure \ref{fig:sbp} we provide a visual intuition of IPF applied to the Schr\"odinger Bridge Problem.

\begin{figure}
      \centering
      \includegraphics[width=0.9\textwidth]{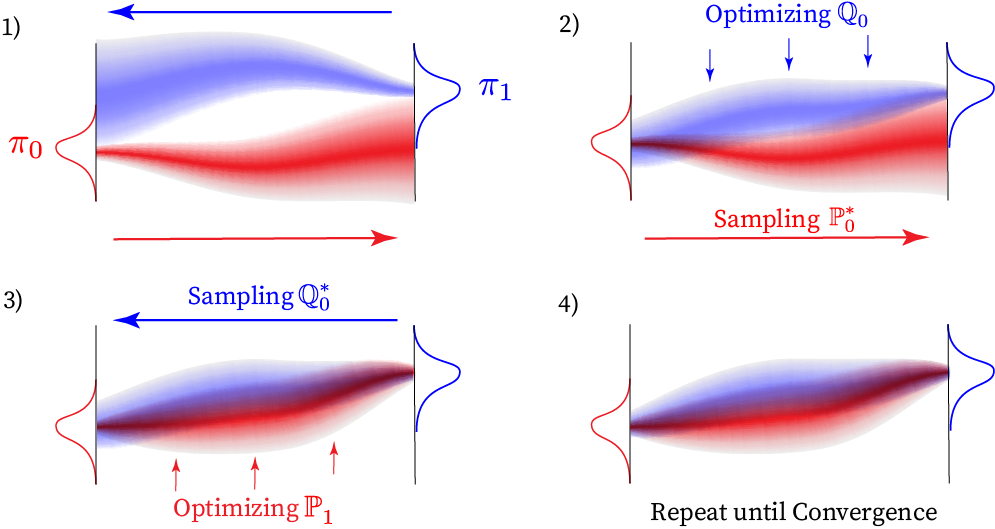}
      \caption{Visualization of the convergence of a one-dimensional Schr\"odinger bridge problem via an Iterative Proportional Fitting style optimization.
            Subplot 1) shows the initial forward $\mathbb{P}^*_0$ and backward $\mathbb{Q}^*_0$ in red and blue.
            In the first half-bridge in subplot 2), $\mathbb{P}^*_0$ is held fixed and $\mathbb{Q}^*_0$ is obtained by optimizing equation (\ref{half_two}).
            Consequently, corresponding to equation (\ref{half_one}), $\mathbb{P}_0$ is fitted on a constant $\mathbb{Q}^*_0$ in subplot 3).
            This procedure is repeated until both $\mathbb{Q}^*_i$ and $\mathbb{P}^*_i$ converge according to some predetermined criterion as indicated by subplot 4).}
      \label{fig:sbp}
\end{figure}

To solve a half--bridge problem, one can use the fact that a given SDE with drift function $\overrightharpoon{\mu}(x, t)$ can also be solved
backward in time, where the resulting backward process is also represented by an SDE.  We define the reversed time as
$\tau \doteq 1 - t$ and the backward SDE as
\begin{align}
      dZ_{\tau} =  \backward{\mu}(Z_{\tau}, \tau) d\tau + \sigma dW_{\tau}
\end{align}
with a {\em backward drift} function which is given by the conditional expectation
\begin{align}
      \label{BW_drift:def}
      \backward{\mu}(z_{1-t}, 1-t) = \lim_{h \to 0} \frac{1}{h} \mathbb{E}\left[X_{t- h} -  X_{t} | X_{t} = z\right]
\end{align}
in terms of the forward process $x_t$.
It can be shown that the statistics of the ensemble of paths $\{Z_{1-t}: 0\leq t \leq 1\}$ coincides with that
of the forward process $\{X_t: 0\leq t \leq 1\}$ when the initialization $Z_{\tau = 0}$ is drawn at random from
the density state variable $X_{t=1}$.
The KL--divergence between path measures can also be expressed in terms of the backward processes and drifts as
\begin{align}
      \label{KL_BW}
      \KLfunc{\mathbb{P}}{\mathbb{Q}} = \KLfunc{\pi_1^P}{\pi_1^Q} + \frac{1}{2\sigma^2} \int_0^1
      \mathbb{E}_\mathbb{P}\left[\left(\backward{\mu}^P(z_\tau, \tau) - \backward{\mu}^Q(z_\tau, \tau)\right)^2\right] d\tau
\end{align}

Equations \eqref{KL_FW} and \eqref{KL_BW} show that for given initial or final densities respectively, the KL divergences are minimized by matching the drift functions of the processes (the KL divergences between initial/end marginal densities equal zero). 
The loss function in \eqref{KL_BW} can be interpreted as a optimal stochastic control problem where the drift functions are used as a control of the respective stochastic processes to minimize the KL divergence between the two processes.
This notion of optimal control can also be applied to the optimization of stochastic neural networks as shown in \cite{winkler2022stochastic}.
Hence, if we assume that the mapping
\begin{align}
      \forward{\mu}(\cdot, t) \leftrightarrow \backward{\mu}(\cdot, \tau)
\end{align}

is known explicitly, the solution of the half-bridges becomes simple. The minimizer $ \mathbb{P}_i^*$ of the KL--divergence in equation \eqref{half_one} corresponds to an SDE which has the backward drift corresponding to the forward SDE given by
the process $\mathbb{Q}_{i-1}^{*}$ but is started with the density $Z_{\tau = 0} \sim \pi_1$ in backward time.
The same construction holds for  $\mathbb{Q}_i^*$ in equation \eqref{half_two}. This is given by a new SDE with a forward drift
which corresponds to the backward drift of $\mathbb{P}_i^*$, and is started from $\pi_0(x)$ in forward time.
Hence, the IPF algorithm reduces the Schr\"odinger bridge problem to the computation of backward and forward drift
functions from the corresponding forward and backward processes.

The {\em explicit relation} between forward and backward drifts was published first \cite{nelson1966derivation}, and discussed in \cite{anderson1982reverse} and \cite{nelson1988stochastic} and is given by
\begin{align}
      \label{eq:BW_driftNelson1}
      \backward{\mu}(x, \tau) & = - \forward{\mu}(x, 1-\tau) + \sigma^2 \nabla_x \ln \forward{p}_{1-\tau}(x) \\
      \label{eq:BW_driftNelson2}
      \forward{\mu}(x, t)     & = - \backward{\mu}(x, 1-t) + \sigma^2 \nabla_x \ln \backward{p}_{1- t}(x).
\end{align}

Noting the relationship between the forward time index $t$ and reverse time index $\tau = 1 - t$, $\forward{p}_{1-\tau}(x)$ is the marginal density of the state variable $x_t$.
Likewise the density $\backward{p}_{1- t}(x)$ corresponds to the marginal density of the backward state variable $z_\tau$ evaluated for $x$.
The superimposed harpoons indicate the flow of time with $\forward{\mu}$ being the drift of the forward process and $\backward{\mu}$ being the corresponding drift of the backward process.
For the interested reader, we provide a derivation of the reverse-time drift resulting in the relationship above in the Appendix \ref{app:cha2reversetimederivation}.

\begin{figure}
      \centering
      \includegraphics[width=0.9\textwidth]{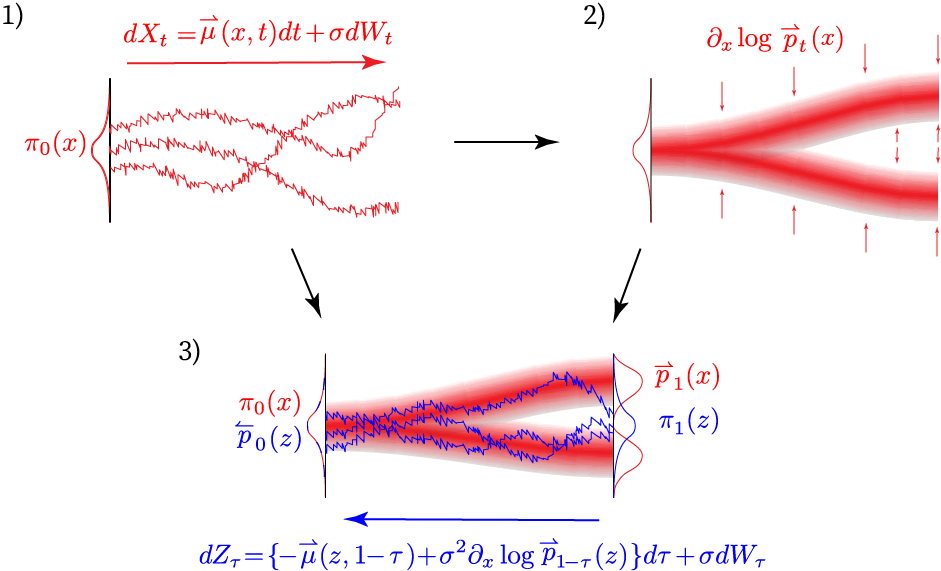}
      \caption{
            Visualization of the construction of the reverse process.
            The subplot 1) exemplifies three trajectories generated by solving the forward stochastic differential equation $dx_t$.
            The path measure $\mathbb{P}$ induces a probability distribution $\forward{p}_t(x)$ from which the score is estimated in subplot 2).
            Finally subplot 3) shows three possible trajectories of the reverse process starting from $\pi_1$ and finishing in $\pi_0$.
            A simpler visualization of the score can be found in subplot A), which serves as a figurative illustration of the behavior of the score $- \partial_x \log p(x)$ on a simple one-dimensional distribution.
      }
      \label{fig:reversesde}
\end{figure}

Figure \ref{fig:reversesde} exemplifies visually how the reverse drift can be obtained from the respective forward drift and the score of the probability distribution over paths induced by the forward SDE.
It is only possible in rare cases to compute this density analytically by solving the Fokker-Planck equation.
Luckily there is a numerical method {\em score matching} \cite{hyvarinen2005estimation} which allows for a direct estimation of the gradient of log-densities in \eqref{eq:BW_driftNelson1} from an ensemble of simulated data.
This technique is well established in the field of machine learning.

This approach has been previously suggested in the literature but deemed to be impractical \cite{de2021diffusion} for a solution of
the Schr\"odinger problem.
The direct implementation of score matching to \eqref{eq:BW_driftNelson2} in the IPF iterations would create considerable algorithmic problems.
Every full iteration of the IPF algorithm would add a score term to the previous reverse drift, which in later iterations would itself be a sum of
previous drifts and the score over the previous probability of paths which scales with each IPF iteration $i$.

If one represents both the score estimator and the current (e.g. the forward) drift by a nonlinear function approximation such as a neural network, the updated (backward) drift \eqref{eq:BW_driftNelson1} becomes the sum of two neural networks which is not easily represented as a single one.
Hence, during the iterations, one would have to store the entire sequence of past drift functions in order to compute the present one.
This would make the algorithm extremely complicated and slow as we would have to keep in memory $2i$ score matching neural networks of the $i$'th IPF iteration.
This would also mean that for a single drift evaluation in the $i$'th IPF iteration, we would have to evaluate $2i-1$ and $2i$ neural networks for a single drift evaluation at the $i$'th IPF iteration. For this reason, the score-matching approach had not been applied to the Schr\"odinger bridge problem.

Alternative approaches to computing the drift functions are based on the \textit{Euler--Maruyama} \cite{oksendal2013stochastic} temporal discretization of
forward and backward SDE. From its
conditional Gaussian transition densities, one can obtain a likelihood function for the drift function evaluated at the discrete time points.
Using samples obtained from a forward process, one can estimate the corresponding backward drift using a maximum likelihood
or Bayesian approach. This method was applied to the generation of half--bridges by \cite{vargas2021solving} where
Gaussian processes were used as a prior distribution over functions.  \cite{de2021diffusion} developed a different method
that used the conditional Gaussian transition densities of the EM discretization to approximate the score function.
This expression can then be converted into an approximation (which becomes exact in the limit when the time interval used
for discretization goes to zero) for the drift function. Both approaches could be viewed as methods for approximating the backward drift
\eqref{BW_drift:def} using a small time interval $h$ and by computing the conditional expectations within a regression framework.
This approach needs strong regularization as denoted in \cite{de2021diffusion} which required running averages of the entire function approximators to guard against fatal training divergences as the drifts were trained on local estimates dependent on the interval $h$.

In summary, previous approaches in \cite{vargas2021solving} and \cite{de2021diffusion} approximated the drifts with the expected infinitesimal change in order to yield a locally tractable reverse drift thus omitting the influence of the score necessary for  the analytical reverse process.
In this chapter, we propose to include a surrogate form of score term in the reverse drift such that the respective reverse drifts are trained to approximate the complete reverse drift and not just its localized estimates.
We will show in the following, that the representation of the drifts \eqref{eq:BW_driftNelson1} and \eqref{eq:BW_driftNelson2} can be directly estimated in a straightforward way by a modification of the score matching approach.

\section{Score matching with Reference Functions}

To simplify the notation, we will denote by $\mu(x_t, t): \mathbb{R}^{D} \times [0,1] \rightarrow \mathbb{R}^D$ one of the two drift functions and a corresponding marginal density $p_t(x)$ induced by the SDE with drift $\mu(x_t, t)$ and a scaled Wiener process with constant diffusion $\sigma$.
Following \cite{maoutsa2020interacting}, we define the following cost functional of the smooth vectorial function  $\phi(x,t): \mathbb{R}^{D} \times [0,1] \rightarrow \mathbb{R}^D$
\begin{align}
      \mathcal{L}[\phi, \mu] & = \int_0^1 dt \int dx p_t(x) \left\{\phi(x, t)^T \phi(x, t) + 2 \mu(x, t)^T \phi(x_t, t) + 2 \sigma^2 \Trfunc{ \text{J}_\phi(x, t)}\right\}
      \label{eq:costfu_ideal}
\end{align}

where $\text{Tr}[J_f(x)]$ is the trace of the Jacobian of a function $f(x): \mathbb{R}^D \rightarrow \mathbb{R}^D$.
With respect to the Schr\"odinger Bridge Problem, $\mu(x, t)$ would be the drift of the reference process and $\phi(x, t)$ would represent its reverse-time process.
We purposefully withheld the harpoons denoting the flow of time earlier as each half-bridge in equations (\ref{half_one}) and (\ref{half_two}) alternates its reference and reverse drift.
This score matching with a reference function does not require access to the true score but instead uses a surrogate function which is constructed from readily available numerical quantities.
\begin{proposition}
      For an Ito drift-diffusion process $dX_t = \mu(X_t, t) dt + \sigma_t dW_t$ with drift $\mu: \mathbb{R}^D \times [0,1] \rightarrow \mathbb{R}^D$ and the corresponding induced marginal density $p_t(x): \mathbb{R}^D \rightarrow [0,1]$, the function $\phi: \mathbb{R}^D \times [0,1] \rightarrow \mathbb{R}^D$, the cost functional
      \begin{align}
            \mathcal{L}[\phi, \mu] & = \int_0^1 \mathbb{E}_{p_t(x)} \left[ \phi(x, t)^T \phi(x, t) + 2 \mu(x, t)^T \phi(x_t, t) + 2 \sigma^2  \text{\emph{Tr}} \left[ \text{J}_\phi(x, t) \right] \right] dt
            \label{eq:proposition cost functional}
      \end{align}
      is minimized when
      \begin{align}
            \phi(x,t) = - \mu(x, t) + \sigma_t^2 \nabla_x \log p_t(x).
      \end{align}
\end{proposition}
\begin{proof}
      We refer to \cref{app:cha5scorematchingproof} for the proof.
\end{proof}

We want to note that the expectation in \eqref{eq:proposition cost functional} can be computed with respect to any distribution $p_t(x)$ and is not necessarily the marginal distribution of the process.
In the context of the Schr\"odinger Bridge Problem, the marginal distribution $p_t(x)$ is induced by the forward process with drift $\mu(x, t)$.

Hence, a comparison with eqs. \eqref{eq:BW_driftNelson1} and \eqref{eq:BW_driftNelson2} shows that the minimizer of the functional, for a given forward or backward drift, provides the corresponding reverse drift.
For a practical computation of the cost function, the integrals over time and over the unknown density in \eqref{eq:costfu_ideal} are approximated by numerically generating $N_X$ independent trajectories of the process sampled at $N_t$ regular time points $t_j$.

Hence, we approximate the cost function by its sample-based estimator
\begin{align}
      \mathcal{\hat{L}}[\phi, \mu] & =  \sum_{j=1}^{N_t}\sum_{i=1}^{N_x}\left\{\phi(x^{(i)}_{t_{j}}, t_j)^T \phi(x^{(i)}_{t_{j}}, t_{j}) +
      2 \mu(x^{(i)}_{t_{j}}, t_j)^T \phi(x^{(i)}_{t_{j}}, t_{j}) + 2 \sigma^2 \Trfunc{ \text{J}_\phi(x^{(i)}_{t_{j}}, t_{j})}\right\}
      \label{eq:costfu_est}
\end{align}

For finite sample size, the empirical cost function must be regularized by controlling the complexity of
the functions $\phi(\cdot, \cdot)$.
In contrast to \cite{maoutsa2020interacting}, we model $\phi(x,t): \mathbb{R}^{D\times[0,1]} \rightarrow \mathbb{R}^D$ by a {\em single nonlinear parametric} function which is given by a multilayer neural network (rather working with time slices a using a distinctive function of $x$ for each).
In such a way, we implicitly incorporate the smoothness of the drift in both space $x$ and time $t$.

\begin{figure}
      \includegraphics[width=\textwidth]{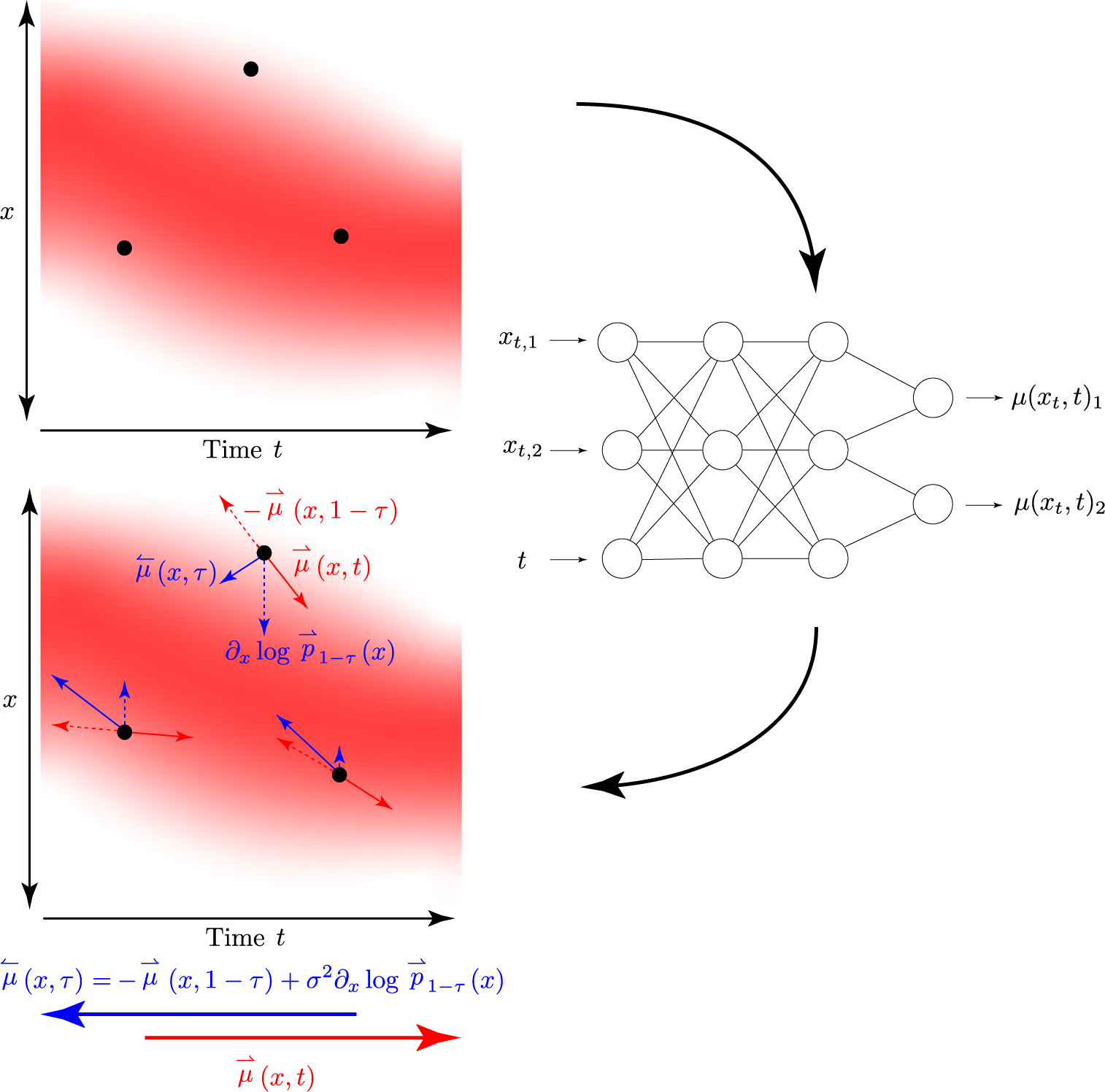}
      \caption{Visualization of the construction of the backward stochastic process parameterized by the backward SDE with drift $\backward{\mu}(x, \tau)$.
            The red gradient represents the marginal distribution $\forward{p}_{t}(x)$ induced by the forward SDE with drift $\forward{\mu}(x, t)$ and marked in the color red.
            We employ neural networks to learn both the forward and backward drift as it allows for a single function approximator per process for the entire input domain as neural networks are inherently able to model vector-valued data.
            \label{fig:neuralneworkdrift}}
\end{figure}

\subsection{Trace Estimators}

The trace of the Jacobian requires the evaluation of the derivative of a single output with respect to the single input in the same dimension $d$, independent of all other outputs and inputs.
Since we evaluate a single drift jointly for all dimensions $D$, this makes the computation of the analytical Jacobian expensive for higher dimensions as we have to do $D$ independent backward passes to compute each entry in the Jacobian matrix.
The number of gradient computations required for the Jacobian in a vector-valued function thus scales quadratically with the number of dimensions.

For data with few dimensions, computing the diagonal terms of the Jacobian can be done via batched backpropagation of one-hot output gradients.
This approach falters computationally and memory-wise when we consider data in higher dimensions.
For higher dimensional data we opt for the trace estimation trick of Hutchinson with samples from an i.i.d. Rademacher distribution in $\mathbb{R}^D$, namely
\begin{align}
      \tr{J_\phi(x, t)} = \Efuncc{z \sim p(z)}{z^T \nabla_x \left[ \phi(x, t)^T z \right]}.
\end{align}

An elaboration on the trace estimation trick can be found in appendix \ref{appendix:traceestimationtrick}.
The main idea of the stochastic approximation of the trace is that we are only interested in the diagonal elements of the Jacobian.
The Hutchinson trace estimation trick proceeds by computing the derivative with respect to the input of the inner product of the prediction $\phi(x, t)$ and a random variable $z$.
The same random variable is then applied a second time in an inner product to obtain an approximation of the scalar quantity of the trace of the Jacobian.
The advantage of the trace estimation trick is that only a single additional derivative evaluation on the inner product $\phi(x, t)^T z$ is required which scales linearly with the number of samples of $z$.

Taking the gradients of the loss with respect to the parameters requires a second derivative such that a function approximator trained on the reverse drift has to be twice differentiable.
Enforcing this property in neural networks requires us to use at least twice differentiable evaluations of the prediction with respect to the spatial input $x$ which necessitates twice differentiable activation functions such as the hyperbolic tangent or Gelu activation functions \cite{hendrycks2016gaussian}.

If the function approximator is only once differentiable as with the use of Rectified Linear Unit activation functions \cite{agarap2018relu}, we can employ Stein's lemma to estimate the trace of the Jacobian.
For this estimator, following \cite{boffi2022probability} we define an isotropic Gaussian perturbation distribution $z \sim \mathcal{N}(x, \sigma_z^2 I)$ with $x, z \in \mathbb{R}^D$ around each data point $x_t \in \mathbb{R}^D$ and average the gradients in the $z$-neighborhood of the data point
\begin{align}
      \text{Tr}[ J_\phi(x, t)] = \lim_{\sigma_z \downarrow 0} \Efuncc{p(z)}{\phi(x + z, t)^T \frac{z}{\sigma_z^2}}
\end{align}
The derivation of Stein's lemma can be followed up in appendix \ref{appendix:steinslemma} and it's application to the trace estimation therefrom is detailed in appendix \ref{appendix:steingradients}.
For a practical implementation,  we approximate the Stein--estimator using a sufficiently small $\sigma_z$ by drawing only a {\textit single}
random vector $z^{(i)}_j$ for each trajectory $i$ and each time point $j$.
The perturbed $\phi$ function values can be computed along side the unperturbed values in a single forward pass through the neural network.

\section{Results}

We evaluate our proposed method on artificially generated data sets with varying dimensions and with and without dependencies between the
dimensions at the two target marginals.
The first set of experiments were done on the construction of the Schr\"odinger Bridge between Gaussian Mixture Models with which we could explore the behavior with a changing set of dimensions.
The second collection of experiments focused on manifold learning in which implicit distributions were learned to be generated from a standard normal distribution.
Finally, we employed our proposed framework on the generation of intermediary distributions of embryoid single-cell RNA as a real-world application.
The general experimental setup for all experiments can be followed up in \cref{app:cha5experimentalsetup}.

\begin{figure}[H]
      \centering
      \includegraphics[width=\textwidth]{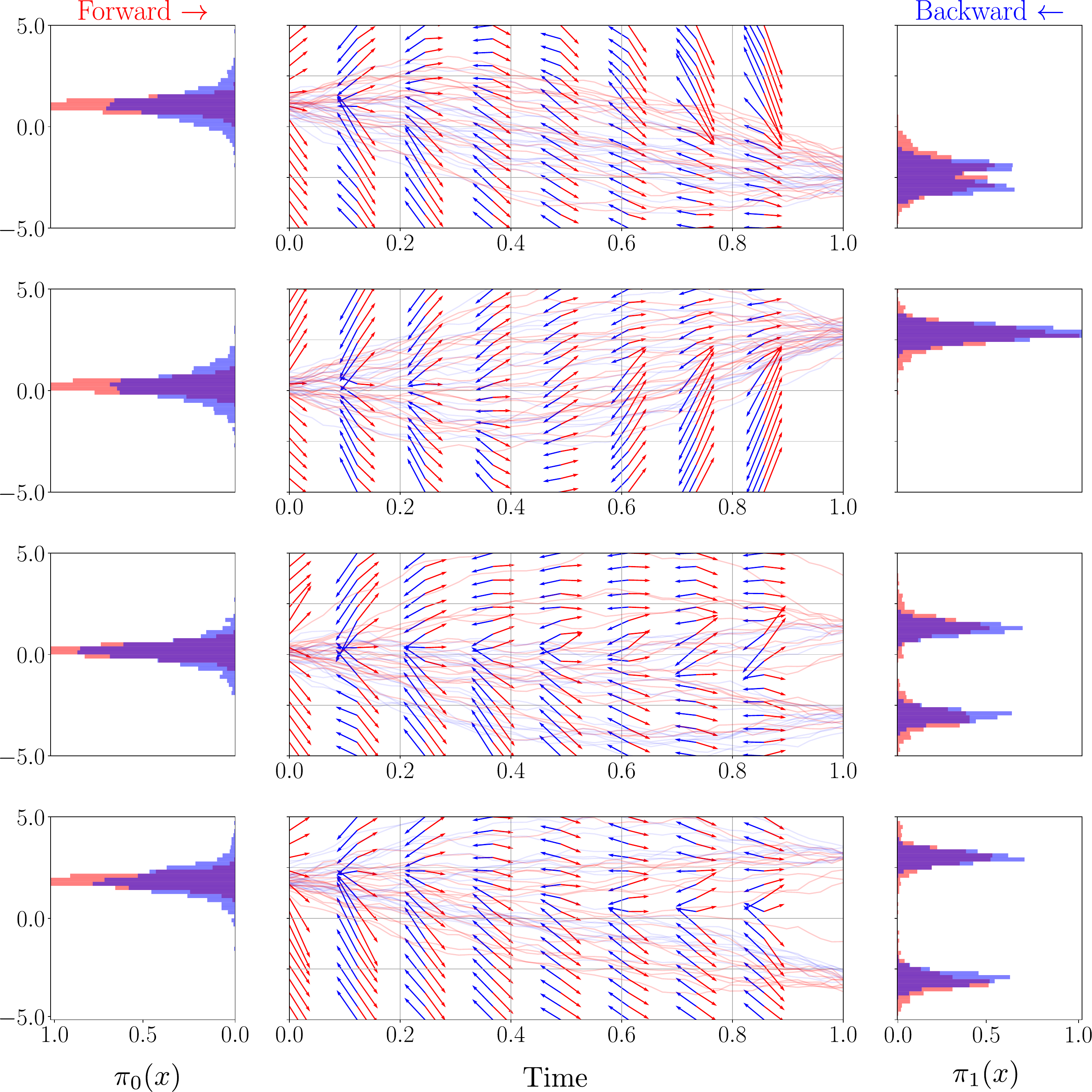}
      \caption{Visualization of a solved Schr\"odinger Bridge in $\mathbb{R}^4$ in which each dimension is plotted independently.
            The sampled trajectories and drift of the forward process and its initial condition $\pi_0$ is shown in red, whereas the backward process is shown in blue.
            \label{fig:ndviz}}
\end{figure}

The question remains how the reference process $\mathbb{Q}^*_0$ should be chosen for our applications.
The authors in \cite{de2021diffusion} used an Ornstein-Uhlenbeck (OU) process \cite{uhlenbeck1930theory} for $\mathbb{Q}^*_0$.
Its marginal distributions can be computed analytically and do not require solving an SDE numerically which saves time and computational resources during the very first half bridge.
This choice of a Gaussian reference process could be also motivated from the fact that for larger times $t$ the marginal of the process converges to a stationary Gaussian density that could approximately match Gaussian targets used for a denoising style data generating application of Schr\"odinger bridges.

In our implementations, we did not want to make any specific assumptions on end marginals.
Hence, a choice of a zero initial drift $\mu(x, t)\equiv 0$ (reducing $\mathbb{Q}^*_0$ to a simple Wiener process) seemed more natural.
However, practical considerations have suggested a slightly different approach, in which $\mathbb{Q}^*_0$ corresponds to a drift function represented by a neural network which has (untrained) small random weights which serve as useful initial conditions for the subsequent training \cite{glorot2010understanding} \cite{he2015delving}.
We observed experimentally that for the first half-bridge $\mathbb{Q}^*_0$,  the Wiener process dominates the characteristics of the sampled trajectories.

For our applications the marginal densities at initial and end times $\forward{p}_0$ and $\forward{p}_1$ which are generated by the bridge models as well as the desired targets $\pi_{0}$ and $\pi_1$ are represented by random samples rather than by analytical expressions. To evaluate the quality of the converged bridge we
computed the Wasserstein-1 distances $W_1(\forward{p}_1, \pi_1)$ and $W_1(\pi_0, \backward{p}_0)$.
The Wasserstein-1 distance \cite{villani2009optimal} between two probability measures $\mu$ and $\nu$ is defined as
\begin{align}
      W_1(\mu, \nu) = \inf_{\gamma \in \mathbb{D}(\mu, \nu)} \Efunc{||x-y||}
\end{align}
where $\mathbb{D}(\mu, \nu)$ is the set of all couplings of $\mu$ and $\nu$.
These can be straightforwardly evaluated on empirical distributions which are given by samples.
For the underlying optimal transport problem and its efficient solution via linear programming in its dual representation see e.g. \cite{peyre2019computational}.
The Wasserstein-1 distance is also known as the Earth Movers Distance (EMD) in computer science.

\subsection{Multimodal Parametric Distributions}

We model the marginal distribution $\pi_0(x)$ as a Gaussian distribution with a diagonal covariance matrix.
The opposite marginal distribution $\pi_1(x)$ was a Gaussian Mixture Model with two modes with a uniform prior over the GMM component centers.
The mean values of all Gaussian distributions, uni-modal in $\pi_0(x)$ as well as bi-modal in $\pi_1(x)$ were sampled uniformly from $\mathcal{U}(-2.5, 2.5)$ and a standard deviation of 1.0 was used throughout.
The visualization of the inferred Schrödinger Bridge highlights the learning of the time-dependent drift and the ability to model bifurcations in the case of bi-modal GMMs as seen in Figure \ref{fig:ndviz}.

\begin{figure}[tbp]
      \centering
      \includegraphics[width=0.6\textwidth]{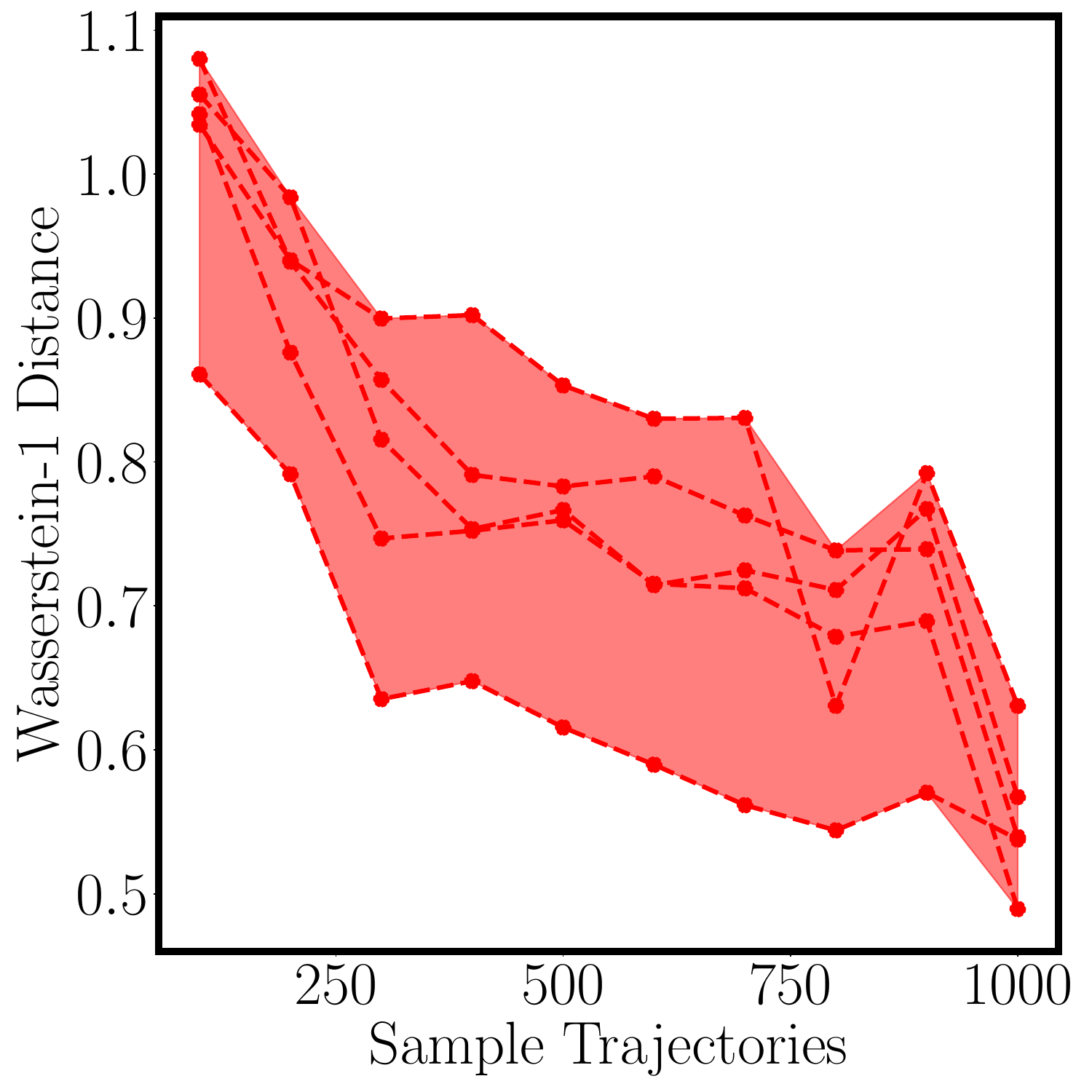}
      \caption{A comparison of the Wasserstein-1 distance on the number of sample trajectories sampled per IPF iteration for the Multi-Modal distribution problem in section 3.2.
            The Wasserstein-1 distances $W_1(\forward{p}_1, \pi_1)$ and $W_1(\pi_0, \backward{p}_0)$ were averaged and show an overall decrease relative to the number of trajectories sampled from the stochastic processes.
            \label{fig:num_trajs}}
\end{figure}

The data set was created as the more tractable experiment in comparison to subsequent data sets.
The use of GMMs allows for the analytical evaluation of the probability of the generated data at the marginal distributions.
Furthermore, the modes of the GMM could be handcrafted which turned out important to validate numerous design choices of the drift approximators.

Estimating the score in regions with little probability mass is a difficult problem as an insufficient amount of samples may be drawn from that region to accurately model the score.
This was explored in detail in \cite{song2020score} and \cite{song2019generative} and was one of the inspirations for the perturbation protocol in diffusion models.
Diffusion models have the advantage of computing an analytical score at any point in space and time through their analytical perturbation kernels defined in the forward process.

The Schrödinger Bridge Problem offers none of these luxuries as no reference process is available which yields analytical scores.
Thus we require sufficient data even in low-probability regions to enable us to estimate the necessary score.
We found in our experiments that the hyperparameter with the single largest influence was the number of trajectories that were sampled from the path measures.
We can thus see in Figure \ref{fig:num_trajs} that increasing the number of trajectories decreases the Wasserstein-1 distance with respect to the marginal distributions $\pi_0(x)$ and $\pi_1(x)$ as the score estimation becomes more precise as even low probability regions of the stochastic processes, on which the score is estimated, are sampled adequately.

We chose the multi-modal data set as a test data set to compare Hutchinsons trace estimator with the trace estimation via Stein's lemma.
The result can be seen in Figure \ref{fig:steinvshutchinson} in which we evaluated the Wasserstein-1 distance over a fixed set of dimensions.
As baselines, the Wasserstein-1 distances between the respective marginals are shown.
One can see that the gradient estimators perform as well the Wasserstein-1 distances between samples drawn from the marginals denoted as $W_1(\pi_0(x), \pi_0(x))$ and $W_1(\pi_1(x), \pi_1(x))$.

The Hutchinson trace estimator performs best while requiring a second derivative.
Interestingly, smaller sampling variances for the Gaussian distribution in the Stein gradient estimator yield better overall performance on matching the marginal distributions.
We observe that the Wasserstein-1 distance increases linearly with the number of dimensions for a fixed number of samples and it was interesting to observe that the Schrödinger Bridges with different trace estimators behaved very similar in terms of performance.

In theory, Hutchinsons trace estimation can also be applied with a Gaussian distribution sampling the random projection vectors, yet, Rademacher's distribution has the lowest estimator variance.
Stein's trace estimation can only be done with the Gaussian distribution, as it relies on integration by parts and the special derivative of the Gaussian distribution.
This leads us to hypothesize that the inherently higher variance of the Gaussian distribution in Stein's trace estimator leads to worse performance.

\begin{figure}[tbp]
      \centering
      \includegraphics[width=0.7\textwidth]{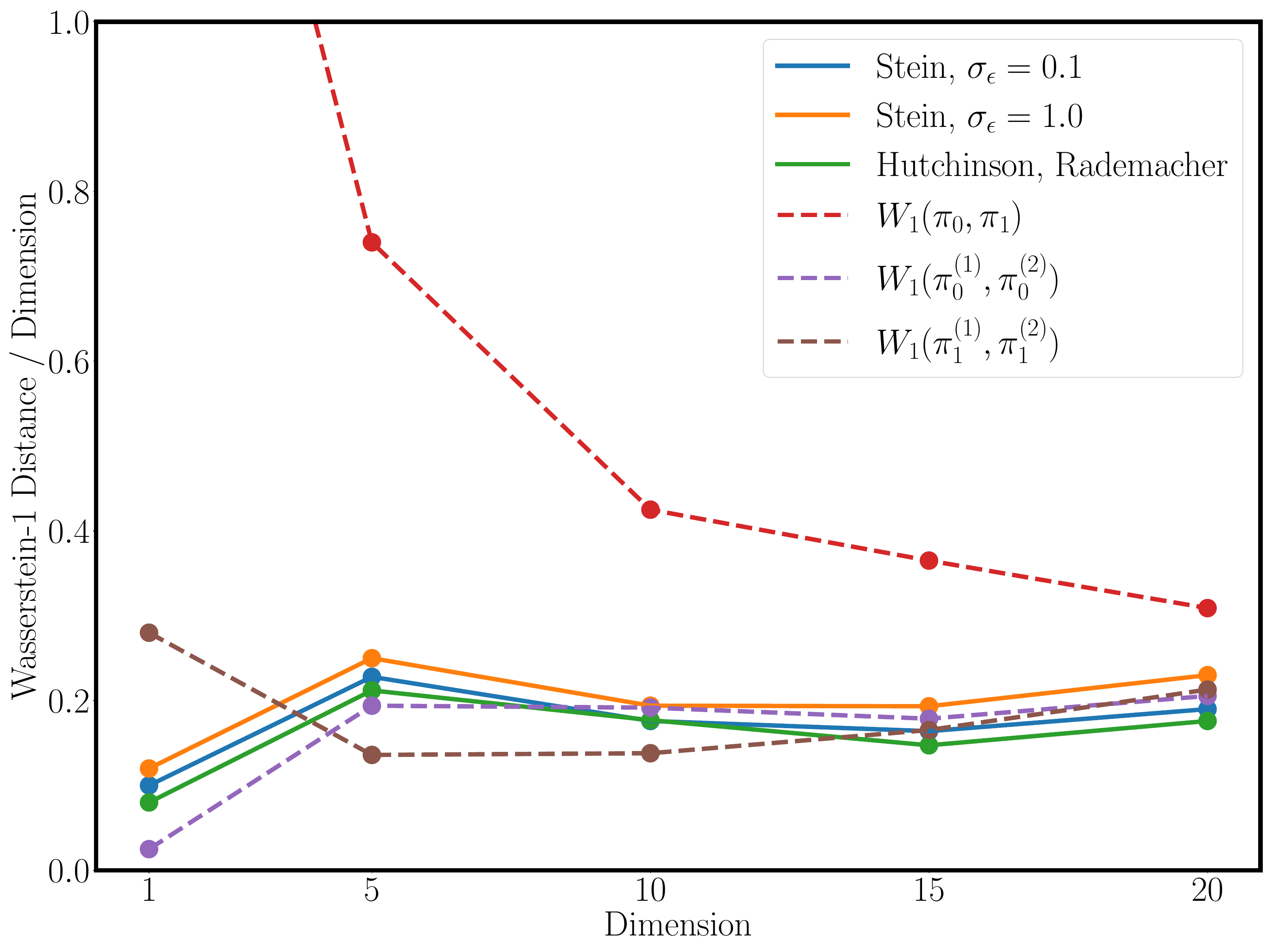}
      \caption{
            A comparison of the gradient estimators on an increasing number of dimensions with the average Wasserstein-1 distance per dimension between the true marginal and the predicted marginal distributions of the forward and backward process.
            \label{fig:steinvshutchinson}}
\end{figure}

\subsection{Manifold Data Sets}

For the second data set we trained the forward and backward drifts on generated manifold data via the Sklearn machine learning package \cite{buitinck2013api}.
The manifolds used were 'make\_swiss\_roll', 'make\_s\_curve' and 'make\_moons' from the sklearn.dataset code base which were concatenated to create a higher dimensional manifold.
This increased the complexity of the manifold setting it apart form earlier manifold modeling approaches as in \cite{de2021diffusion}.

We trained the stochastic processes to predict multiple manifolds at once by modeling them jointly with a single fully connected neural network.
For $\pi_0(x)$ we chose a standard normal Multivariate distribution $\pi_0(x) = \mathcal{N}(0, I)$ while $\pi_1(x)$ was the implicit distribution generated by samples on the manifold.
Whereas the previous Gaussian Mixture Models were statistically independent in each dimension, the manifolds explicitly model statistical correlation between different dimensions.
A visualization of the Schrödinger Bridge between the two distributions can be seen in Figure \ref{fig:noisesklearn}.

The use of a tractable probability distribution as one marginal distribution is inspired by the purely generative task of diffusion models.
The data set is commonly used as a visual benchmark of new generative models and allows to evaluate the drift approximators ability to model non-linear manifolds.
A lack of this data set is the absence of a tractable data likelihood under the marginal distribution, as the manifold generation function are modeled as an implicit distribution.

\begin{figure}[tbp]
      \centering
      \includegraphics[width=\textwidth]{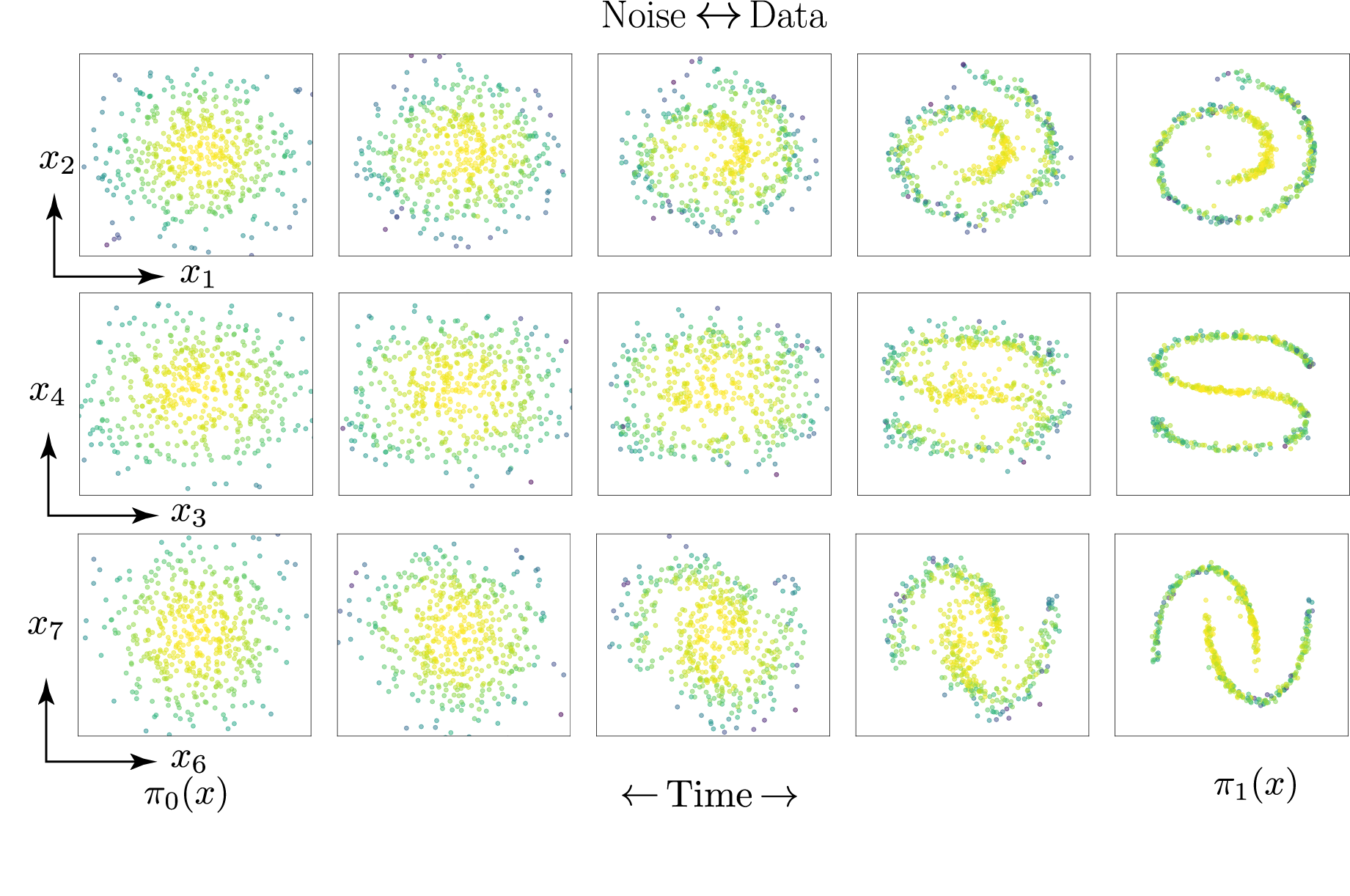}
      \caption{Visualization of a constructed Schrödinger Bridge for the manifold $\mathbb{R}^{6}$ which lies in $\mathbb{R}^{10}$.
            The left-most column represents samples from the marginal distribution $\pi_0(x)$ which are transformed into samples on the manifolds in the right-most column.
            The color coding of each particle is according to its proximity to the mean of the prior distribution $\pi_0(x)$ such that we can distinguish which particles from $\pi_0(x)$ correspond to the particles on the manifold $p(x_1, 1)$
            \label{fig:noisesklearn}}
\end{figure}

\subsection{Embryoid Data Set}

Single-cell RNA sequencing analyzes the RNA of individual cells, destroying it unfortunately and making it inaccessible to further analysis \cite{tong2020trajectorynet, moon2019visualizing, hendriks2019nasc, la2018rna}.
In a population of cells, we can remove individual cells and analyze their RNA.
As each cell is eliminated from the population, we have to turn to a probabilistic method to simulate the development of the RNA at the \emph{population} level.

Therefore, it is of interest to develop a methodology which can simulate full trajectories of RNA sequences over time
which would allow for predicting outcomes of such measurements on a single cell without actually having to perform them.
As suggested in \cite{moon2019visualizing}, an interesting solution to this problem would be the construction of a stochastic
generative model for the possible measurements with the marginal distributions (for a population of cells) of the actual measurements at initial and end time as boundary conditions.
A visualization of the application of the Schr\"odinger Bridge to the RNA measurement task is provided in Figure \ref{fig:embryoid}.
If we represent this generative model by a diffusion process, we naturally end up with the idea of applying Schr\"odinger Bridges to this problem.

\begin{figure}[htbp]
      \centering
      \includegraphics[width=0.9\textwidth]{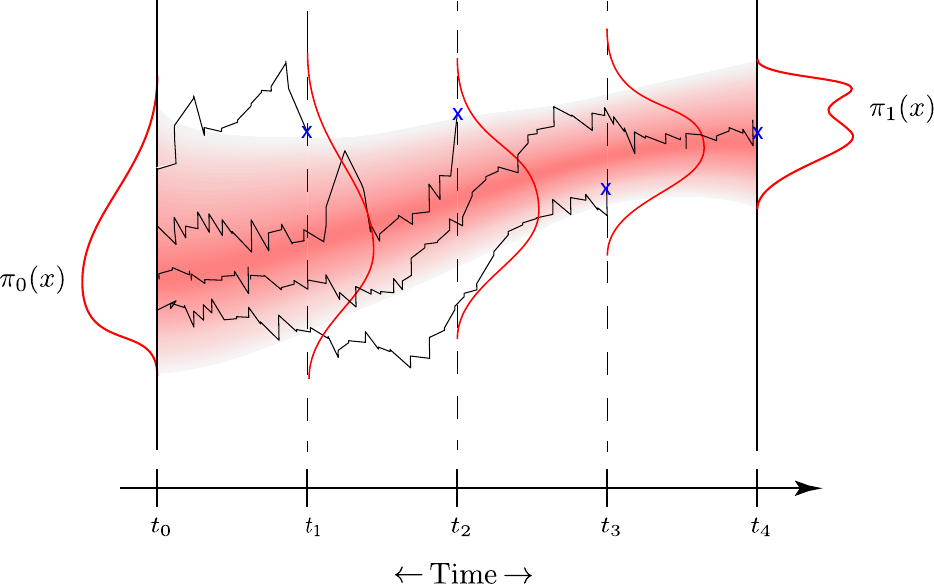}
      \caption{
            A visualization of how the Schr\"odinger Bridge is applied to the RNA measurements.
            The RNA of different cells are measured at the discrete time steps $t_1$, $t_2$ and $t_3$ and eliminated from the population marked by the blue crosses and the end of each of the four trajectories.
            Given these snapshots of the RNA distribution at discrete time steps, the task is now to construct a Schr\"odinger Bridge between the discrete time steps as indicated by the stochastic process with a fading color.
            The generative model can then be queried at any time step between $t_0$ and $t_4$.
            \label{fig:embryoid}}
\end{figure}

We apply this approach to the embryoid data set \cite{tong2020trajectorynet} for which
single cell RNA measurements are taken at five different times $t \in \{0,1,2,3,4\}$.
The ensemble of measurements at each time index $t$, consists of varying numbers (2380, 4162, 3277, 3664, 3331) of RNA measurement samples.
We consider the problem of constructing a Schr\"odinger Bridge using only the first day and the very last day measurements.
We then try to infer the hold--out intermediate marginal distributions of the RNA measurements with the help of the constructed bridge.
The samples drawn from the bridge are thus evaluated at the intermediate marginal distributions at $t \in \{1,2,3\}$
using the Wasserstein-1 distance.
In addition, we also evaluate the Schr\"odinger bridge at the two end--points by computing
$W_1(\backward{p}_0, \pi_0)$ and $W_1(\forward{p}_1,\pi_1)$.
Finally, we want to highlight that we define the above-described problem purely as a benchmark with real-world data and did not interpret our results in a biological context.

We compare our method with two other generative model approaches:
\textit{TrajectoryNet} \cite{tong2020trajectorynet} implements constrained normalizing flows using neural networks. The paper
was also the first to propose this benchmark. IPML \cite{vargas2021solving} presents an alternative method to constructing a Schr\"odinger
bridge in which the required drift functions are estimated from the trajectories using Gaussian process regression
without computing score functions as remarked upon in section \ref{subsec:ipf}.
Finally, we have also included a simpler method which is based on optimal transport (OT). It  computes a linear transport map
between samples of the initial and end distributions.
However, the comparison of this method with the others is slightly unfair. While the OT approach allows for making predictions
of the marginals at intermediate time steps, it is not formulated as a generative model and thus could not be used to make predictions for measurements on single cells.

The RNA-sequencing data was generated from FACS sorted embryoid bodies via surface marker indication.
The data set was preprocessed with PHATE \cite{moon2019visualizing} to a dimensionality of $D=5$ as done in the reference methods.
An important feature of the PHATE preprocessing is the reversibility of the dimensionality reduction.
The preprocessing pipeline can be accessed and replicated with the public \href{https://github.com/KrishnaswamyLab/PHATE}{PHATE code base} provided by the authors of \cite{moon2019visualizing}.
Thus it is possible to apply trajectory reconstructing algorithms in the low dimensional representation of the genetic data and project the resulting trajectory back into the high-dimensional space.
This is in contrast to commonly used dimensionality reduction algorithms which do not offer these advantages.
We refer the reader to \cite{moon2019visualizing} for an in-depth treatment of the algorithm.
The experimental setup was kept identical for \textit{IPML} and \text{TrajectoryNet}.

Table \ref{tab:embryoidperformance} compares the performance of our Schr\"odinger Bridge method with
the other methods.
We outperform both Trajectory-Net and the alternative Schr\"odinger bridge method IPML and are also better than
the OT approach in two out of three instances.
This corroborates the result from \cite{vargas2021solving} that OT's linear transport map is not a well fitting intermediate distribution for $t=4$.
Especially the Wasserstein distance at the marginal distributions of the backward process at $t=0$ and the forward process $t=4$ are well modeled.
      {
            \begin{table}[H]
                  \renewcommand{\bf}[1]{\normalfont\bfseries{#1}}
                  \centering
                  \setlength{\tabcolsep}{2pt}
                  \begin{tabularx}{0.7\textwidth}{m{30mm}ccccccc}
                        \toprule
                        \textbf{Method $\downarrow$} & \textbf{t=1} & \textbf{t=2} & \textbf{t=3} & \textbf{t=4} & \textbf{t=5} & \textbf{Path} & \textbf{Full} \\
                        \midrule
                        TrajectoryNet                & 0.62         & 1.15         & 1.49         & 1.26         & 0.99         & 1.30          & 1.18          \\
                        IPML EQ                      & 0.38         & 1.19         & 1.44         & 1.04         & 0.48         & 1.22          & 1.02          \\
                        IPML EXP                     & 0.34         & 1.13         & 1.35         & 1.01         & 0.49         & 1.16          & 0.97          \\
                        OT                           & N/A          & 1.13         & 1.10         & 1.11         & N/A          & 1.16          & N/A           \\
                        Reverse-SDE                  & \bf{0.23}    & 1.16         & \bf{1.00}    & \bf{0.64}    & \bf{0.28}    & \bf{0.93}     & \bf{0.66}     \\
                        \bottomrule
                  \end{tabularx}
                  \caption{Comparison of comparable methodologies on the Embryoid data set with the Wasserstein-1 distance
                        The 'Path' column denotes the average EMD of the intermediate time steps $T \in \{2,3,4\}$ whereas the 'Full' column averages all time steps.
                        The optimal transport linear transport map is only defined for the intermediate time steps as it requires the two marginal distributions at $T\in \{1,5\}$.}
                  \label{tab:embryoidperformance}
            \end{table}
      }

\section{Discussion and Limitations}

We have presented a method for solving the Schr\"odinger Bridge problem
based on the Iterative Proportional Fitting algorithm. In contrast to the simpler diffusion models for data generation,
Schr\"odinger Bridges can provide interpolations between \textit{two} arbitrary data distributions. Unfortunately, this
advantage comes with the drawback that score functions which can be used to compute drift functions for the time reversed process
are not analytically available.
The novel aspect of our approach is a formulation of reverse time drift functions as solutions of a minimization problem (an extension of the score matching method) involving expectations over the forward process.
This allows for the efficient training of neural networks
representing the drift functions on simulated trajectories of the corresponding stochastic differential equations.

The core advantage and disadvantage of our method is the lack of the analytical score function in the computation of the time-reversible dynamics in both directions of time.
Naturally, we are not bound by choosing particular stochastic dynamics for the forward and backward process.
Yet, this also forfeits the analytical score function which is available in the case of carefully chosen dynamics.
Therefore, the optimization of the respective neural networks approximating the drift functions is more challenging.
It takes considerable computational resources to train the neural networks to accurately model the score function.
In fact, we have seen that the scaling of these computational resources is not linear with the dimensionality of the problem.
This is a well-known problem in the field of generative models and is not unique to our approach.
The authors in \cite{song2019generative} have also observed that the number of samples drawn from the stochastic process is the most important hyperparameter to ensure convergence of the Schr\"odinger Bridge.
Without an analytical score function, the number of samples drawn from the stochastic process is the most important hyperparameter to ensure convergence of the Schr\"odinger Bridge which incurs high computational costs.


Our variational approach of estimating drift functions in the IPF algorithm for solving Schr\"odinger bridges
could be extended in various ways.
The variational formulation of our drift estimators is independent of the temporal discretization
used for creating sample trajectories. It only depends on the \textit{marginal distributions of the state variables}. This fact opens
up alternative possibilities for generating appropriate samples by forward and backward simulations which
could allow for larger step--sizes. One might e.g. consider \textit{weak} approximation schemes \cite{kloeden2002numerical}
for numerically simulating SDE.  A different alternative is the application of \textit{deterministic}, particle based simulations
\cite{maoutsa2020interacting} where the reduced variance of estimators might allow to keep the number of trajectories small.
Finally, \textit{exact sampling} methods (see e.g. \cite{beskos2006exact}) which entirely avoid temporal discretization,
would be interesting candidates.

Reliably estimating the score-based drift functions in regions with small marginal probability densities
remains a challenging problem, especially for higher dimensions.
This is evident in our synthetic experiments, as the number of trajectory samples remained the most important hyper parameter to ensure convergence of the Schr\"odinger Bridge (see also \cite{song2019generative} for similar observations). It would be interesting to see if prior knowledge
expressed by exact analytical results for asymptotic scaling of densities could be implemented in
the function approximators to improve on that problem.

It is relatively straightforward to adapt the variational approach (see e.g. \cite{maoutsa2020interacting}),
together with a corresponding change in the relations \eqref{eq:BW_driftNelson1} and \eqref{eq:BW_driftNelson2} to solve
Schr\"odinger bridge problems for more general types of stochastic differential equations. Interesting cases
could include SDE with (fixed) state and time-dependent diffusion matrices but also processes based on  \textit{Langevin dynamics}
(i.e. systems of second order SDE) well known for modeling physical systems which are also used for Hamilton Monte--Carlo simulations.

%% file: chapters/chapterConclusion.tex
\chapter{Conclusion}
\label{cha:conclusion}

Throughout this thesis, we have presented novel approaches to infer time-reversible dynamics with boundary conditions imposing constraints on the solution of the underlying differential equation.
We showed that neural networks can be successfully employed to learn deterministic, time-reversible dynamics in boundary value problems, discrete stochastic dynamics in half bridges, and full stochastic bridges in continuous state spaces.
Importantly, the dynamics were \emph{learned} by neural networks in a data-driven way, which is in contrast to the traditional approach of solving apriori known differential equations by numerical integration.

In particular, the contributions of this thesis can be summarized in the following three parts:
\begin{itemize}
      \item In Chapter \ref{cha:deterministicbvps}, a novel approach to learn deterministic dynamics in boundary value problems was proposed by training neural networks as time-reversible integrators.
            This approach enables the learning of the dynamics of molecular systems in a data-driven manner and reconstructs the vibrational spectra of complex molecular systems with high fidelity.
      \item In Chapter \ref{cha:discretehalfbridge}, a diffusion model for time-reversible, discrete-space jump processes based on the Ehrenfest process was introduced.
            A rigorous analysis of the convergence of the Ehrenfest process to the Ornstein-Uhlenbeck process was provided and utilized to show that the dynamics of the Ehrenfest process can be learned by training a neural network with a regression loss to learn the OU denoising objective of DDPM in a novel way.
      \item In Chapter \ref{cha:stochasticbridge} the Schrödinger Bridge Problem was approached with machine learning by parameterizing the drifts of two separate Ito drift-diffusion processes with neural networks.
            The intractable score term was reformulated with a surrogate loss function which allowed learning a bridge between two probability distributions by reverting otherwise intractable stochastic processes.
\end{itemize}
\section{Summary}

\vspace{0.25cm}
\textbf{Learning Deterministic Dynamics in Boundary Value Problems}
\vspace{0.25cm}

In \Cref{cha:deterministicbvps} we formulated an approach to learning time-reversible deterministic molecular dynamics trajectories.
Using the modeling capacity of neural networks as powerful universal function approximators, we demonstrated the capability to capture intricate dynamical behaviors inherent in molecular systems, thereby paving the way for accurate and efficient trajectory modeling.
Remarkably, we were able to reconstruct the dynamics of complex molecular systems with high fidelity, showcasing the potential of machine learning in computational chemistry.

By using time-reversible recurrent neural networks, we provided the blueprint to learn the dynamics of molecular systems in a data-driven manner by using the neural networks as time-reversible integrators with significantly lower computational overhead.
Importantly, the relevant trajectories in both the forward and the backward direction were available and could be learned simultaneously from ground truth data.
We saw that the choice of the neural network architecture and the training strategy had a significant impact on the performance of the model.
Experimentally, the convincing reconstructive capabilities of the vibrational spectra with a machine learning-based integrator provided ample evidence for the potential of the approach.
As the training data was subject to thermal noise, the model had to be robust to noise and integrate trajectory information over time.
We found that Long Short-Term Memory networks with memory cells were more suitable for learning the dynamics of molecular systems than other network architectures.

\vspace{0.25cm}
\noindent\textbf{Inferring Stochastic Jump Processes in Discrete Half-Bridges}
\vspace{0.25cm}

In \cref{cha:discretehalfbridge} we considered the learning of discrete half bridges by learning the dynamics of a stochastic jump process.
We presented a novel diffusion model for time-reversible, discrete-space, continuous time stochastic jump processes.
In particular, we rediscovered the Ehrenfest process which is a birth death process.
We successfully trained a neural network to learn the dynamics of the Ehrenfest process and achieved comparable performance in the image generation benchmarks of MNIST and CIFAR-10.

A particular property of the Ehrenfest process is that it converges to the Ornstein-Uhlenbeck process in the limit of infinite states.
We provide a rigorous analysis of the convergence of the Ehrenfest process to the Ornstein-Uhlenbeck process.
This convergence is a key insight into the interplay between discrete and continuous diffusion processes and creates a direct link between the conditional score of the Ornstein-Uhlenbeck process and the birth and death rates of the Ehrenfest process.
Building on this connection, we provided a novel approach to learning the dynamics of the Ehrenfest process by training a neural network with a regression loss to learn the OU denoising objective of DDPM.
This stands in contrast to previous work which learned the predictive distribution over the initial conditions.
Subsequently, we could model a continuous time Markov chain with the corresponding score of a continuous space model for the first time.

\vspace{0.25cm}
\noindent\textbf{Tackling the Schrödinger Bridge Problem with Machine Learning}
\vspace{0.25cm}

In \cref{cha:stochasticbridge} we considered learning stochastic dynamics to connect two continuous probability distributions.
Previous chapters either had access to deterministic ground truth solutions connecting their boundaries or had a reference process available connecting the boundary probability distributions.
The so-called Schrödinger Bridge Problem is posed without any information on how to connect these two probability distributions.
Consequently, the stochastic dynamics inducing a connection between the two provided boundary probability distributions have to be inferred.
We followed the Iterative Proportional Fitting methodology and parameterized the drifts of two separate Ito processes with neural networks.
To reverse the stochastic processes correctly in time, we reformulated the intractable score term with a surrogate loss function.
The resulting criterion is substituted via integration by parts of the Jacobian of the model for the intractable gradient of the marginal distributions of the reference stochastic process.
This enabled the learning of the full bridge in the space of probability distributions.

For computational reasons, we employed stochastic trace estimation methods to accelerate the computation of the surrogate score term.
Numerically, we conducted a series of experiments evaluating the performance and the influence of the hyperparameters of the model and optimization on two illustrative data sets and a bacteria data set.
We found that the model was able to learn the dynamics of the bridge and accurately model the evolution of the bacteria populations between observations.


\section{Discussion and Future Directions}

This thesis has demonstrated the potential of using machine learning to learn time-reversible dynamics constrained by boundary conditions across various dynamic systems. This approach introduces both significant advantages and notable limitations that are crucial for understanding its broader implications and opportunities for further research.

\vspace{0.25cm}
\noindent\textbf{Neural Networks Integrators for Initial Value Problems in MD}
\vspace{0.25cm}

Training neural networks as time-reversible integrators for molecular dynamics simulations has shown promising results in reconstructing the dynamics of complex molecular systems \cite{winkler_2022}.
The time-reversible nature of the deterministic part of the dynamics simplifies the learning process and allows for the reconstruction of higher-frequency components of the molecular dynamics simulation.
In terms of applications, our approach from \cite{winkler_2022} has also been successfully applied to non-adiabatic Hamiltonians in \cite{wang2023interpolating}.
However, the requirement of providing the final condition for the dynamics poses a limitation, as it necessitates solving the dynamics forward until the final condition.
We see the potential to extend the approach to the learning of the dynamics of molecular systems in a multi-scale fashion posed as an initial value problem \cite{noe2020machine, kohler2023flow}.
The key goal of molecular dynamics is the fast and accurate simulation of the dynamics to estimate macroscopic observables reliably \cite{wang2019machine}.

Recently, there have been impressive advancements in protein structure prediction, such as AlphaFold \cite{jumper2021highly, varadi2022alphafold, abramson2024accurate}, which can predict crystallized protein structures from amino acid sequences.
Whereas molecular dynamics provides a dynamical simulation of the protein's conformational changes over time, models like AlphaFold are a specialized tool that can predict static protein structures with high accuracy and speed, providing a supplementary method to molecular dynamics simulations for certain applications \cite{unke2024biomolecular}.

In this work, we have used neural networks as time-reversible integrators to learn the dynamics of molecular systems in a data-driven manner, which reduces the computational cost at each step of the molecular dynamics simulation to the evaluation of a neural network.
This has clear computational advantages at the cost of precision.
As a well-established alternative, force fields can be learned directly from data and have been shown to be competitive with traditional force fields in terms of accuracy and efficiency and modeling existing symmetries \cite{unke2021machine, chmiela2017machine, unke2021spookynet, chmiela2023accurate}.

As yet another alternative methodology, generative modeling has been employed to directly model the probability distribution of the molecular system \cite{noe2019boltzmann, wu2020stochastic, kohler2019equivariant, kohler2021smooth, vaitl2024fast}.
These models aim to directly model the targeted probability distribution under which an expectation is computed, thereby sidestepping the molecular dynamics simulation.

The approach presented in this thesis is a promising step towards fast and accurate molecular dynamics simulations, but the learning of the dynamics is still an open problem.
Empirically, we have seen that although the prior knowledge of the dynamics yields good results in theory and on toy problems, at scale the modeling capacity of the function approximators seems to matter more.
This is in line with recent advances in natural language processing and video generation, in which the computational capabilities were the bottleneck for the performance of the models.
It would therefore be interesting to see how far the learning of the dynamics can be pushed with the current state of the art in machine learning which mostly relies on size and modelling capability.

\vspace{0.25cm}
\noindent\textbf{Bridging Discrete and Continuous Diffusion Models}
\vspace{0.25cm}

The convergence of the Ehrenfest process to the Ornstein-Uhlenbeck process provides a key insight into the link between discrete and continuous diffusion processes \cite{winkler2024ehrenfest}.
From this link, numerous theoretically and practically advantageous properties can be derived.
The main advantage of this approach is the ability to learn the dynamics of the discrete process with the continuous process, which can be directly learned with neural networks framing it as a regression instead of a variational inference task.

The main drawback is that the discrete to continuous link is only possible through the correspondence of discrete and continuous processes and does not work for arbitrary discrete processes, such as categorical jump processes.
Fundamentally, the link between the discrete Ehrenfest process and the continuous Ornstein-Uhlenbeck process is shown to exist in an ordered state space \cite{sumita2004numerical}.
While the number of problem classes with ordered state spaces is certainly high, this approach is not readily applicable to arbitrary discrete processes in unordered state spaces.
Here a recent approach using the 'concrete' score estimator has shown to surpass GPT-2 models in terms of text generation \cite{meng2022concrete, lou2023discrete, radford2019language}.

A particular numerical bottleneck is the sampling of the birth-death process, which requires either an explicit factorization of the dimensions or an implicit factorization via the used $\tau$-leaping sampling method \cite{cao2006efficient, gillespie1976general, gillespie2001approximate}.
Remarkable advances in sampling speed have been achieved for the continuous diffusion models by employing novel solvers based on the probability flow formulation.
These sampling methods leverage efficient and fast ODE solvers or use novel noise schedules \cite{maoutsa2020interacting, song2020score, song2020denoising, karras2022elucidating}.
Compared to continuous generative diffusion models the development of efficient, exact, and fast sampling of either birth-death or categorical jump processes can be considered to be still an open problem.

\vspace{0.25cm}
\noindent\textbf{Tractable and Scalable Schrödinger Bridges}
\vspace{0.25cm}

The Schrödinger Bridge Problem has been tackled with a novel criterion that uses a surrogate loss via integration by parts, moving the gradient of the log probability of the marginal distributions to the Jacobian of the model output \cite{winkler_2023}.
The main advantage of this approach is the ability to learn the full bridge without any ground truth data or reference process.

Yet, this requires the evaluation of the Jacobian of the neural network, which in its explicit form is computationally expensive and poses a significant bottleneck.
For the work resulting in this thesis, we employed stochastic trace estimation methods to accelerate the computation of the trace of the Jacobian term.
Naturally, the stochasticity of the trace estimator introduces variance into the loss, and already for moderately sized state spaces the variance of the estimator is prohibitively high.
This has been identified as the main bottleneck for scaling the proposed criterion to high-dimensional problems.
Alternative approaches have used simulation-based methods which learned the gradient of the log marginal probabilities implicitly, \cite{de2021diffusion}.
In this approach, the models were trained to reverse the diffusive process directly for very small time steps.
This circumvents the need to compute the score as a property of the marginal probabilities.

Another drawback is the need for explicitly sampling from the reference process during the IPF iterations.
This simulation-based training is slow and requires a large number of samples to obtain accurate estimates of the gradient, especially in high-dimensional spaces.
Recent work has shown that conditional flow matching with subsequent divergence correction is simulation-free and conceptually easy to implement \cite{lipman2022flow, shi2024diffusion}.
This approach directly learns a vector field of velocities inferred by an interpolation scheme.
Flow matching learns a velocity field between a Gaussian prior distribution and the target distribution.
This methodology can be extended to stochastic bridge learning with conditional flow matching, which enables learning bridges between two arbitrary distributions \cite{shi2024diffusion}.

%% file: chapters/chapterAppendix.tex
\appendix

\newcommand{\KL}[2]{\text{KL}\left[ #1 \ || \ #2 \right]}
\newcommand{\denom}[1]{\frac{1}{#1}}

\chapter{Appendix}

\begin{tcolorbox}[colback=gray!10!white, colframe=black]
  Parts of this chapter are constituted from the publications:
  \begin{itemize}
    \item \fullcite{winkler_2022}
    \item \fullcite{winkler_2023}
    \item \fullcite{winkler2024ehrenfest}
  \end{itemize}
  or have been published on my personal blog at \href{https://ludwigwinkler.github.io}{ludwigwinkler.github.io}.
\end{tcolorbox}

\section{Supplementary Material for Chapter 2}
\subsection{Derivation of Fokker-Planck Equation}
\label{app:ch2fpederivation}

Let us consider the random variable $X(t)$ that follows an Ito drift-diffusion \cite{ito1951stochastic,ito1984introduction,oksendal2003stochastic,risken1996fokker} process of the form

\begin{align}
  dX(t) = \mu(X(t), t) dt + \sigma(X(t), t) dW(t)
\end{align}

where $W_t$ is a Wiener process with $W_t \sim \mathcal{N}(0, t)$ and induces a probability distribution $p(x, t)$.

What follows is a rederivation of the Fokker-Planck equation for the interested machine learning reader without a PhD in pure mathematics.
A more elaborate treatment of the Fokker-Planck equation can be found in \cite{risken1996fokker,tabar2019analysis}.

We want to study an arbitrary function $\mu(X(t), t)$ with a compact support, meaning that $\lim_{X(t) \rightarrow \pm \infty} f(X(t))=0$.
Intuitively, this means that for the extreme values of $\pm \infty$ the function $f(X(t))$ evaluates to zero.
The function $f(X(t))$ should be twice differentiable in its argument $X(t)$ such that we can use the Taylor expansion up to the second order, giving us

\begin{align}
  df(X(t)) = \partial_x f(X(t)) dX(t) + \frac{1}{2} \partial_x^2 f(X(t)) dX(t)^2.
\end{align}

For the infinitesimal values $dt$, any term with an exponent higher than one will go towards zero at a faster rate.
Thus the terms $dt^2$, $dt dW_t = dt^{1.5}$ will evaluate to zero at their infinitesimal limit.
We can then plug in the dynamics of $X(t)$ to obtain

\begin{align}
  df(X(t)) = & \partial_x f(X(t)) dX(t) + \denom{2} \partial_x^2 f(X(t)) dX(t)^2                           \\
  =          & \partial_x f(X(t)) \left( \drift dt + \diff dW_t \right)                                    \\
             & + \denom{2} \partial_x^2 f(X(t)) \Big( \drift dt + \diff dW_t \Big)^2                       \\
  =          & \partial_x f(X(t)) \left( \drift dt + \diff dW_t \right)                                    \\
             & + \denom{2} \partial_x^2 f(X(t)) \Big( \drift^2 \underbrace{dt^2}_{=0}                      \\
             & \quad + \drift \diff \underbrace{ dt \ dW_t}_{=0} + \diff^2 \underbrace{dW_t^2}_{=dt} \Big) \\
  =          & \left(\drift \partial_x f(X(t)) + \denom{2} \diff^2 \partial_x^2 f(X(t)) \right) dt         \\
             & + \diff \partial_x f(X(t)) dW_t
\end{align}

We can easily see that the differential $df$ follows an Ito drift-diffusion process, although with modified drift and diffusion terms in direct comparison to $dX(t)$.
Naturally we can take the expectation to isolate the drift of $df(X(t))$ since $\Efunc{dW_t}=0$ by definition,

\begin{align}
  \Efunc{df(X(t))}             & = \Efunc{ \drift \partial_x f(X(t)) + \denom{2} \diff^2 \partial_x^2 f(X(t)) } dt \\
  \frac{d}{dt} \Efunc{f(X(t))} & = \Efunc{ \drift \partial_x f(X(t)) + \denom{2} \diff^2 \partial_x^2 f(X(t)) }
\end{align}

Since the Wiener process $W_t$ introduces stochasticity into the evolution of $X(t)$, we are in fact dealing with a distribution $p(x, t)$.
We can then proceed by plugging in the distribution $p(x, t)$ into the expectation and writing it out,

\begin{align}
  \frac{d}{dt} \Efunc{f(X(t))} = & \Efunc{ \drift \partial_x f + \denom{2} \diff^2 \partial_x^2 f(X(t)) }        \\                                                                     \\
  =                              & \int_{-\infty}^\infty \drift \ \partial_x f(X(t)) \ p(x, t) dx                \\
                                 & + \denom{2} \int_{-\infty}^\infty \diff^2 \ \partial_x^2 f(X(t)) \ p(x, t) dx
\end{align}

The state so far is that we reduced the expected change in $f$ to two integrals which we now have to solve.
For this we can utilize integration by parts which utilizes the integral of the product rule.
Remember that

\begin{align}
  \partial_x \left[ u(x) v(x) \right] = \partial_x \left[  u(x) \right] v(x) + u(x) \partial_x \left[ v(x) \right]
\end{align}

or in a easier form

\begin{align}
  \left( u(x) v(x) \right)' = u'(x) v(x) + u(x) v'(x)
\end{align}

The integration by parts rule states that for a range $x \in [ a, b ]$

\begin{align}
  \left[ u(x) v(x) \right]_a^b = \int_a^b u'(x) v(x) dx + \int_a^b u(x) v'(x) dx
\end{align}

or alternatively

\begin{align}
  \int_a^b u(x) v'(x) dx = \left[ u(x) v(x) \right]_a^b - \int_a^b u'(x) v(x) dx
\end{align}

We can now proceed to identify the relevant terms $u(x)$ and $v(x)$ in the two integrals,

\begin{align}
  \frac{d}{dt} \Efunc{f(X(t))} = & \int_{-\infty}^\infty \underbrace{\drift \ p(x, t)}_{u(x)} \ \underbrace{\partial_x f(X(t))}_{v'(x)}  dx               \\
                                 & + \denom{2} \int_{-\infty}^\infty \underbrace{ \diff^2 \ p(x, t)}_{u(x)} \ \underbrace{\partial_x^2 f(X(t))}_{v(x)} dx \\
  =                              & \underbrace{\left[  \drift \ p(x, t)  \  f(X(t)) \right]_{-\infty}^\infty}_{=0}                                        \\
                                 & - \int_{-\infty}^\infty  \partial_x \left[ \drift \ p(x, t) \right] \ f(X(t)) \ dx                                     \\
                                 & + \denom{2} \underbrace{\left[  \diff^2 \ p(x, t)  \  \partial_x f(X(t)) \right]_{-\infty}^\infty}_{=0}                \\
                                 & - \denom{2} \int_{-\infty}^\infty  \partial_x \left[ \diff^2 \ p(x, t) \right] \ \partial_x f(X(t)) \ dx
\end{align}

For any reasonable continuous probability distribution, evaluating $p(x,t)$ at $\pm \infty$ evaluates to zero such that the evaluation brackets $\left[ p(x,t) \ldots \right]_{-\infty}^\infty = 0$.
We can then apply the integration by parts a second time on the second integral to obtain

\begin{align}
  \frac{d}{dt} \Efunc{f(X(t))} = & - \int_{-\infty}^\infty  \partial_x \left[ \drift \ p(x, t) \right] \ f(X(t)) \ dx                                                                      \\
                                 & - \denom{2} \int_{-\infty}^\infty  \underbrace{\partial_x \left[ \diff^2 \ p(x, t) \right]}_{u(x)} \ \underbrace{\partial_x f(X(t))}_{v'(x)} \ dx       \\
  =                              & \int_{-\infty}^\infty  \partial_x \left[ \drift \ p(x, t) \right] \ f(X(t)) \ dx                                                                        \\
                                 & - \denom{2} \underbrace{\left[ \partial_x \left[ \diff^2 \ p(x, t) \right] \ f(X(t)) \right]_{-\infty}^\infty}_{=0}                                     \\
                                 & + \denom{2} \int_{-\infty}^\infty \partial_x^2 \left[ \diff^2 \ p(x, t) \right] \ f(X(t))\ dx                                                           \\
  =                              & \int_{-\infty}^\infty f(X(t)) \left( - \partial_x \left[ \drift \ p(x, t) \right] + \denom{2} \partial_x^2 \left[ \sigma^2 \ p(x, t) \right] \right) dx
\end{align}

With Leibniz' rule we can pull in the time derivative on the left hand side to obtain

\begin{align}
  \frac{d}{dt} \Efunc{f(X(t))} = & \frac{d}{dt} \int_{-\infty}^\infty f(X(t)) p(x,t) dx   \\
  =                              & \int_{-\infty}^\infty f(X(t)) \ \partial_t \ p(x,t) dx
\end{align}

which gives us

\begin{align}
  \int_{-\infty}^\infty f(X(t)) \ \partial_t \ p(x,t) dx = \int_{-\infty}^\infty f(X(t)) \Big( & - \partial_x \left[  \drift \ p(x, t) \right]                      \\
                                                                                               & + \denom{2} \partial_x^2 \left[ \diff^2 \ p(x, t) \right] \Big) dx
\end{align}

The last step to obtain the Fokker-Planck equation is to observe that the function $f$ which is integrated over occurs both on the left and the right hand side.
Since the integrals $\int f(x) \ldots dx$ is identical on both sides we can equate the derivatives directly to obtain

\begin{align}
  \partial_t \ p(x,t) = - \partial_x \left[ \drift \ p(x, t) \right] + \denom{2} \partial_x^2 \left[ \diff^2 \ p(x, t) \right]
\end{align}

which results in the celebrated Fokker-Planck PDE.

\subsection{Derivation of Reverse Time Stochastic Differential Equations}
\label{app:cha2reversetimederivation}

For the interested reader, this is a basic derivation of reverse time stochastic differential equations.
A more thorough treatment can be found in \cite{nelson1979connection,nelson1966derivation,nelson1967dynamical,anderson1982reverse}.

We start out with the Fokker-Planck equation (FPE) which relates the change over time for the probability for a specific value of $x$ with a diffusion term $\sigma(t)$ which is only dependent on the time,
\begin{align}
  \partial_t \ p(x,t) = & - \partial_x \left[ \drift \ p(x, t) \right] + \denom{2} \partial_x^2 \left[ \sigma(t)^2 \ p(x, t) \right]  \\
  =                     & - \partial_x \left[ \drift \ p(x, t) \right] + \denom{2} \sigma(t)^2 \partial_x^2 \left[ \ p(x, t) \right].
\end{align}

Now we consider a time reversion $\tau(t) = 1 - t$ and are interested in what the change of the probability distribution is under this reversed time index,
\begin{align}
  \partial_{t} \ p(x,\tau(t)) = & - \partial_x \left[ \mu(X(\tau(t)), \tau(t)) \ p(x, \tau(t)) \right]       \\
                                & + \denom{2} \partial_x^2 \left[ \sigma(\tau(t))^2 \ p(x, \tau(t)) \right].
\end{align}
With the time transformation $\tau(t)$, we apply the chain rule on the time transformation on the left hand side to obtain
\begin{align}
  \frac{\partial p(x,\tau(t))}{\partial t}
  = \frac{ \partial p(x,\tau(t))}{\partial \tau} \ \frac{\partial \tau(t)}{\partial t}
  = \frac{ \partial p(x,\tau(t))}{\partial \tau} \
  \underbrace{ \frac{\partial \tau(t)}{\partial t}}_{-1}
  = -\frac{ \partial p(x,\tau(t))}{\partial \tau}.
\end{align}
Then, we pull the negative factor from the chain rule to the right hand side and combine the drift and diffusion term into a single derivative via the distributive property of the partial derivative,
\begin{align}
  \frac{ \partial p(x,\tau(t))}{\partial \tau} = & \partial_x \left[ \mu(X(\tau(t)), \tau(t)) \ p(x, \tau(t)) \right]         \nonumber                              \\
                                                 & - \denom{2} \sigma(\tau(t))^2 \partial_x^2 \left[ \ p(x, \tau(t)) \right]  \label{eq:app_reversetime_derivation1} \\
  =                                              & - \partial_x \Bigg[ -\mu(X(\tau(t)), \tau(t)) \ p(x, \tau(t))              \nonumber                              \\
                                                 & \qquad \quad + \denom{2} \sigma(\tau(t))^2 \partial_x \left[ \ p(x, \tau(t)) \right] \Bigg]
\end{align}
Applying the log derivative identity $\partial_x \log p(x) = \nicefrac{1}{p(x)} \partial_x p(x)$, which rearranged yields $ \partial_x p(x) = p(x) \partial_x \log p(x)$, we obtain
\begin{align}
  \frac{ \partial p(x,\tau(t))}{\partial \tau} = & - \partial_x \Bigg[ -\mu(X(\tau(t)), \tau(t)) \ p(x, \tau(t))                            \nonumber                                                                             \\
                                                 & \qquad \quad + \denom{2} \sigma(\tau(t))^2 \partial_x \log p(x, \tau(t)) p(x, \tau(t) ) \Bigg]                                                                                 \\
  =                                              & - \partial_x \Bigg[ \Big( \underbrace{-\mu(X(\tau(t)), \tau(t)) + \denom{2} \sigma(\tau(t))^2 \partial_x \log p(x, \tau(t))}_{\text{reverse drift}} \Big) p(x, \tau(t))\Bigg].
\end{align}

The equation above states that the inverted drift with an additional scaled score term of the forward distribution will invert the stochastic process.
Interestingly, there is no diffusion term occurring in this formulation of the reverse FPE.
We can in fact derive a more flexible reverse drift by returning to a slightly rearranged equation \ref{eq:app_reversetime_derivation1} with an additional scaling factor $\alpha^2$,
\begin{align}
  \frac{ \partial p(x,\tau(t))}{\partial \tau} = & - \partial_x \left[ - \mu(X(\tau(t)), \tau(t)) \ p(x, \tau(t)) \right]         \nonumber                                                                                                            \\
                                                 & - \denom{2} \sigma(\tau(t))^2 \partial_x^2 \left[ \ p(x, \tau(t)) \right] \underbrace{\pm \ \alpha^2 \ \sigma(\tau(t))^2 \partial_x^2 \left[ \ p(x, \tau(t)) \right]}_{\text{additional diffusion}} \\
  =                                              & - \partial_x \left[ - \mu(X(\tau(t)), \tau(t)) \ p(x, \tau(t)) \right]         \nonumber                                                                                                            \\
                                                 & - \sigma(\tau(t))^2 \left( \nicefrac{1}{2} + \alpha^2 \right) \partial_x^2 \left[ \ p(x, \tau(t)) \right]                                                                                           \\
                                                 & \underbrace{ + \ \frac{\alpha^2}{2} \ \sigma(\tau(t))^2 \partial_x^2 \left[ \ p(x, \tau(t)) \right]}_{\text{additional diffusion}}                                                                  \\
  =                                              & - \partial_x \Bigg[ \Big( - \mu(X(\tau(t)), \tau(t))                                                                                                                                                \\
                                                 & \qquad \qquad + \sigma(\tau(t))^2 \left( \nicefrac{1}{2} + \alpha^2 \right) \partial_x \log p(x, \tau(t)) \Big) \ p(x, \tau(t)) \Bigg]         \nonumber                                            \\
                                                 & + \ \alpha^2 \ \sigma(\tau(t))^2 \partial_x^2 \left[ \ p(x, \tau(t)) \right].
\end{align}

from which we can infer the reverse drift consisting of the inverted original drift with the additionally scaled score with $\alpha$ and the additional diffusion,
\begin{align}
  dX(\tau) = & ( - \mu(X(\tau(t)), \tau(t)) + \sigma(\tau(t))^2 \left( \nicefrac{1}{2} + \alpha^2 \right) \partial_x \log p(x, \tau(t)) ) dt \\
             & + \alpha \ \sigma(\tau(t)) dW(\tau)
\end{align}


\section{Supplementary Material for Chapter 3}


\graphicspath{{./img/chapterMD}}

\subsection{Neural Network Architectures and Optimization}
\label{app:cha3training}

The networks were trained to predict the change in the $3N$ position and $3N$ momentum of every atom for a total input and output dimensions of $6N$ where $N$ is the number of atoms in the respective molecule.
For all networks we employed a standardized number of hidden layers and the number of neurons in the hidden layers was chosen as an integer multiple of the input dimension.
For our experi\-ments we chose five hidden layers and a multiple of 5 which resulted, i.e. for a molecule of $N=10$ atoms in $10 \cdot 3 \cdot 2=60$ input features and $60 \cdot 5=300$ neurons per hidden layer.

The neural networks were trained on predicting entire trajectories of varying length.
For each mini-batch a sub-trajectory of length $\Delta t$ with $n$ time steps was sampled and the first value was used as the initial condition for the integration of the differential equation.
For bi-directional models, the final time step was used for the backward integration.
The recurrent architectures such as RNN and LSTMs allowed for a trajectory as the initial respectively final condition.
This gave better results but made an unfair comparison vis-a-vis the NeuralODE and Hamiltonian Neural Network models, which by design can only take a single time step as their respective initial condition.

The optimization was done on mini-batches of 200 samples with the default parameters of the ADAM optimizer \cite{kingma2014adam}.
Furthermore, a plateau learning rate scheduler was employed which halved the learning rate every epoch if the criterion hadn't improved every epoch by a minimum of $0.001$ until a minimal learning rate of $0.00001$ was reached with a patience of three epochs.
The validation samples were sampled from the final 10\% of the training data set by the same sub-trajectory sampling as described above.

In order to speed up the training and inference speed in bi-directional models, the integration both forward and backward in time was combined into a single combined mini-batch.
The neural networks predicted the forward and backward integration in a single evaluation, and the backward integration was appropriately prepared by reversing the momentum in time.

The experiments were done on the extended-MD17 database~\cite{chmiela2017machine, sauceda2020JCP}. This consists in a set of molecular dynamics simulations carried out for nine molecules with different levels of fluxionality computed at the DFT level of theory using the PBE functional together with the Tkatchenko-Scheffler dispersion correction scheme~\cite{PBE1996,TS}. These molecules are shown in Fig.~\ref{fig:performance}.
This database is available at the \href{http://quantum-machine.org/gdml/#datasets}{quantum-machine website} \cite{chmiela2017machine}.
The length of the trajectories range from 100,000 time steps to almost 1,000,000 time steps with a step size of one femtosecond.
The position and momentum were normalized by subtracting the center of the molecule at each time step and rescaling the momenta such that they were standard normally distributed.
The original characteristics were easily recovered by reversing the normalization.

\begin{figure}[htbp!]
  \centering
  \includegraphics[width=0.66\textwidth]{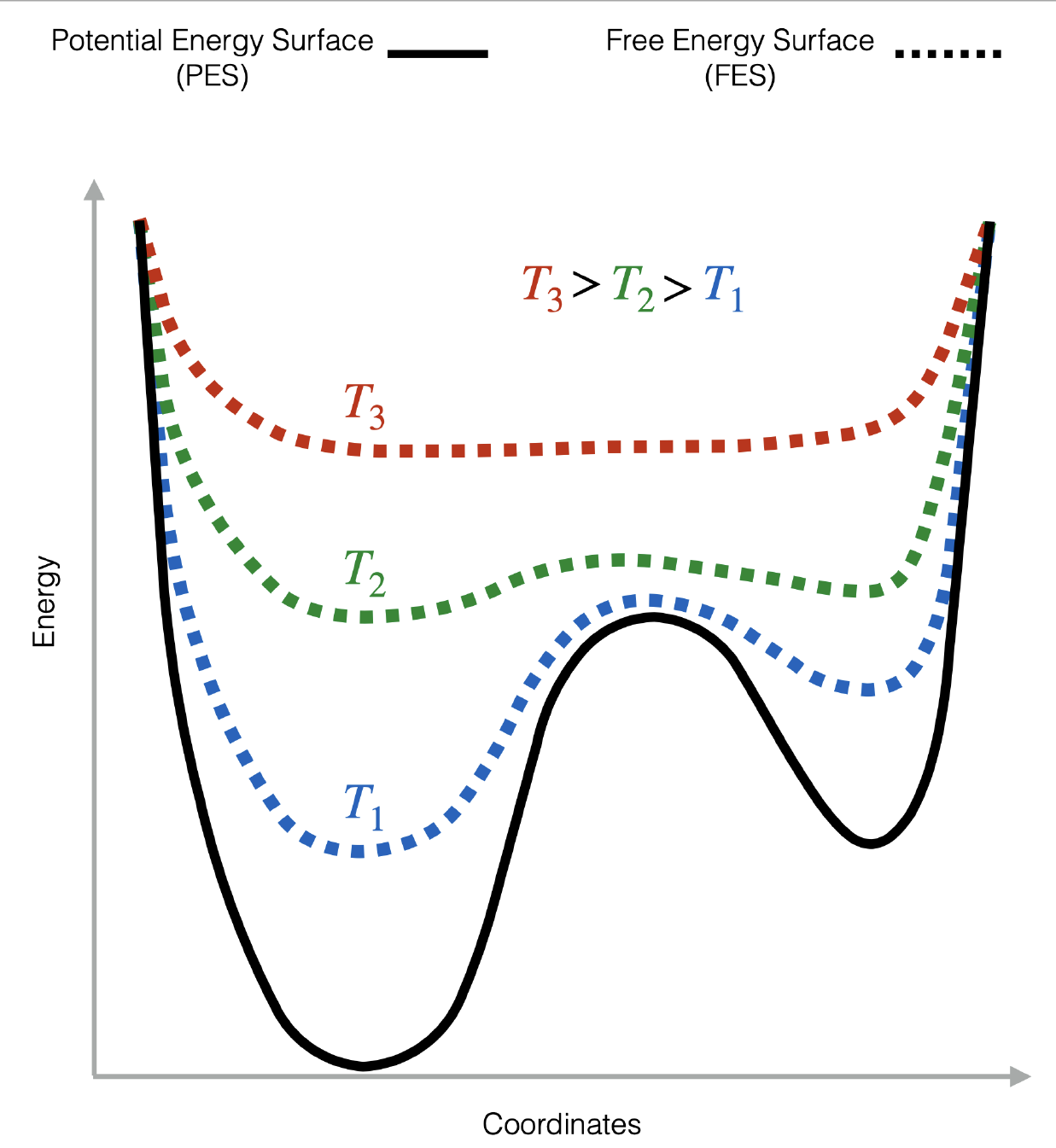}
  \caption{Pictorial representation of the effect of the temperature on the free energy surface (FES) (dotted lines) and its underlying potential energy surface (PES) (black continuous line). Here, the FES is shown at three different temperatures ($T_1 < T_2<T_3$) to show how increasing the temperature generates smoother surfaces.
  }
  \label{fig:FESvsPES}
\end{figure}

\begin{figure}[htbp!]
  \includegraphics[width=0.66\textwidth]{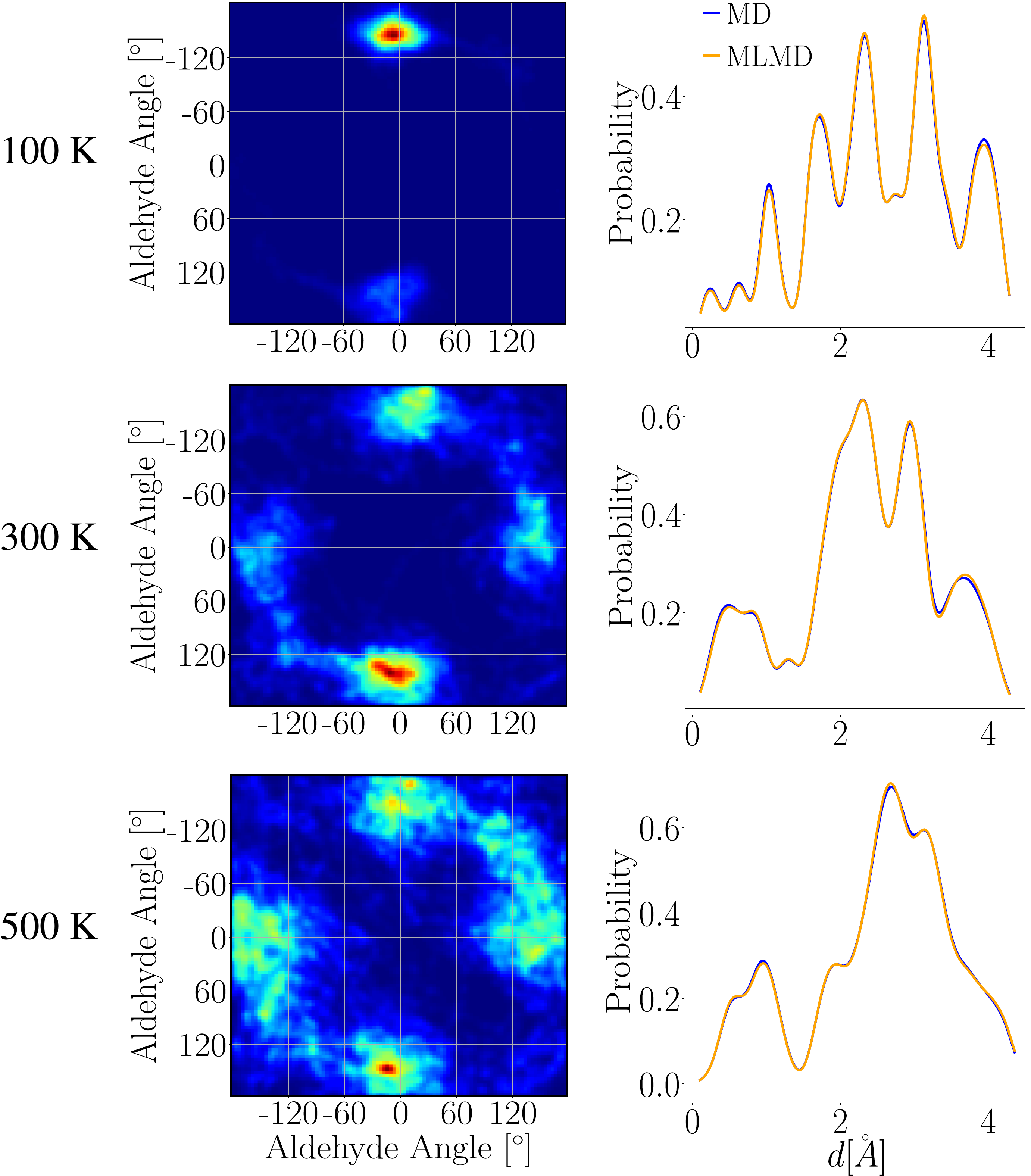}
  \caption{The distribution of interatomic distances $d$[\AA] of keto-MDA at 100 K, 300 K and 500 K. The predicted distribution of interatomic distances is shown in red and the target distribution is shown in blue.
    The distributions become less multi-modal as the temperature increases.
  }
  \label{fig:HistkMDA_distTemps}
\end{figure}

\begin{figure}[htbp!]
  \centering
  \includegraphics[width=\textwidth]{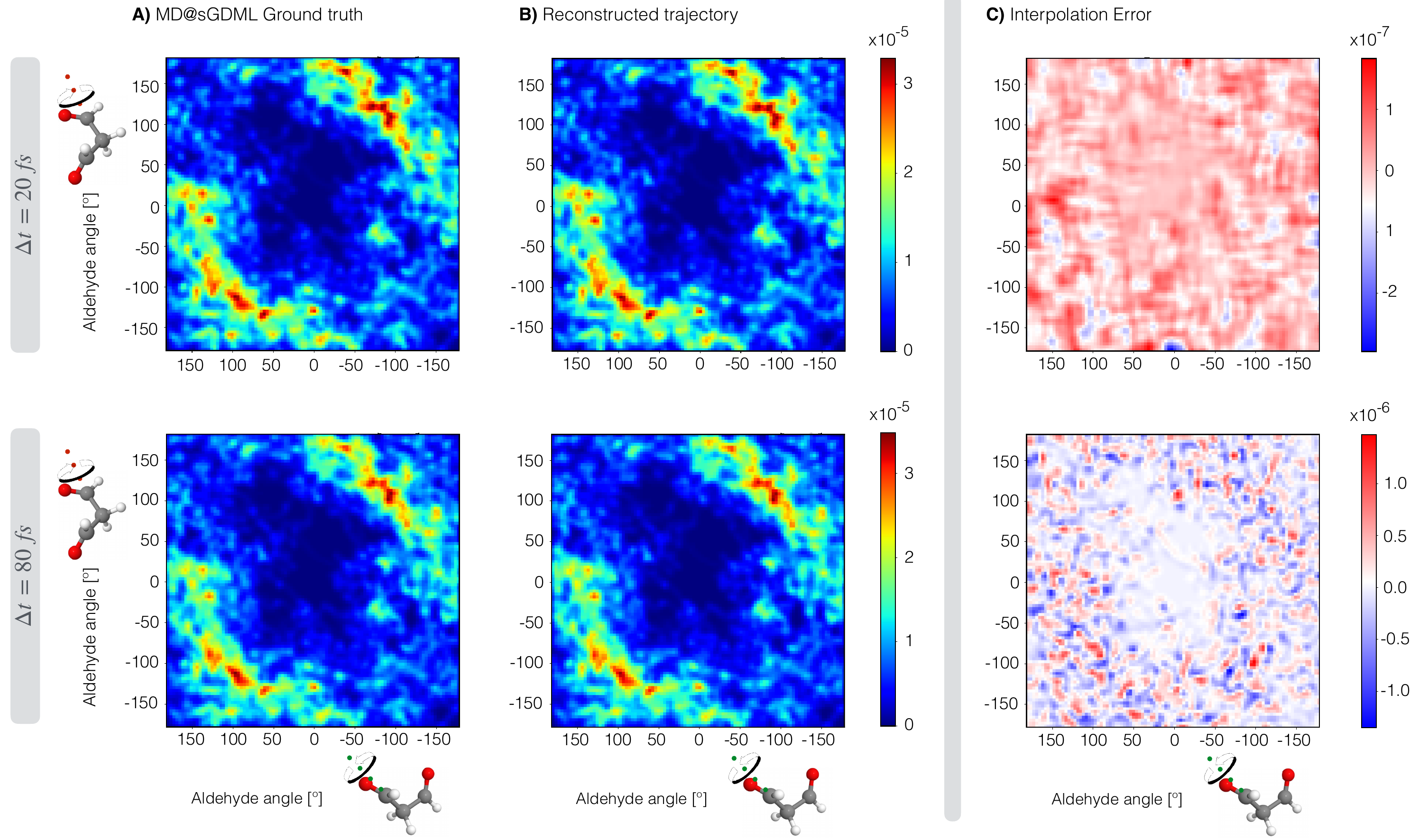}
  \caption{Sampling generated from molecular dynamics using (A) the sGDML force field trained on density functional theory (MD17 database~\cite{chmiela2017machine}) for the keto-MDA, (B) and its interpolation generated using bi-directional LSTM models for $\Delta t=20$ fs and $\Delta t=80$ fs shown in the upper and bottom rows, respectively. (C) Difference between the reference and the interpolated trajectories.}
  \label{fig:Malondialdehyde_DiffHist}
\end{figure}

\newpage
\section{Supplementary Material for Chapter 4}

\subsection{Proof of Convergence of the Ehrenfest Process}

\label{app:cha4proofconvergence}

In order to show the convergence of the Ehrenfest process to the Ornstein-Uhlenbeck process as first shown in \cite{sumita2004numerical}, we can gain some intuition by looking at the first two jump moments of the Markov jump process.
The first jump moment of a continuous time jump process $M(t)$ is defined as
\begin{subequations}
  \begin{align}
    b(x) & := \lim_{\Delta t \to 0} \frac{1}{\Delta t}\E\left[M(t + \Delta t) - M(t) | M(t) = x \right]                                                                \\
         & = \lim_{\Delta t \to 0}\frac{1}{\Delta t}\left(\sum_{y \in \Omega} y \, p_{t + \Delta t|t}(y | x) - x \sum_{y \in \Omega} p_{t + \Delta t|t}(y | x) \right) \\
         & = \lim_{\Delta t \to 0}\frac{1}{\Delta t}\left(\sum_{y \in \Omega; y\ne x} (y-x) p_{t + \Delta t|t}(y | x) \right)                                          \\
         & = \sum_{y \in \Omega; y\ne x} (y-x) r(y | x).
  \end{align}
\end{subequations}
Similarly, the second jump moment is defined as
\begin{subequations}
  \begin{align}
    D(x) & :=  \lim_{\Delta t \to 0} \frac{1}{\Delta t}\E\left[\left(M(t + \Delta t) - M(t)\right)\left(M(t + \Delta t) - M(t)\right)^\top | M(t) = x \right] \\
         & =  \sum_{y \in \Omega; y\ne x} (y-x)(y-x)^\top r(y | x).
  \end{align}
\end{subequations}
Note that the drift and diffusion coefficient for SDEs are defined analogously.

For the scaled Ehrenfest process \eqref{eq: scaled Ehrenfest} we can readily compute the corresponding jump moments with

\begin{align}
  \label{eq: forward first jump moments}
  b(x)
   & = \sum_{y\in \{x \pm \nicefrac{2}{\sqrt{S}}\}} (y - x) r(y | x)                                                             \\
   & = \frac{2}{\sqrt{S}} \frac{\sqrt{S}}{4}(\sqrt{S} - x) + \left( -\frac{2}{\sqrt{S}} \right) \frac{\sqrt{S}}{4}(\sqrt{S} + x) \\
   & = -x
\end{align}
and
\begin{align}
  \label{eq: forward second jump moment}
  D(x) & = \sum_{y\in \{x \pm \nicefrac{2}{\sqrt{S}}\}} (y - x)(y - x)^\top r(y | x)                                                                    \\
       & = \left( \frac{2}{\sqrt{S}} \right)^2 \frac{\sqrt{S}}{4}(\sqrt{S} - x) + \left( -\frac{2}{\sqrt{S}} \right)^2 \frac{\sqrt{S}}{4}(\sqrt{S} + x) \\
       & = (1 - x) + (1 + x)                                                                                                                            \\
       & = 2.
\end{align}
which is summarized in Proposition \ref{prop: convergence of forward Ehrenfest}.

\subsection{Proof of Rates of Reverse Ehrenfest Process}
\label{app:cha4proofreverseehrenfest}

\graphicspath{{./img/chapterDiscDiff}}

In the following proof, we will denote the state change as $\delta = \nicefrac{2}{\sqrt{S}}$.
To make the derivation of the reverse rates more rigorous, we shall introduce the notation
\begin{equation}
  \label{eq: forward finite difference}
  \Delta_\delta p_{t|0}(x | x_0) := \frac{ p_{t|0}(x + \delta | x_0) -  p_{t|0}(x  | x_0)}{\delta}
\end{equation}
and note the identity
\begin{equation}
  p_{t|0}(x + \delta | x_0) = p(x | x_0) + \delta \, \Delta_\delta p_{t|0}(x | x_0),
\end{equation}
which is sometimes called \textit{Newton's series for equidistant nodes} and can be seen as a discrete analog of a Taylor series, where, however, terms of order higher than one vanish.

In \cref{eq: forward first jump moments} and \cref{eq: forward second jump moment} it was shown how the first two jump moments of the Ehrenfest process converges to the drift and diffusion of an OU process.
Similarly, we can now compute the first jump moment of the reverse Ehrenfest process as well and see that
\begin{subequations}
  \begin{align}
    \cev{b}(x)
     & := \sum_{\delta \in \left\{-\frac{2}{\sqrt{S}}, \frac{2}{\sqrt{S}}\right\}} \delta \, \cev{r}(x + \delta | x)                                                                                                             \\
     & =  \sum_{\delta \in \left\{-\frac{2}{\sqrt{S}}, \frac{2}{\sqrt{S}}\right\}} \delta \, \E_{x_0 \sim p_{0|t}(x_0 | x)}\left[ \frac{p_{t|0}(x + \delta | x_0)}{p_{t|0}(x | x_0)} \right] r(x | x + \delta)                   \\
     & = \sum_{\delta \in \left\{-\frac{2}{\sqrt{S}}, \frac{2}{\sqrt{S}}\right\}} \delta \, \E_{x_0 \sim p_{0|t}(x_0 | x)}\left[ 1 + \frac{\delta \, \Delta_\delta p_{t|0}(x | x_0)}{p_{t|0}(x | x_0)} \right] r(x | x + \delta) \\
     & = \sum_{\delta \in \left\{-\frac{2}{\sqrt{S}}, \frac{2}{\sqrt{S}}\right\}} \delta \,   r(x | x + \delta)                                                                                                                  \\
     & \quad + \sum_{\delta \in \left\{-\frac{2}{\sqrt{S}}, \frac{2}{\sqrt{S}}\right\}} \delta^2 \,  \E_{x_0 \sim p_{0|t}(x_0 | x)}\left[\frac{\Delta_\delta p_{t|0}(x | x_0)}{p_{t|0}(x | x_0)} \right] r(x | x + \delta)
  \end{align}
\end{subequations}

At this point, we can use the fact that the forward rates are symmetric in their arguments, i.e. $r(x | x + \delta) = r(x + \delta | x)$ up to a constant factor, to rewrite the forward rates as
\begin{subequations}
  \begin{align}
    r\left(x \bigg| x \pm \frac{2}{\sqrt{S}}\right)
     & = \frac{\sqrt{S}}{4}\left(\sqrt{S} \pm \left(x \pm \frac{2}{\sqrt{S}}\right) \right) \\
     & =  \frac{\sqrt{S}}{4}\left(\sqrt{S} \pm x  \right) \pm \frac{1}{2}                   \\
     & =  - \frac{\sqrt{S}}{4}\left(\sqrt{S} \mp x  \right) \pm \frac{1}{2}                 \\
     & = - r\left(x \pm \frac{2}{\sqrt{S}}  \bigg| x\right) \pm \frac{1}{2}.
  \end{align}
\end{subequations}

Consequently, we can rewrite the first jump moment as
\begin{align}
  \sum_{\delta \in \left\{-\frac{2}{\sqrt{S}}, \frac{2}{\sqrt{S}}\right\}} \delta \, r(x | x + \delta)
   & = \sum_{\delta \in \left\{-\frac{2}{\sqrt{S}}, \frac{2}{\sqrt{S}}\right\}} \delta \ \left(-r(x+ \delta | x )  + \frac{1}{2}\right) \\
   & = -\sum_{\delta \in \left\{-\frac{2}{\sqrt{S}}, \frac{2}{\sqrt{S}}\right\}} \delta \ r(x+ \delta | x ) + o(S^{-1/2}),
\end{align}
where we extracted the sign of the jump direction into the summands.
Similarly, the second jump moment can be rewritten as
\begin{equation}
  \label{eq: rate reversal}
  \sum_{\delta \in \left\{-\frac{2}{\sqrt{S}}, \frac{2}{\sqrt{S}}\right\}} \delta^2 r(x | x + \delta) = \sum_{\delta \in \left\{-\frac{2}{\sqrt{S}}, \frac{2}{\sqrt{S}}\right\}} \delta^2 r(x+ \delta | x ) + o(S^{-1}).
\end{equation}

Secondly, we want to express the ratio $\nicefrac{\Delta_\delta p_{t|0}(x | x_0)}{p_{t|0}(x | x_0)}$ independently of the increment $\delta$ in order to distill the second jump moment.
To that end we substitute the forward finite difference $\Delta_{\delta}p_{0|t}(x|x_0)$ between two neighboring states with its backward finite difference $\Delta_{\bar{\delta}}p_{0|t}(x|x_0)$ such that
\begin{subequations}
  \begin{align}
    \Delta_{\bar{\delta}}p_{0|t}(x|x_0) & =
    \frac{p(x | x_0) - p(x - \bar{\delta} | x_0)}{\bar{\delta}}                                                                               \\
                                        & = \frac{p(x + \bar{\delta}| x_0) - p(x | x_0)}{\bar{\delta}}                                        \\
                                        & \quad + \frac{2 \, p(x | x_0) - p(x + \bar{\delta}| x_0) - p(x - \bar{\delta} | x_0)}{\bar{\delta}} \\
                                        & =\frac{p(x + \bar{\delta}| x_0) - p(x | x_0)}{\bar{\delta}} + o(\bar{\delta}),
  \end{align}
\end{subequations}

The reverse-time Ehrenfest process can therefore be expressed in terms of the forward rates and the conditional expectation of the ratio of the forward transition probabilities as
\begin{subequations}
  \begin{align}
    \cev{b}(x) & = \sum_{\delta \in \left\{-\frac{2}{\sqrt{S}}, \frac{2}{\sqrt{S}}\right\}} \delta \,   r(x | x + \delta)                                                                                                                                       \\
               & \quad + \E_{x_0 \sim p_{0|t}(x_0 | x)}\left[\frac{\Delta_{\bar{\delta}} p_{t|0}(x | x_0)}{p_{t|0}(x | x_0)} \right] \sum_{\delta \in \left\{-\frac{2}{\sqrt{S}}, \frac{2}{\sqrt{S}}\right\}} \delta^2 \,   r(x | x + \delta) + o(\bar{\delta}) \\
               & = -b(x)+ D(x)\E_{x_0 \sim p_{0|t}(x_0 | x)}\left[\frac{\Delta_{\bar{\delta}} p_{t|0}(x | x_0)}{p_{t|0}(x | x_0)} \right]  + o(\bar{\delta}).
  \end{align}
\end{subequations}

In this section we elaborate on numerical details regarding our experiments in \Cref{sec: experiments}.

\subsection{A tractable Gaussian toy example}
\label{app: Gaussian mixture score}

In order to illustrate the properties of the Ehrenfest process, we consider the following toy example. Let us start with the SDE setting and consider the data distribution
\begin{equation}
  p_\mathrm{data}(x) = \sum_{m=1}^M \gamma_m \mathcal{N}(x; \mu_m, \Sigma_m),
\end{equation}
where $\sum_{m=1}^M \gamma_m = 1$ and $\mu_m \in \R^d, \Sigma_m \in \R^{d\times d}$. Further, for the inference SDE we consider the Ornstein-Uhlenbeck process
\begin{equation}
  \mathrm d X_t = - AX_t \mathrm  dt + B \, \mathrm dW_t, \qquad X_0 \sim p_\mathrm{data},
\end{equation}
with matrices $A, B \in \R^{d\times d}$. For simplicity, let us consider $A = \alpha \mathbbm{1}$ and $B = \beta \mathbbm{1}$ with $\alpha, \beta \in \R$. Conditioned on an initial condition $x_0$, the marginal densities of $X$ are then given by
\begin{equation}
  p_t^\mathrm{SDE}(x; x_0) = \mathcal{N}\left(x; x_0 e^{-\alpha t}, \frac{\beta^2}{2\alpha}\left(1-e^{-2\alpha t} \right)\mathbbm{1} \right).
\end{equation}
We can therefore compute
\begin{subequations}
  \begin{align}
    p^\mathrm{SDE}_t(x) & = \int_{\R^d} p^\mathrm{SDE}_t(x; x_0)p_\mathrm{data}(x_0) \mathrm d x_0                                                                                                                               \\
                        & = \sum_{m=1}^M \gamma_m \int_{\R^d}  \mathcal{N}\left(x; y_0 e^{-\alpha t}, \frac{\beta^2}{2\alpha}\left(1-e^{-2\alpha t}  \right)\mathbbm{1} \right)  \mathcal{N}(x_0; \mu_m, \Sigma_m) \mathrm d x_0 \\
                        & = \sum_{m=1}^M \gamma_m  \mathcal{N}\left(x; e^{-\alpha t} \mu_m, \frac{\beta^2}{2\alpha}\left(1-e^{-2\alpha t}  \right)\mathbbm{1}+ e^{-2\alpha t} \Sigma_m  \right).
  \end{align}
\end{subequations}
We can now readily compute the score $\nabla \log p_t^\mathrm{SDE}(x)$.


\subsection{Illustrative example}
\label{app: illustrative example}

As an illustrative example we choose a distribution which is tractable and perceivable.
We model a two dimensional distribution of pixels, which are distributed proportionally to the pixel value of an image of a capital ``E''. The visualization of the data distribution is governed by its $33 \times 33$ pixels and a single sample from the distribution is a black pixel indexed by its location $(x, y)$ on the $33 \times 33$ pixel grid. The diffusive forward process acts upon the coordinate and diffuses with progressing time the black pixels into a two dimensional (approximately) binomial distribution at time $t=1$.

We use the identical architecture as \cite{campbell2022continuous}, used for their illustrative example.
Subsequently, the architecture incorporates two residual blocks, each comprising a Multilayer Perceptron (MLP) with a single hidden layer characterized by a dimensionality of 32, a residual connection that links back to the MLP's input, a layer normalization mechanism, and ultimately, a Feature-wise Linear Modulation (FiLM) layer, which is modulated in accordance to the time embedding. The architecture culminates in a terminal linear layer, delivering an output dimensionality of $2$.
The time embedding is accomplished utilizing the Transformer's sinusoidal position embedding technique, resulting in an embedding of dimension 32. This embedding is subsequently refined through an MLP featuring a single hidden layer of dimension 32 and an output dimensionality of 128. In order to generate the FiLM parameters within each residual block, the time embedding undergoes processing via a linear layer, yielding an output dimension of 2.

We test our proposed reverse rate estimators by training them to reconstruct the data distribution at time $t=0$. For evaluation, we draw $500.000$ individual pixels proportionally to the approximated equilibrium distribution and plot their respective histograms at time $t=0$ in Figure \ref{fig: toy example E}.

For training, we sample 1.000.000 pixel values proportional to the gray scale value of the 'E' image serving as the true data distribution.
We perform optimization with Adam with a learning rate of $0.001$ and optimize for 100.000 time steps with a batch size of 2.000.

\begin{figure}[ht]
  \centering
  \includegraphics[width=\linewidth]{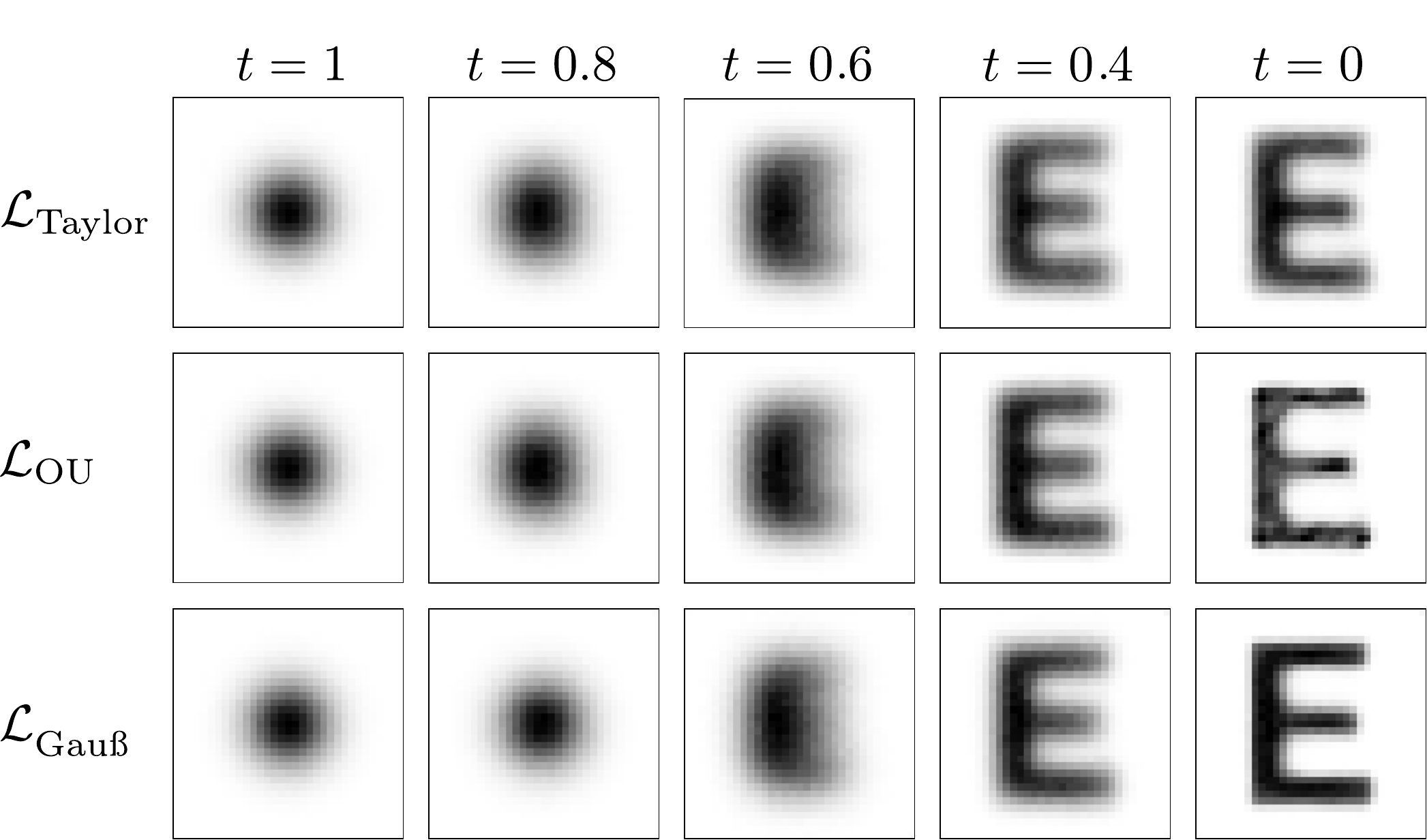}
  \caption{We plot histograms of $500.000$ samples from the time-reversed scaled Ehrenfest process at different times. The processes have been trained with three different losses.}
  \label{fig: toy example E}
\end{figure}

\subsection{MNIST}
\label{app: MNIST}

The MNIST experiments were conducted with the scaled Ehrenfest process.
The MNIST data set \cite{lecun1998gradient} consists of $28 \times 28$ gray scale images which we resized to $32 \times 32$ in order to be processable by use our standard DDPM architecture \cite{ho2020denoising}.
We used $S=256^2$ states to ensure 256 states in the range of $[-1, 1]$ with a difference between states of of $\frac{2}{\sqrt{S}}$.
For optimization, we resorted to the default hyperparameters of Adam \cite{kingma2014adam} and used an EMA of 0.99 with a batch size of 128.
For the rates we chose the continuous DDPM schedule proposed by Song and we stopped the reverse process at $t=0.01$ due to vanishing diffusion and resulting high variance rates close to the data distribution.

\begin{figure}[H]
  \centering
  \includegraphics[width=\linewidth]{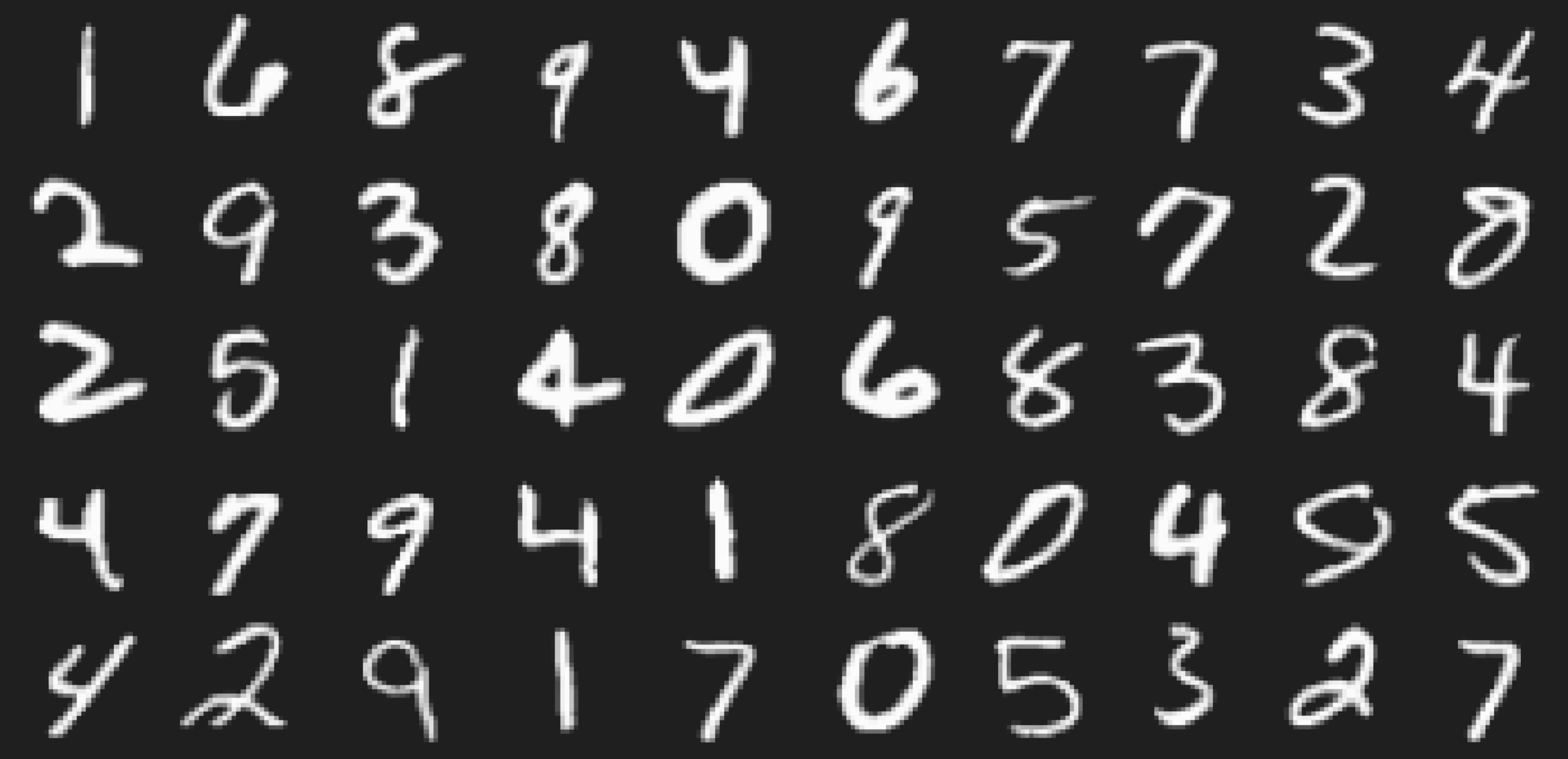}
  \caption{MNIST samples obtained with the time-reversed scaled Ehrenfest process which was trained with $\mathcal{L}_\mathrm{OU}$.}
  \label{fig: mnist}
\end{figure}

\subsection{Image modeling with CIFAR-10}
\label{app: CIFAR}


We employ the standard DDPM architecture from \cite{ho2020denoising} and adapt the output layer to twice the size when required by the conditional expectation and the Gaussian predictor.
The score and first order Taylor approximations did not need to be adapted.
For the ratio case, we adapted the architecture by doubling the final convolutional layer to six channels such that the first half (three channels) predicted the death rate and the second three channels predicted the birth rate.
For the time dependent rate $\lambda_t$ we tried the cosine schedule of \cite{nichol2021improved} and the variance preserving SDE schedule of \cite{song2020score}. The cosine schedule ensures the expected value of the scaled Ehrenfest process to converges to zero with $\mathbb{E}_{x_0}[x_t] = \cos\left(\frac{\pi}{2} \ t\right)^2 x_0$ and translates to a time dependent jump process rate of $\lambda_t = \frac{1}{4} \,\pi \tan\left(\frac{\pi}{2} \ t\right)$, which is unbounded close to the equilibrium distribution and therefore has to be clamped. We choose $\lambda_t \in [0, 500]$ in our case.
Due to numerical considerations regarding the exploding rates due to diminishing diffusion close to $t=0$, we restricted the reverse process to times $t\in [0.01, 1]$. In general, we can transform any deliberately long sampling time $T$ to $T=1$ via the time transformation of the master equation.

We use the standard procedure for training image generating diffusion models \cite{loshchilov2016sgdr}.
In particular, we employ a linear learning rate warm up for $5.000$ steps and a cosine annealing from $0.0002$ to $0.00001$ with the Adam optimizer. The batch size was chosen as $256$ and an EMA with the factor $0.9999$ was applied for the model used for sampling. For sampling we ran the reverse process for $1.000$ steps and employed $\tau$-Leaping as showcased in \cite{campbell2022continuous} with a resulting $\tau=0.001$. We also utilized the predictor-corrector sampling method starting at $t=0.1$ to the minimum time of $t=0.01$. Whereas \cite{campbell2022continuous} reported significant gains performing corrector sampling, we observe behavior close to other state-continuous diffusion models which only apply few or no corrector steps at all.

\begin{figure}[ht]
  \centering
  \begin{minipage}{0.48\textwidth}
    \centering
    \includegraphics[width=\linewidth]{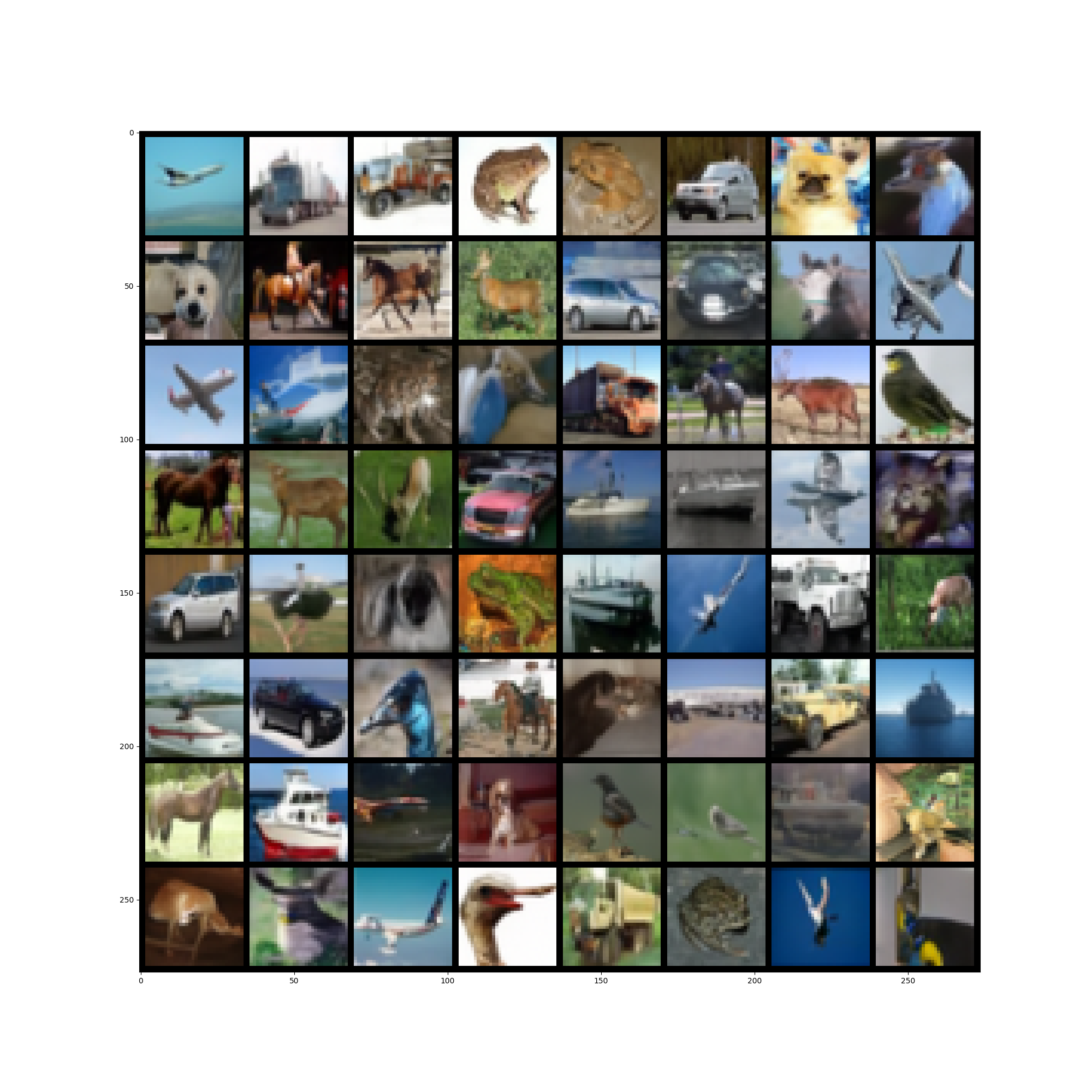} 
    \caption{Samples from the reverse scaled Ehrenfest process obtained by finetuning the DDPM architecture with $\mathcal{L}_\text{Taylor}$ (\ref{eq: first Taylor ratio loss}).}
    \label{fig:cifar 10 taylor big 1}
  \end{minipage}\hfill
  \begin{minipage}{0.48\textwidth}
    \centering
    \includegraphics[width=\linewidth]{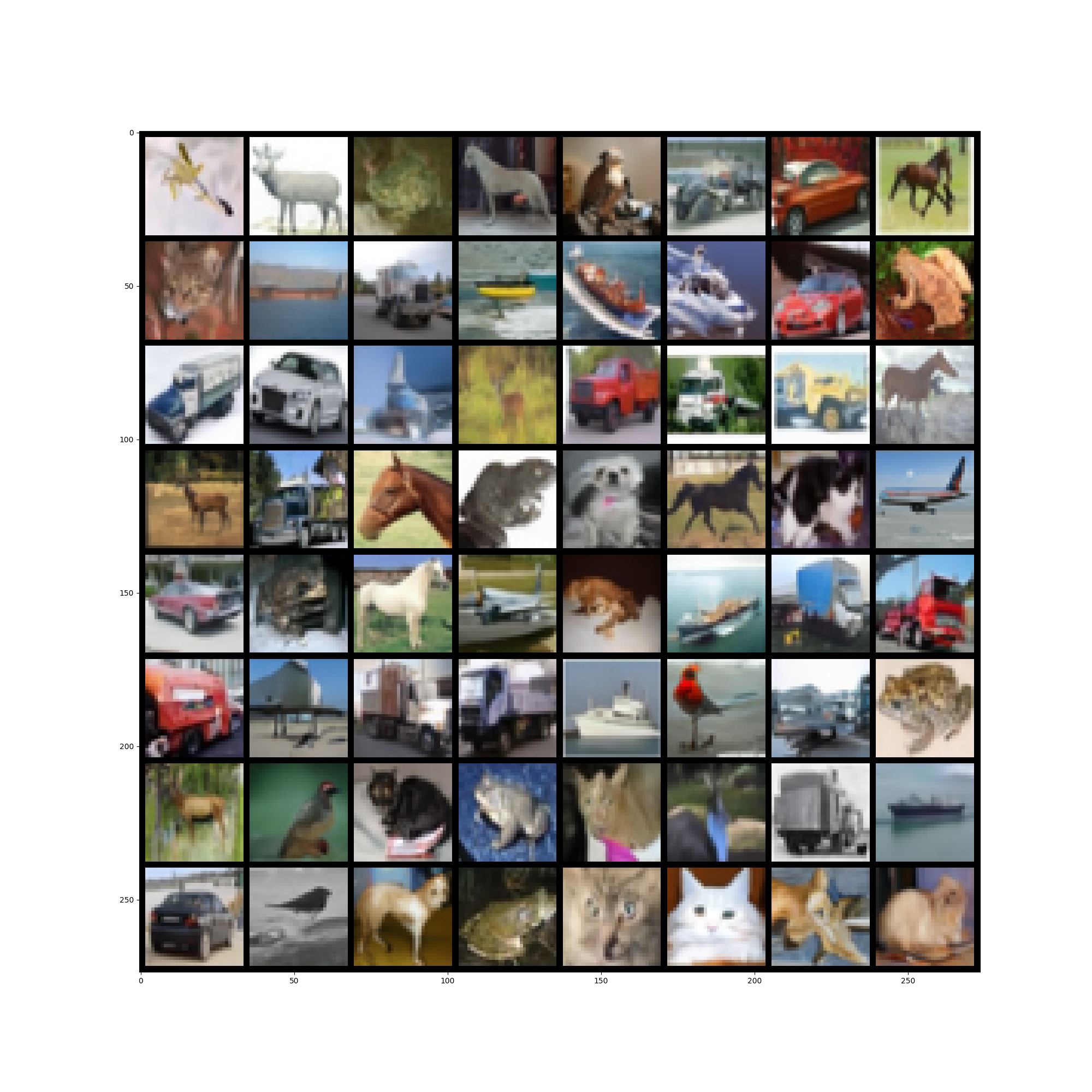} 
    \caption{Samples from the reverse scaled Ehrenfest process obtained by finetuning the DDPM architecture with $\mathcal{L}_\text{Taylor}$ (\ref{eq: first Taylor ratio loss}).}
    \label{fig:cifar 10 taylor big 2}
  \end{minipage}
\end{figure}

\begin{figure}[ht]
  \centering
  \begin{minipage}{0.48\textwidth}
    \centering
    \includegraphics[width=\linewidth]{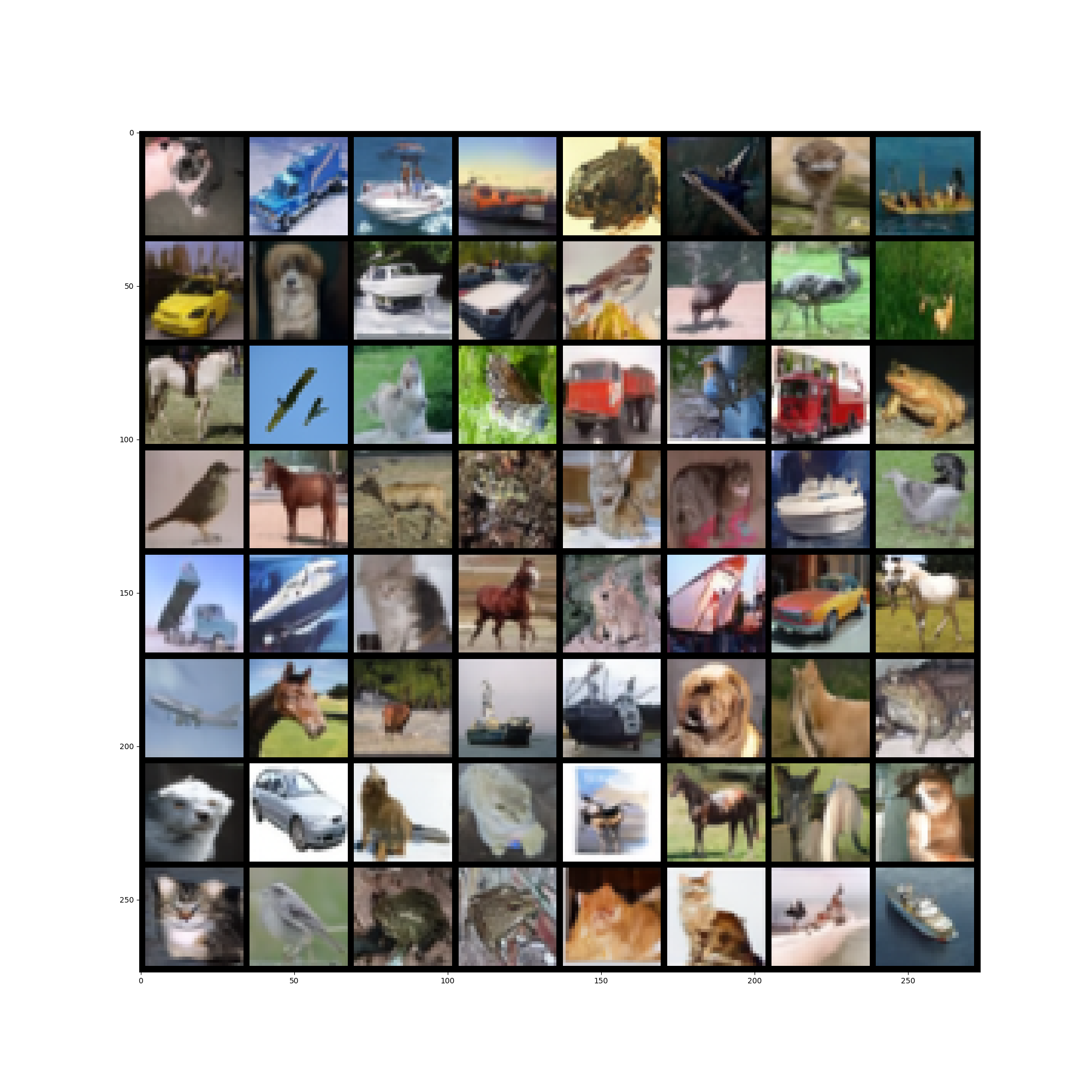} 
    \caption{Samples from the reverse scaled Ehrenfest process obtained by finetuning the DDPM architecture with $\mathcal{L}_\text{OU}$ (\ref{eq: forward OU loss}).}
    \label{fig:cifar 10 score big 1}
  \end{minipage}\hfill
  \begin{minipage}{0.48\textwidth}
    \centering
    \includegraphics[width=\linewidth]{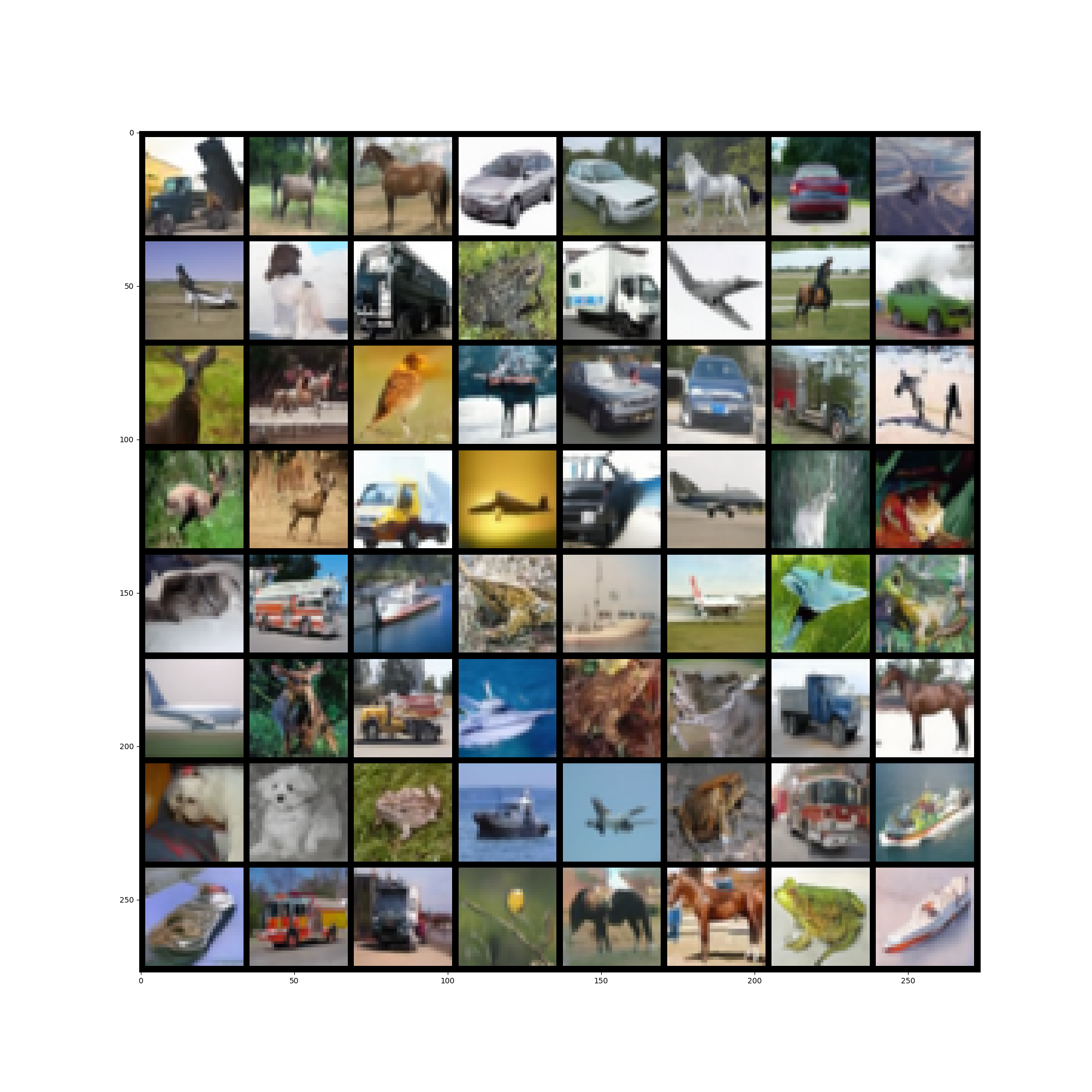} 
    \caption{Samples from the reverse scaled Ehrenfest process obtained by finetuning the DDPM architecture with $\mathcal{L}_\text{OU}$ (\ref{eq: forward OU loss}).}
    \label{fig:cifar 10 score big 2}
  \end{minipage}
\end{figure}

\section{Supplementary Material for Chapter 5}
\subsection{Proof of the Score Matching with a Reference Function}
\label{app:cha5scorematchingproof}

Here we present the proof of the equivalency that the optimum of the criterion does indeed yield the correct optimum when $\phi(x)=-\mu(x) + \sigma^2 \nabla_x \log p(x)$.
Without loss of generality and for notational simplicity we drop the dependency on the time index $t$ and consider the one dimensional case, as the proof holds for higher dimensions and for auxiliary inputs.

Under the assumption of vanishing probability of data at infinity, namely
\newline $\lim_{x \rightarrow \pm \infty} p(x)=0$, integration by parts yields
\begin{align}
  \int_{-\infty}^{\infty} p(x) \partial_x \phi(x) dx
   & = \overbrace{\left[ p(x) \phi(x) \right]_{-\infty}^{\infty}}^{=0} - \int \partial_x p(x) \phi(x) dx \\
   & = - \int \partial_x p(x) \phi(x) dx
\end{align}

Additionally, we will refer to two additional identities which are used in the proof, namely the log-derivative trick and the completion of the square.
The log-derivative trick is a standard technique to express the derivative of a function in terms of the function itself and its log-derivative, i.e.
\begin{align}
  \partial_x \log f(x) = \frac{1}{f(x)}\partial_x f(x) \quad \Leftrightarrow \quad \partial_x f(x) = f(x) \partial_x \log f(x)
\end{align}
Completing the square is a standard technique to express a quadratic function in a more convenient form, i.e.
\begin{align}
  a^2 + 2 ab & = a^2 + 2 ab \pm b^2    \\
             & = a^2 + 2ab + b^2 - b^2 \\
             & = (a + b)^2 - b^2
\end{align}

The proof goes as follows,
\begin{align}
  \mathcal{L}[\phi, \mu] & = \Efunc{\phi^2(x) + 2 \mu(x) \phi(x) + 2 \sigma^2 \partial_x \phi(x)}                                                                                        \\
                         & = \int p(x) \left( \phi^2(x) + 2 \mu(x) \phi(x) + 2 \sigma^2 \partial_x \phi(x) \right) dx                                                                    \\
                         & = \int p(x) \left( \phi^2(x) + 2 \mu(x) \phi(x) \right) dx + 2 \sigma^2 \underbrace{\int p(x) \partial_x \phi(x) dx}_{\mathclap{\text{integration by parts}}} \\
                         & = \int p(x) \left( \phi^2(x) + 2 \mu(x) \phi(x) \right) dx - 2 \sigma^2 \int \underbrace{\partial_x p(x) \phi(x)}_{\mathclap{\text{log-derivative trick}}} dx \\
                         & = \int p(x) \left( \phi^2(x) + 2 \mu(x) \phi(x) - 2 \sigma^2 \phi(x) \partial_x \log p(x)) \right) dx                                                         \\
                         & = \int p(x) \Big( \underbrace{\phi^2(x) + 2 \phi(x) \{ \mu(x) - \sigma^2 \partial_x \log p(x) \}}_{\text{complete the square}} \Big) dx                       \\
                         & = \int p(x) \left( \left(\phi(x) + \{ \mu(x) - \sigma^2 \partial_x \log p(x) \} \right)^2 \right) \nonumber                                                   \\
                         & \quad - \int p(x) \left( \sigma^2 \partial_x \log p(x) - \mu(x) \right)^2 dx                                                                                  \\
                         & = \Efunc{ \left(\phi(x) - \{ -\mu(x) + \sigma^2 \partial_x \log p(x)\} \right)^2} \nonumber                                                                   \\
                         & \quad - \Efunc{ (- \mu(x) + \sigma^2\partial_x \log p(x))^2}
\end{align}

The criterion above achieves its optimum with respect to $\phi(x)$ when the first term evaluates to zero, namely $\arg \min_\phi \mathcal{L}[\phi, \mu]=-\mu(x) + \sigma^2 \nabla_x \log p(x)$.
Furthermore, we can see that the original criterion which does not require the explicit score, trains the reverse drift $\phi(x)$ implicitly on the correct score up to an additive term.

\subsection{Experimental Setup}
\label{app:cha5experimentalsetup}

For all our experiments we used a deep neural network taking both the spatial input $x_t$ and the time $t$ as distinct inputs.
We fixed the size of the fully connected layers in the hidden layers of the deep neural networks to a integer multiple of the spatial dimension $D$.
As a rule of thumb, we used $10 D$ neurons in the hidden layers and scaled the depth of the deep neural network with $\text{max}(2, D/5)$.

We used LayerNorm \cite{ba2016layer} before the spatial features of the hidden layers as to not destroy the time embeddings.
LayerNorm normalizes the representation of each sample to a standard Gaussian distribution and has empirically been shown to numerically aid the gradient computation.
Thus we used blocks of the shape $x_{i+1}$ = $x_i$ + Tanh( Linear( LayerNorm($x_i$), Embedding(t))).
We used the ADAM optimizer \cite{kingma2014adam} and Cosine Annealing \cite{loshchilov2016sgdr}, training each half-bridge for 1000 steps and annealing the learning rate from $10^{-3}$ to $10^{-5}$.
We drew $N_x = 128$ trajectories per bridge and stored them in a buffer of 512 sample trajectories, discarding old trajectories as needed to maintain the fixed buffer size.
We constructed the Schr\"odinger bridge by running 10 IPF iterations.

The simulations of the SDE were performed using by the \textit{Euler--Maruyama} \cite{oksendal2013stochastic} approximation with a step size $dt$ with $N_t$ and we found $N_t=100$ and $dt = 0.01$ to be robust values working well for our experiments.
Our choice of the diffusion $\sigma$ was motivated by the idea the samples of the half bridge process at the first IPF iteration
should sufficiently cover the marginal distribution increasing the possibility that this distribution is 'hit' with at least some sampled trajectories.
As the diffusion parameter is yet another hyperparameter, we chose the diffusion according to $\sigma=\nicefrac{1}{N_t \cdot dt}$.

The drifts of both processes, forward and backward, received as input the spatial information $x$ and the time $t$.
We normalized the time index $t \in [0, N_t \cdot dt]$ to $t \in \{0, 1\}$ as the time index remained fixed over the course of the entire training of the bridge.

\subsection{Hutchinson's Stochastic Trace Estimation}
\label{appendix:traceestimationtrick}

For a square matrix $A \in \mathbb{R}^{d \times d}$ the trace is defined as
\begin{align}
  \tr{A} = \sum_i^d A_{ii}
\end{align}
which sums over the diagonal terms of the matrix $A$.

We can approximate the exact trace with a sampled approximation.
We therefore from a sample random samples $Z \in \mathbb{R}^D$ for which the mean is a zero vector and the covariance matrix is a identity matrix, i.e. $\Sigma[Z] = I$.

The Rademacher distribution which samples from the set $\{-1, +1\}$ with equal probability offers the lowest estimator variance and is commonly used in the trace estimation trick for this reason.

\begin{align}
  \tr{A}
   & = \tr{I A}                           \\
   & = \tr{\Efuncc{z \sim p(z)}{z z^T} A} \\
   & = \Efuncc{z \sim p(z)}{z^T A z}
\end{align}
where the trace operator disappears as $z^T A z \in \mathbb{R}$ is a scalar value for which the trace is a superfluous operation.

For estimating the trace of the Jacobian, we can circumvent the quadratic nature of the Jacobian by reducing the network output with a random vector z to a scalar, which can then be readily derived with a single backward pass.
\begin{align}
  \tr{J_f(x)}
   & = \Efuncc{z \sim p(z)}{z^T J_f(x) z}             \\
   & = \Efuncc{z \sim p(z)}{z^T \nabla_x [f(x)^T] z}  \\
   & = \Efuncc{z \sim p(z)}{z^T \nabla_x [f(x)^T z] }
\end{align}


\subsection{Stein's Lemma}
\label{appendix:steinslemma}

Let $X \in \mathbb{R}^N$ be a normally distributed random variable $p(x) =\mathcal{N}(x ; \mu, \sigma^2)$ with mean $\mu$ and variance $\sigma^2$.
Here, we will derive an identity which is commonly known as \emph{Stein's lemma} \cite{lin2019stein, ingersoll1987theory}.
Let the derivative of the Gaussian distribution with respect to $x$ be
\begin{align}
  \partial_x p(x)
   & = \partial_x \left[\frac{1}{\sqrt{2\pi} \sigma} e^{-\frac{(x-\mu)^2}{2\sigma^2}} \right]  \\
   & = -\frac{(x-\mu)}{\sigma^2} \frac{1}{\sqrt{2\pi} \sigma} e^{-\frac{(x-\mu)^2}{2\sigma^2}} \\
   & = - \frac{(x-\mu)}{\sigma^2} p(x).
\end{align}
Integration by parts (IbP) which yields the often used identity
\begin{align}
  \int_{x=-\infty}^{\infty} u(x) \partial_x v(x) dx
   & = [u(x)v(x)]_{x=-\infty}^{\infty} - \int_{x=-\infty}^{\infty} \partial_x u(x) v(x) dx.
\end{align}
In practice, the property that either $u(x)$ or $v(x)$ or both evaluate to zero at $x = \pm \infty$ as it is the case with common probability distributions is leveraged as an algebraic trick to 'switch the derivative to the other function'.

Given a function $g(x)$ we can obtain a gradient estimator with the following steps via integration by parts
\begin{align}
  \Efuncc{p(x))}{g(x) ( x - \mu)}
   & = \int g(x) (x-\mu) \frac{1}{\sqrt{2\pi} \sigma} e^{-\frac{(x-\mu)^2}{2\sigma^2}} dx                                                           \\
   & = -\sigma^2 \int g(x) \underbrace{\frac{(x-\mu)}{-\sigma^2}\frac{1}{\sqrt{2\pi} \sigma} e^{-\frac{(x-\mu)^2}{2\sigma^2}}}_{\partial_x p(x)} dx \\
   & = - \sigma^2 \underbrace{\int g(x) \partial_x p(x) dx}_{\text{IbP}}                                                                            \\
   & = -\sigma^2 \big\{ \underbrace{[ g(x) p(x)]_{x=-\infty}^{\infty}}_{p(\pm \infty)=0} - \int \partial_x g(x) p(x) dx \big\}                      \\
   & = \sigma^2 \Efuncc{p(x)}{\partial_x g(x)}
\end{align}

\subsection{Trace Estimation with Stein's Lemma}
\label{appendix:steingradients}
By choosing a perturbation $\epsilon \sim p(0, \sigma_\epsilon^2)$ with zero mean and a small variance $\sigma_\epsilon^2$ we can define a perturbed data point $x' \sim p(x,\sigma_\epsilon^2)$ via $x' = x + \epsilon$.
This transforms Stein's lemma \cite{lin2019stein, ingersoll1987theory} into
\begin{align}
   & \Efuncc{p(\nu))}{g(x') ( x' - x)}
  = \Efuncc{p(\epsilon))}{g(x + \epsilon) \epsilon}
  = \sigma_\epsilon^2 \Efuncc{p(\epsilon)}{\partial_{x'} g(x')}.
\end{align}
In practice we rescale with $1/\sigma_\epsilon^2$ and evaluate the left side of the following identity
\begin{align}
  \Efuncc{p(\epsilon)}{g(x + \epsilon) \frac{\epsilon}{\sigma_\epsilon^2}} = \Efuncc{p(\epsilon)}{\partial_{x+\epsilon} g(x+\epsilon)}.
\end{align}
which gives us an estimator of the gradient $\partial_x g(x)$ by averaging the gradients in the $\epsilon$-neighborhood of $x$.
For a function $g: \mathbb{R}^M \rightarrow \mathbb{R}^N$, the gradient estimation with Stein's lemma estimates the trace of the Jacobian $J_g(x+\epsilon)$
\begin{align}
  \Efuncc{p(\epsilon)}{g(x + \epsilon) \frac{\epsilon}{\sigma_\epsilon^2}} = \Efuncc{p(\epsilon)}{\text{Tr}\left[ J_g(x+\epsilon)\right]}.
\end{align}
In the limit of $\sigma_\epsilon \rightarrow 0$ we obtain the trace estimator
\begin{align}
  \text{Tr}\left[ J_g(x) \right]
  = \lim_{\sigma_\epsilon \downarrow 0} \Efuncc{p(\epsilon)}{\text{Tr}\left[ J_g(x+\epsilon)\right]}
  = \lim_{\sigma_\epsilon \downarrow 0} \Efuncc{p(\epsilon)}{g(x + \epsilon) \frac{\epsilon}{\sigma_\epsilon^2}}
\end{align}
in which we compute the right most term to obtain the left most term.
